\documentclass[a4paper,fleqn]{cas-sc}

\usepackage[authoryear,longnamesfirst]{natbib}
\usepackage{caption}
\def\tsc#1{\csdef{#1}{\textsc{\lowercase{#1}}\xspace}}
\tsc{WGM}
\tsc{QE}
\tsc{EP}
\tsc{PMS}
\tsc{BEC}
\tsc{DE}
\usepackage{amsthm}
\usepackage{enumitem}   
\newtheorem{proposition}{Proposition}[section]
\newtheorem{theorem}{Theorem}[section]

\newtheorem{remark}{Remark}[section]
\newtheorem{definition}{Definition}[section]
\usepackage{subcaption}
\newcommand\BibTeX{{\rmfamily B\kern-.05em \textsc{i\kern-.025em b}\kern-.08em
T\kern-.1667em\lower.7ex\hbox{E}\kern-.125emX}}
\usepackage{multicol}
\usepackage[algo2e]{algorithm2e}
\usepackage{adjustbox}
\received{<day> <Month>, <year>}
\revised{<day> <Month>, <year>}
\accepted{<day> <Month>, <year>}

\DeclareMathOperator*{\argmin}{arg\,min}

\newcommand{\revise}[1]{\textcolor{black}{#1}}

\usepackage[authoryear,longnamesfirst]{natbib}
\usepackage{rotating}
\usepackage{natbib}
\def\tsc#1{\csdef{#1}{\textsc{\lowercase{#1}}\xspace}}
\tsc{WGM}
\tsc{QE}

\usepackage{hyperref}
\usepackage{algorithm,algcompatible,lipsum}
\algnewcommand{\lst}{\texttt{lst}}
\algnewcommand{\slst}{\texttt{slst}}
\algnewcommand{\SEND}{\textbf{send}}
\usepackage{float}
\newsavebox{\alglefta}
\newsavebox{\algrighta}
\newsavebox{\algleftb}
\newsavebox{\algrightb}

\begin{document}

\let\WriteBookmarks\relax
\def\floatpagepagefraction{1}
\def\textpagefraction{.001}

\shorttitle{A Hyper-Transformer Model for Controllable Pareto Front Learning with Split Feasibility Constraints}

\shortauthors{Tran Anh Tuan et~al.}

\title [mode = title]{A Hyper-Transformer model for Controllable Pareto Front Learning with Split Feasibility Constraints}                      



%
\author[1]{Tran Anh Tuan}[type=editor,
                        auid=000,bioid=1,
                        orcid=0000-0001-6287-0173]



\ead{tuan.ta222171m@sis.hust.edu.vn}



\affiliation[1]{organization={Faculty of Mathematics and Informatics, Hanoi University of Science and Technology},
            city={Hanoi},
            country={Vietnam}}

\author[1]{Nguyen Viet Dung}[style=vietnamese]
\ead{dung.nv232215m@sis.hust.edu.vn}
\author[1]{Tran Ngoc Thang}[]
\cormark[1]
\ead{thang.tranngoc@hust.edu.vn}
\cortext[1]{Corresponding Author}







\begin{abstract}
Controllable Pareto front learning (CPFL) approximates the Pareto solution set and then locates a Pareto optimal solution with respect to a given reference vector. However, decision-maker objectives were limited to a constraint region in practice, so instead of training on the entire decision space, we only trained on the constraint region. Controllable Pareto front learning with Split Feasibility Constraints (SFC) is a way to find the best Pareto solutions to a split multi-objective optimization problem that meets certain constraints. In the previous study, CPFL used a Hypernetwork model comprising multi-layer perceptron (Hyper-MLP) blocks. With the substantial advancement of transformer architecture in deep learning, transformers can outperform other architectures in various tasks. Therefore, we have developed a hyper-transformer (Hyper-Trans) model for CPFL with SFC. We use the theory of universal approximation for the sequence-to-sequence function to show that the Hyper-Trans model makes MED errors smaller in computational experiments than the Hyper-MLP model.
\end{abstract}
\begin{graphicalabstract}
\centering
\includegraphics[width=16.5cm]{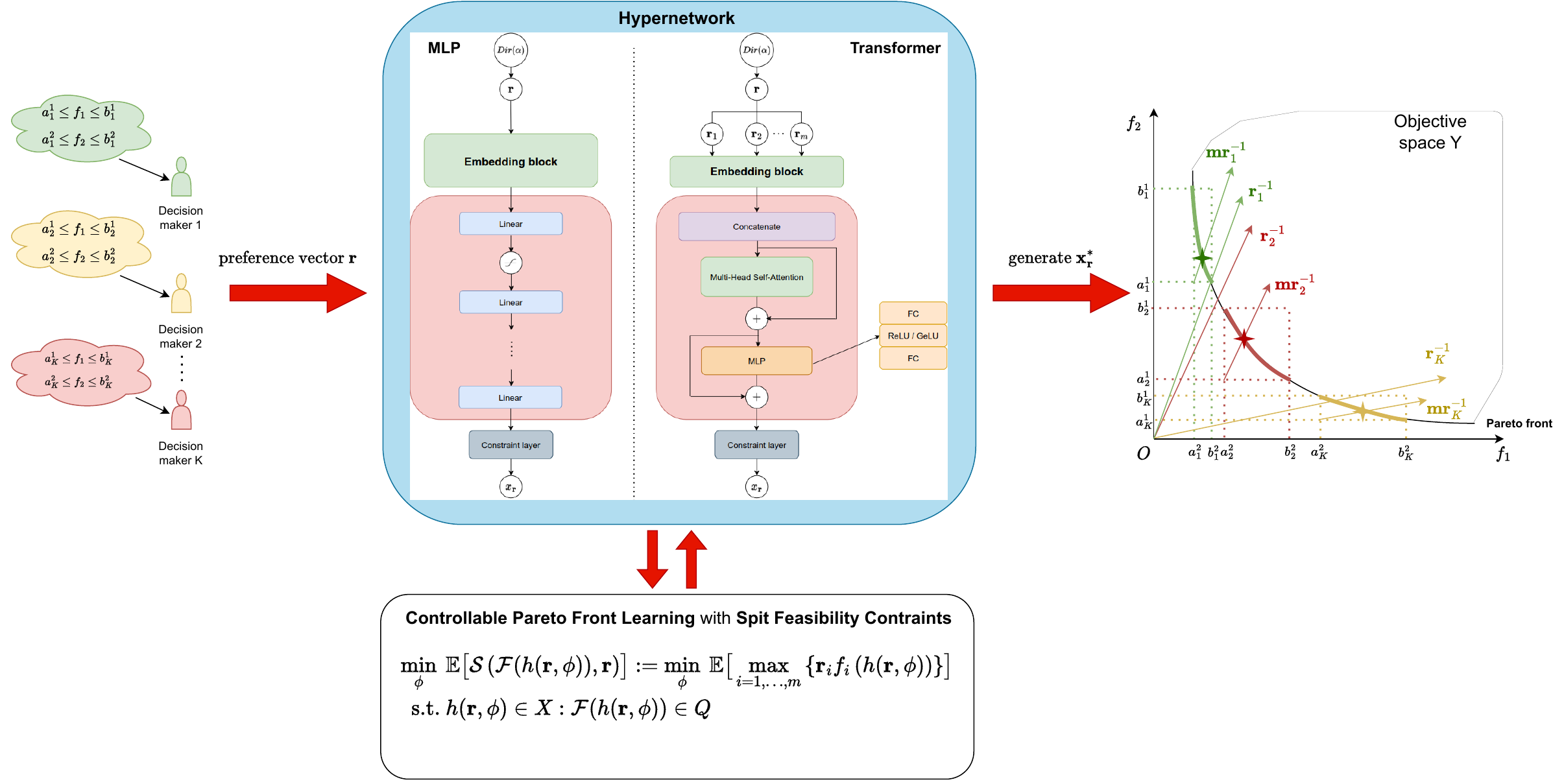}

\textbf{Graphical Abstract:} \textit{Controllable Pareto Front Learning} (CPFL) approximates the entire or part of the Pareto front and helps map a preference vector to a Pareto optimal solution, respectively. By using a deep neural network, CPFL adjusts the learning algorithm to prioritize specific objectives over others or to explore the Pareto front based on an extra criterion function. Add-in \textit{Split Feasibility Constraints} (SFC) in the training process by \textit{Hypernetwork} allows the model to precisely reflect the decision maker's preferences on the Pareto front topology.
\end{graphicalabstract}

\begin{highlights}
\item A \textbf{Hyper-Transformer} model for \textbf{Controllable PFL}.
\item \textbf{Controllable PFL} with \textbf{Split Feasibility Constraints}.
\item Learning Disconnected PF with \textbf{Joint Input} and \textbf{MoEs}.
\item Considerable experiments with \textbf{MOO} and \textbf{MTL} problems.
\end{highlights}


\maketitle

\section{Introduction}
\textbf{Multi-objective optimization (MOO)}, an advanced solution for modern optimization problems, is increasingly driven by the need to find optimal solutions in real-world situations with multiple criteria. Addressing the complex trade-offs inherent in decision-making problems resolves the challenges of simultaneously optimizing conflicting objectives on a shared optimization variable set. The advantages of MOO have been recognized in several scientific domains, including chemistry \citep{cao2019multi}, biology \citep{lambrinidis2021multi}, and finance, specifically investing \citep{vuong2023optimizing}. Specifically, its recent accomplishment in deep multitask learning \citep{sener2018multi} has attracted attention.

\textbf{Split Feasibility Problem (SFP)} is an idea that \cite{censor1994multiprojection} initially proposed. It requires locating a point in a nonempty closed convex subset in one space whose image is in another nonempty closed convex subset in the image space when subjected to a particular operator. While projection algorithms that are frequently employed have been utilized to solve SFP, they face challenges associated with computation, convergence on multiple sets, and strict conditions. The SFP is used in many real-world situations, such as signal processing, image reconstruction \citep{Stark1998vector, Byrne2003unified}, and intensity-modulated radiation therapy \citep{Censor2005multiple, Brooke2021dynamic}.

Previous methods tackled the entire Pareto front; one must incur an impracticably high cost due to the exponential growth in the number of solutions required in proportion to the number of objectives. Several proposed algorithms, such as evolutionary and genetic algorithms, aim to approximate the Pareto front or partially \citep{jangir2021elitist}. Despite the potential these algorithms have shown, only small-scale tasks \citep{ murugan2009nsga} can use them in practice.  Moreover, these methods limit adaptability because the decision-maker cannot flexibly adjust priorities in real-time. After all, the corresponding solutions are only sometimes readily available and must be recalculated for optimal performance \citep{lin2019pareto, mahapatra2021exact, pmlr-v162-momma22a}.

Researchers have raised recent inquiries regarding the approximability of the solution to the priority vector. While prior research has suggested using a hypernetwork to approximate the entire Pareto front \citep{lin2020controllable, navon2020learning,hoang2022improving}, Pareto front learning (PFL) algorithms are incapable of generating solutions that precisely match the reference vectors input into the hypernetwork. The paper on controllable Pareto front learning with complete scalarization functions \citep{tuan2023framework} explains how hypernetworks create precise connections between reference vectors and the corresponding Pareto optimal solution. The term "controllable" refers to the adjustable trade-off between objectives with respect to the reference vector. In such a way, one can find an efficient solution that satisfies his or her desired trade-off. Before our research, \citep{Raychaudhuri2022} exploited hypernetworks to achieve a controllable trade-off between task performance and network capacity in multi-task learning. The network architecture, therefore, can dynamically adapt to the compute budget variation. \cite{Chen2023} suggests a controllable multi-objective re-ranking (CMR) method that uses a hypernetwork to create parameters for a re-ranking model based on different preference weights. In this way, CMR can adapt the preference weights according to the changes in the online environment without any retraining. These approaches, however, only apply to the multi-task learning scenario and require a complicated training paradigm. Moreover, they do not guarantee the exact mapping between the preference vector from user input and the optimal Pareto point.

 \begin{figure*}[ht]
     \centering
        \begin{subfigure}[b]{0.32\textwidth}
         \centering
        \includegraphics[width=\textwidth]{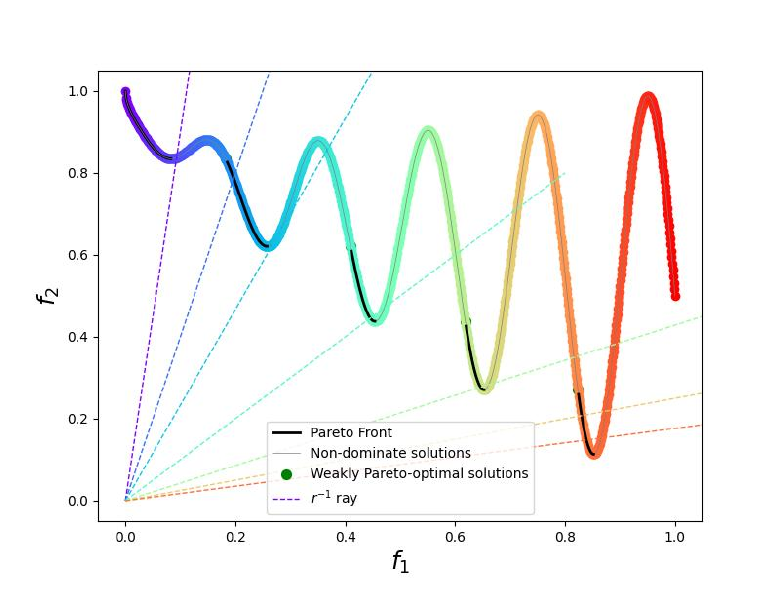}
         \caption{Pareto Front Learning}
     \end{subfigure}
     \hfill
     \begin{subfigure}[b]{0.32\textwidth}
         \centering
        \includegraphics[width=\textwidth]{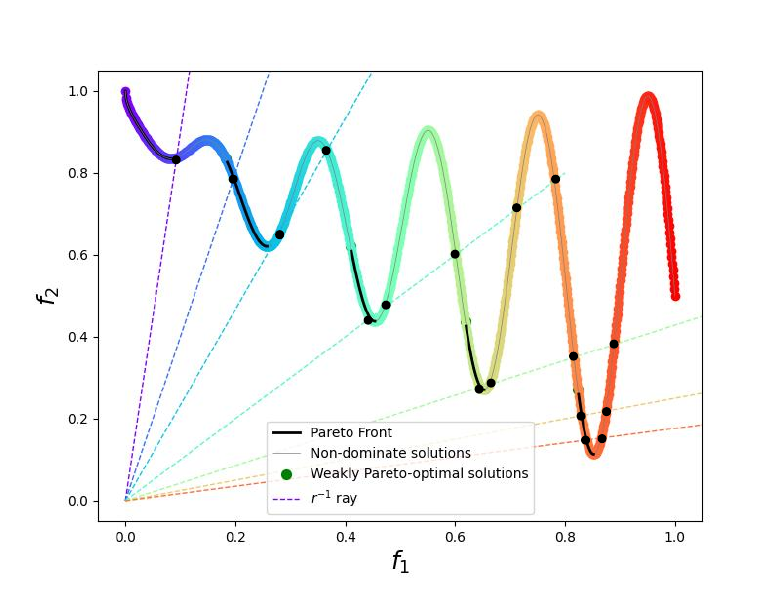}
         \caption{Controllable Pareto Front Learning}
     \end{subfigure}
     \hfill
     \begin{subfigure}[b]{0.32\textwidth}
         \centering
        \includegraphics[width=\textwidth]{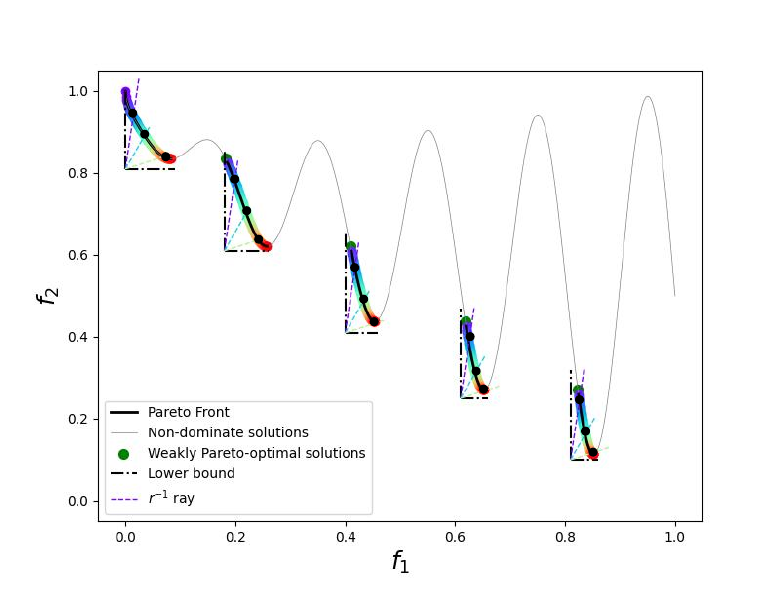}
         \caption{Controllable PFL with SFC}
     \end{subfigure}

     \caption{\textbf{\textit{Left:}} Pareto Front Learning by Hypernetwork, which is used to approximate the entire Pareto front, including non-dominated solutions. \textbf{\textit{Middle:}} Controllable Pareto Front Learning with Completed Scalarization Function uses a single Hypernetwork model, mapping any given preference vector to its corresponding solution on the Pareto front; these solutions may not be unique. \textbf{\textit{Right:}} Controllable Disconnected Pareto Front Learning with Split Feasibility Constraints by a Robust Hypernetwork helps avoid non-dominated solutions.}
\end{figure*}

Primarily, problems involving entirely connected Pareto fronts are the focus of the current research. Unfortunately, this is unrealistic in real-world optimization scenarios \citep{ishibuchi2019regular}, whereas the performance can significantly deteriorate when the PF consists of disconnected segments. If we use the most recent surrogate model's regularity information, we can see that the PFs of real-world applications are often shown as disconnected, incomplete, degenerated, and badly scaled. This is partly because the relationships between objectives are often complicated and not linear. \cite{chen2023data} proposed a data-driven EMO algorithm based on multiple-gradient descent to explore promising candidate solutions. It consists of two distinctive components: the MGD-based evolutionary search and the Infill criterion. While the D2EMO/MGD method demonstrated strong performance on specific benchmarking challenges involving unconnected PF segments, it needs more computational efficiency and flexibility to meet real-time system demands. In our research, we developed two different neural network architectures to help quickly learn about disconnected PF problems with split feasibility constraints.

The theories behind the hyper-transformers we made for controllable Pareto front learning with split feasibility constraints are well known \citep{tuan2023framework,yun2019transformers,jiang2023brief}. This is because deep learning models are still getting better. Transformers are a type of neural network architecture that has helped a lot with computer vision \citep{Dosovitskiy2020image}, time series forecasting \citep{DEIHIM2023549, SHEN2023953}, and finding models for few-shot learning research \citep{zhmoginov2022hypertransformer}. This progress is attributed mainly to their renowned self-attention mechanism \citep{vaswani2017attention}. A Transformer block has two layers: a self-attention layer and a token-wise feed-forward layer. Both tiers have skip connections. Inside the Recurrent Neural Networks (RNNs) framework, \citep{bahdanau2014neural} first introduced the attention mechanism. Later on, it was used in several actual network topologies. The attention mechanism, similar to the encoder-decoder mechanism, is a module that may be included in current models.

Our main contributions include:
\begin{itemize}
    \item In this study, we express a split multi-objective optimization problem. From there, we focus on solving the controllable Pareto front learning problem with split feasibility constraints based on scalarization theory and the split feasibility problem. In reality, when decision-makers want their goals to be within the area limited by bounding boxes, this allows them to control resources and provide more optimal criteria for the Pareto solution set.
    \item We propose a novel hypernetwork architecture based on a transformer encoder block for the controllable Pareto front learning problem with split feasibility constraints. Our proposed model shows superiority over MLP-based designs for multi-objective optimization and multi-task learning problems.
    \item We also integrate joint input and a mix of expert architectures to enhance the hyper-transformer network for learning disconnected Pareto front. This helps bring great significance to promoting other research on the controllable disconnected Pareto front of the hypernetwork.
\end{itemize}
Summarizing, the remaining sections of the paper are structured in the following manner: Section 2 will provide an overview of the foundational knowledge required for multi-objective optimization. Section 3 presents the optimization problem over the Pareto set with splitting feasibility box constraints. Section 4 describes the optimization problem over the Pareto set as a controllable Pareto front learning problem using Hypernetwork, and we also introduce a Hypernetwork based on the Transformer model (Hyper-Transformer). Section 5 explains the two fundamental models used in the Hyper-Transformer architecture within Disconnected Pareto Front Learning. Section 6 will detail the experimental synthesis, present the results, and analyze the performance of the proposed model. The last section addresses the findings and potential future endeavors. In addition, we provide the setting details and additional experiments in the appendix.
\section{Preliminaries}
Multi-objective optimization aims to find $\mathbf{x}\in X$ to optimize $m$ objective functions:
\begin{align*}\tag{MOP}\label{MOP}
\underset{\mathbf{x}\in X}{\min}\text{ }\mathcal{F}(\mathbf{x}),
\end{align*}
where $\mathcal{F}(\cdot): X\to Y\subset\mathbb{R}^m, \mathcal{F}(\mathbf{x}) = \{f_1\left(\mathbf{x}\right),\dots, f_m\left(\mathbf{x}\right) \}$, $X\subset\mathbb{R}^{n}$ is nonempty convex set, and objective functions $f_i(\cdot): \mathbb{R}^n \rightarrow \mathbb{R}$, $i= 1,\dots,m$ are convex functions and bounded below on $X$. We denote $Y:=\mathcal{F}(X)=\{\mathbf{y}\in \mathbb{R}^m |\exists \mathbf{x}\in \mathbb{R}^n,\mathcal{F}(\mathbf{x})=\mathbf{y}\}$ the outcome set or the value set of Problem \eqref{MOP}. 

\begin{definition}[Dominance] A solution $\mathbf{x}_1$ dominates $\mathbf{x}_2$ if $f_i\left(\mathbf{x}_1\right) \leq f_i\left(\mathbf{x}_2\right), \forall i$ and $f_i\left(\mathbf{x}_1\right) \neq f_i\left(\mathbf{x}_2\right)$. Denote $\mathcal{F}\left(\mathbf{x}_1\right) \prec \mathcal{F}\left(\mathbf{x}_2\right)$.
\end{definition}

\begin{definition}[Pareto optimal solution] A solution $\mathbf{x}_1$ is called Pareto optimal solution (efficient solution) if $\nexists \mathbf{x}_2: \mathcal{F}\left(\mathbf{x}_2\right) \preceq \mathcal{F}\left(\mathbf{x}_1\right)$.
\end{definition}

\begin{definition}[Weakly Pareto optimal solution] A solution $\mathbf{x}_1$ is called weakly Pareto optimal solution (weakly efficient solution) if $\nexists \mathbf{x}_2: \mathcal{F}\left(\mathbf{x}_2\right) \prec \mathcal{F}\left(\mathbf{x}_1\right)$.
\end{definition}

\begin{definition}[Pareto stationary] A point $\mathbf{x}^{*}$ is called Pareto stationary (Pareto critical point) if $\nexists d \in X:\left\langle J \mathcal{F}\left(\mathbf{x}^{*}\right), d\right\rangle<0$ or $\forall d \in X:\left\langle J \mathcal{F}\left(\mathbf{x}^{*}\right), d\right\rangle \nless 0$, corresponding:
        \begin{align*}
            \max _{i=1, \ldots, m} \nabla f_{i}\left(\mathbf{x}^{*}\right)^{\top} d \geq 0, \quad \forall d \in X,
        \end{align*}
where $J \mathcal{F}\left(\mathbf{x}^{*}\right) = \left[\nabla f_{1}(\mathbf{x}^{*})^{T},\dots,\nabla f_{m}(\mathbf{x}^{*})^{T}\right]^T$ is Jacobian matrix of $\mathcal{F}$ at $\mathbf{x}^{*}$.
\label{def:paretostationary}
\end{definition}

\begin{definition}[Pareto set and Pareto front] The set of Pareto optimal solutions is called the Pareto set, denoted by $X_E$, and the corresponding images in objectives space are Pareto outcome set $Y_E:=\{\mathbf{y}\in\mathbb{R}^m|\mathbf{y}=\mathcal{F}(\mathbf{x})\text{ for some } x\in X_{E}\}$ or Pareto front ($PF_E$).  Similarly, we can define the weakly Pareto set $X_{WE}$ and weakly Pareto outcome set $Y_{WE}$.
\end{definition}

\begin{proposition}
    $\mathbf{x}^*$ is Pareto optimal solution to Problem \eqref{MOP} $\Leftrightarrow \mathbf{x}^*$ is Pareto stationary point.
\end{proposition}
\begin{definition}\citep{mangasarian1994nonlinear}
	The differentiable function $f:\mathbb{R}^m \to \mathbb R$  is said to be 
	\begin{itemize}
		\item[$\bullet$] 		 convex on $ X $ if for all $ \mathbf{x}_1,\mathbf{x}_1\in X $, $ \lambda \in [0,1] $, it holds that
		\begin{equation*}
			f(\lambda \mathbf{x}_1 +(1-\lambda)\mathbf{x}_2)\leq \lambda f(\mathbf{x}_1)+(1-\lambda)f(\mathbf{x}_2).
		\end{equation*}
		\item[$\bullet$] 		pseudoconvex on $ X $ if for all $ \mathbf{x}_1,\mathbf{x}_2\in X $, it holds that
		\begin{equation*}
        f(\mathbf{x}_2)<f(\mathbf{x}_1) \Rightarrow \langle \nabla f(\mathbf{x}_1), \mathbf{x}_2-\mathbf{x}_1 \rangle < 0.
		\end{equation*}
	\end{itemize}
\end{definition}
Let $f$ be a numerical function defined on some open $X$ set in $\mathbb{R}^n$, let $\overline{\mathbf{x}}\in X$, and let $f$ be differentiable at $\overline{\mathbf{x}}$. If $f$ is convex at $\overline{\mathbf{x}}$, then $f$ is pseudoconvex at $\overline{\mathbf{x}}$, but not conversely \citep{mangasarian1994nonlinear}.
\begin{definition}\citep{dinh2005generalized}
    A function $\varphi$ is specified on convex set $X\subset\mathbb{R}^n$, which is called:
\begin{enumerate}
    \item nondecreasing on $X$ if $\mathbf{x}_1\succeq \mathbf{x}_2$ then $\varphi(\mathbf{x}_1)\ge \varphi(\mathbf{x}_2),$ $\forall \mathbf{x}_1,\mathbf{x}_2\in X$.
    \item weakly increasing on $X$ if $\mathbf{x}_1\succ\mathbf{x}_2$ then $\varphi(\mathbf{x}_1)\ge\varphi(\mathbf{x}_2),$ $\forall \mathbf{x}_1,\mathbf{x}_2\in X$.
    \item monotonically increasing on $X$ if $\mathbf{x}_1\succ \mathbf{x}_2$ then $\varphi(\mathbf{x}_1)>\varphi(\mathbf{x}_2),$ $\forall \mathbf{x}_1,\mathbf{x}_2\in X$.
\end{enumerate}
\end{definition}
The Pareto front's structure and optimal solution set of Problem \eqref{MOP} have been investigated by numerous authors in the field \citep{naccache1978connectedness, Luc1989,helbig1990connectedness,xunhua1994connectedness}. In certain situations, $Y_E$ is weakly connected or connected \citep{benoist2001contractibility, Luc1989}. Connectedness and contractibility are noteworthy topological properties of these sets due to their ability to enable an uninterrupted transition from one optimal solution to another along only optimal alternatives and their assurance of numerical algorithm stability when subjected to limiting processes.

\section{Multi-objective Optimization problem with Split Feasibility Constraints}
\subsection{Split Multi-objective Optimization Problem}
In 1994, \cite{censor1994multiprojection} first introduced the Split Feasibility Problem (SFP) in finite-dimensional Hilbert spaces to model inverse problems arising from phase retrievals and medical image reconstruction. In this setting, the problem is stated as follows:
\begin{align*}\tag{SFP}\label{SFP}
    \text{Find}\quad \mathbf{x}^*\in C: \mathcal{F}(\mathbf{x}^*)\in Q,
\end{align*}
where $C$ is a convex subset in $\mathbb{R}^n$, $Q$ is a convex subset in $\mathbb{R}^m$, and a smooth linear function $\mathcal{F}(\cdot): \mathbb{R}^n\to \mathbb{R}^m$. The classical linear version of the split feasibility problem takes $\mathcal{F}(x) = Ax$ for some $m\times n$ matrix $A$ \citep{censor1994multiprojection}. Other typical examples of the constraint set $Q$  are defined by the constraints $\mathcal{F}(x) = b, \|\mathcal{F}(x) - b\| \le r, \text{or } c \le \mathcal{F}(x) - b \le d,$ where $b,c,d, r \in \mathbb{R}^m$. 

Some solution methods were studied for Problem \eqref{SFP} when $C$ and/or $Q$ are solution sets of some other problems such as fixed point, optimization, variational inequality, equilibrium \citep{anh2016projection, byrne2002iterative, censor2012algorithms, lopez2012solving}. However, these works focus on the assumptions when $C$ is a convex set or $\mathcal{F}$ is linear \citep{xu2018majorization,yen2019subgradient,godwin2023image}. 

In the paper, we study Problem \eqref{SFP} where $C$ is the weakly Pareto optimal solution set of Problem \eqref{MOP}, that is
\begin{align*}\tag{SMOP}\label{SMOP}
\text{Find}\quad &\mathbf{x}^*\in X_{WE}: \mathcal{F}(\mathbf{x}^*)\in Q\\
\text{with}\quad &X_{WE} := {\text{Argmin}}\{\mathcal{F}(\mathbf{x})|\mathbf{x}\in X\}.
\end{align*}
This problem is called \textit{a split multi-objective optimization problem}. It is well known that $X_{WE}$ is, in general, a non-convex set, even in the special case when $X$ is a polyhedron and $\mathcal{F}$ is linear on $\mathbb{R}^n$ \citep{kim2013optimization}. Therefore, unlike previous studies, in this study, we consider the more challenging case of Problem \eqref{SFP}  where $C$ is a non-convex set and $\mathcal{F}$ is nonlinear. This challenge is overcome using an outcome space approach to transform the non-convex form into a convex form, in which the constraint sets of Problem \eqref{SFP} are convex sets. This will be presented in Section 3.2 below. 

\subsection{Optimizing over the solution set of Problem \eqref{SMOP}}
MOO aims to find Pareto optimal solutions corresponding to different trade-offs between objectives \citep{ehrgott2005multicriteria}. Optimizing over the Pareto set in multi-objective optimization allows us to make informed decisions when dealing with multiple, often conflicting, objectives. It's not just about finding feasible solutions but also about understanding and evaluating the trade-offs between different objectives to select the most appropriate solution based on specific criteria or preferences. In a similar vein, we consider optimizing over the Pareto set of Problem \eqref{SMOP} as follows:
\begin{align*}\tag{SP}\label{SP}
\underset{\mathbf{x}}{\min}\text{ } &\mathcal{S}\left(\mathcal{F}(\mathbf{x})\right)\\
\text{s.t. } &\mathbf{x}\in X_{WE}: \mathcal{F}(\mathbf{x})\in Q,
\end{align*}
where the function $\mathcal{S}(\cdot):Y\to\mathbb{R}$ is a monotonically increasing function and pseudoconvex on $Y$. Recall that $Y$ is the outcome set of $X$ through the function $\mathcal{F}.$ 


Following the outcome-space approach, the reformulation of Problem \eqref{SP} is given by:
\begin{align*}\tag{OSP}\label{OSP}
\min\text{ } &\mathcal{S}\left(\mathbf{y}\right)\\
\text{s.t. } &\mathbf{y}\in Y_{WE}: \mathbf{y}\in Q,
\end{align*}
where $Y_{WE}$ is the weakly Pareto outcome set of Problem \eqref{MOP}.
\begin{proposition}\label{p3.1}
Problem \eqref{SP} and Problem \eqref{OSP} are equivalent, i.e., if $\mathbf{x}^*$ is the optimal solution of Problem \eqref{SP} then $\mathbf{y}^*=\mathcal{F}(\mathbf{x}^*)$ is the optimal solution of Problem \eqref{OSP}; conversely, if $\mathbf{y}^*$ is the optimal solution of Problem \eqref{OSP} then $\mathbf{x}^*\in X$ such that $\mathcal{F}(\mathbf{x}^*) \le \mathbf{y}^*$ and $\mathcal{F}(\mathbf{x}^*) \in Q$ is the optimal solution of Problem \eqref{SP}.
\end{proposition}
\begin{proof}
Indeed, if $\mathbf{x}^* \in X_{WE}$ is a global optimal solution to Problem \eqref{SP}, then any $\mathbf{x}\in X_{WE} : \mathcal{F}(\mathbf{x})\in Q$ such that $S\left(\mathcal{F}(\mathbf{x}^*)\right)\le S\left(\mathcal{F}(\mathbf{x})\right)$. We imply $ S\left(\mathbf{y}^*\right)\le S\left(\mathbf{y}\right)$ with $\forall\mathbf{y}\in Y_{WE}:\mathbf{y}\in Q$, and $\mathbf{y}^*=\mathcal{F}(\mathbf{x}^*)$ belongs to the feasible domain of Problem \eqref{OSP}. Hence, $\mathbf{y}^*$ is the optimal solution of Problem \eqref{OSP}.

On the contrary, if $\mathbf{y}^* \in Y_{WE}$ is a global optimal solution to Problem \eqref{OSP}, then any $\mathbf{x}^* \in X$ such that $\mathcal{F}(\mathbf{x}^*) \le \mathbf{y}^*$. We imply $S\left(\mathcal{F}(\mathbf{x}^*)\right)\le S\left(\mathbf{y}^*\right)\le S\left(\mathbf{y}\right)$ with $\forall\mathbf{y}\in Y_{WE}:\mathbf{y}\in Q$, i.e, and $S\left(\mathcal{F}(\mathbf{x}^*)\right)\le S\left(\mathcal{F}(\mathbf{x})\right)$. From the definition, we have $\mathbf{x}^*$ as a global optimal solution to Problem \eqref{SP}.
\end{proof}
Let $Y^{+} = Y+\mathbb{R}^m_{+} = \{\mathbf{y}\in\mathbb{R}^m|\mathbf{y}\ge\mathbf{q} \text{ with } \mathbf{q}\in Y\}.$ When $X$ is a convex set and $\mathcal{F}$ is a nonlinear function, the image set $Y = \mathcal{F}(X)$ is not necessarily a convex set. Therefore, instead of considering the set $Y$, we consider the set $Y^+$, which is an effective equivalent set (i.e., the set of effective points of $Y$ and $Y^+$ coincide), and $Y^+$ has nicer properties; for example, $Y^+$ is a convex set. This is illustrated in Proposition \ref{p3.2}. Besides, we also define a set $G\subset \mathbb{R}^m$ is called normal if for any two points $\mathbf{x},\mathbf{x}^{'} \in\mathbb{R}^m$ such that $\mathbf{x}^{'} \le \mathbf{x}$, if $\mathbf{x}\in G$, then $\mathbf{x}^{'} \in G$. Similarly, a set $H\subset \mathbb{R}^m$ is called reverse normal if $\mathbf{x}^{'} \ge \mathbf{x}\in H $ implies $\mathbf{x}^{'} \in H$.

\begin{proposition}\citep{kim2013optimization}\label{p3.2}
 We have:
    \begin{enumerate}[label=(\roman*)]
        \item $Y_{WE} = Y^{+}_{WE}\cap Y$;
        \item $\partial Y^{+} = Y^{+}_{WE}$;
        \item $Y^{+}$ is a closed convex set and is a reverse normal set.
    \end{enumerate}
\end{proposition}
Hence, we transform Problem \eqref{OSP} into: 
\begin{align*}\tag{$\mathrm{OSP}^{+}$}\label{OSP+}
\min\text{ } &\mathcal{S}\left(\mathbf{y}\right)\\
\text{s.t. } &\mathbf{y}\in Y_{WE}^{+}: \mathbf{y}\in Q.
\end{align*}
The equivalence of problems \eqref{OSP} and \eqref{OSP+} is shown in the following Proposition \ref{p3.3}.
\begin{proposition}\label{p3.3}
    If $\mathbf{y}^*$ is the optimal solution of Problem \eqref{OSP}, then $\mathbf{y}^*$ is the optimal solution of Problem \eqref{OSP+}. Conversely, if $\mathbf{y}^*$ is the optimal solution of Problem \eqref{OSP+} and $\mathbf{q}^*\in Y_{WE}\cap Q$ such that $\mathbf{y}^* \ge \mathbf{q}^*$ then $\mathbf{q}^*$ is the optimal solution of Problem \eqref{OSP}.
\end{proposition}
\begin{proof}
If $\mathbf{y}^*$ is the optimal solution of Problem \eqref{OSP}, then $S(\mathbf{y}^*)\le S(\mathbf{y}), \forall \mathbf{y}\in Y_{WE}\cap Q$ and $\mathbf{y}^*\in Y_{WE}\cap Q$. With each of $\overline{\mathbf{y}}\in Y_{WE}^+\cap Q$, following the definition of $Y_{WE^+}$, there exists $\mathbf{y}\in Y_{WE}\cap Q$ such that $\overline{\mathbf{y}}\ge \mathbf{y}$. $S$ is a monotonically increasing function on $Y$, so $S(\overline{\mathbf{y}})\ge S( \mathbf{y})$. Hence $S(\mathbf{y}^*)\le S(\overline{\mathbf{y}}), \forall\overline{\mathbf{y}}\in Y_{WE}^+\cap Q$. Moreover, $\mathbf{y}^*\in Y_{WE}\cap Q$ means $\mathbf{y}^*\in Y_{WE}^+\cap Q$. We imply that $\mathbf{y}^*$ is the optimal solution of Problem \eqref{OSP+}.

Conversely, if $\mathbf{y}^*$ is the optimal solution of Problem \eqref{OSP+}, then $S(\mathbf{y}^*)\le S(\mathbf{y}), \forall \mathbf{y}\in Y_{WE}^+\cap Q$. Assume that there exists $\mathbf{q}^*\in Y_{WE}\cap Q$ such that $\mathbf{y}^* \ge \mathbf{q}^*$. $S$ is a monotonically increasing function on $Y$, then $S(\mathbf{q}^*) \le S(\mathbf{y}^*)\le S(\mathbf{y})$. With each of $\mathbf{y}\in Y_{WE}\cap Q$, then $\mathbf{y}\in Y_{WE}^+\cap Q$. Hence $S(\mathbf{q}^*)\le S(\mathbf{y}), \forall \mathbf{y}\in Y_{WE}\cap Q$, i.e. $\mathbf{q}^*$ is the optimal solution of Problem \eqref{OSP}.
\end{proof}
The problem \eqref{OSP+} is a difficult problem because normally, the set $Y_{WE}^+$ is a non-convex set. Thanks to the special properties of the objective functions $S$ and $Y^+$, we can transform the problem \eqref{OSP+} into an equivalent problem, where the constraint set of this problem is a convex set, as follows:
\begin{align*}\tag{$\overline{\text{OSP}}$}\label{OSP++}
\min\text{ } &\mathcal{S}\left(\mathbf{y}\right)\\
\text{s.t. } &\mathbf{y}\in Y^{+}\cap Q,
\end{align*}
 with the explicit form
\begin{align*}\tag{$\overline{\text{ESP}}$}\label{ESP++}
\underset{(\mathbf{x},\mathbf{y})}{\min}\text{ } &\mathcal{S}\left(\mathcal{F}(\mathbf{x})\right)\\
\text{s.t. } &\mathbf{x}\in X, \mathbf{y}\in Q\\
& \mathcal{F}(\mathbf{x}) \le \mathbf{y}.
\end{align*}
\begin{proposition}\label{p3.4}
Assume $Q\subset\mathbb{R}^m_{+}$ is a normal set. The optimal solution sets of Problems \eqref{OSP+} and \eqref{OSP++} are identical.
\end{proposition}
\begin{proof}
From Proposition 11 \citep{tuy2000monotonic}, the minimum of $S$ over $Y^{+}\cap Q$, if it exists, is attained on $\partial Y^+\cap Q$. Assume $\mathbf{y}^*$ is the optimal solution of Problem \eqref{OSP++}, then $\mathbf{y}^*\in\partial Y^+\cap Q$. Use Proposition \ref{p3.2}, which implies $\mathbf{y}^*\in Y_{WE}^+\cap Q$. Therefore, the optimal solution sets of Problems \eqref{OSP+} and \eqref{OSP++} are identical.
\end{proof}
\begin{proposition}\label{p3.5}
Problem \eqref{OSP++} is a pseudoconvex programming problem with respect to $\mathbf{y}$, and Problem \eqref{ESP++} is a pseudoconvex programming problem with respect to $(\mathbf{x},\mathbf{y})$.
\end{proposition}
\begin{proof}
Because each $f_i(\mathbf{x})$ is a convex function on a nonempty convex set $X$, and $Y^{+}$ is a full-dimension closed convex set. Moreover, $S$ is a monotonically increasing function and pseudoconvex on $Y$. Therefore, Problem \eqref{OSP++} is a pseudoconvex programming problem with respect to $\mathbf{y}$.

If $f_i(\mathbf{x})$ are convex functions, then $\mathcal{F}(\mathbf{x})-\mathbf{y}$ is convex constraint, and $\mathcal{S}\left(\mathcal{F}(\mathbf{x})\right)$ is a convex function on $X,Y$. Hence, $\mathcal{S}\left(\mathcal{F}(\mathbf{x})\right)$ is a convex function with respect to $(\mathbf{x},\mathbf{y})$. Furthermore, $X, Q$ are nonempty convex sets in $\mathbb{R}^n$ and $\mathbb{R}^m$, respectively. Problem \eqref{ESP++} is a pseudoconvex programming problem with respect to $(\mathbf{x},\mathbf{y})$.
\end{proof}
From Proposition \ref{p3.5}, Problem \eqref{ESP++} is a pseudoconvex programming problem. Therefore, each local minimization solution is also a global minimization solution \citep{mangasarian1994nonlinear}. So, we can solve it using gradient descent algorithms, such as \citep{adaptiveThang}, or neurodynamics methods, such as \citep{liu2022one, xu2020neurodynamic, bian2018neural}.

These methods solely assist in locating the Pareto solution associated with the provided reference vector. In numerous instances, however, we are concerned with whether the resulting solution is controllable and whether we are interested in more than one predefined direction because the trade-off is unknown before training or the decision-makers decisions vary. Designing a model that can be applied at inference time to any given preference direction, including those not observed during training, continues to be a challenge. Furthermore, the model should be capable of dynamically adapting to changes in decision-maker references. This issue is referred to as controllable Pareto front learning (CPFL) and will be elaborated upon in the following section.

\section{Controllable Pareto Front Learning with Split Feasibility Constraints}
\citep{tuan2023framework} was the first to introduce Controllable Pareto Front Learning. They train a single hypernetwork to produce a Pareto solution from a collection of input reference vectors using scalarization problem theory. Our study uses a weighted Chebyshev function based on the coordinate transfer method to find Pareto solutions that align with how DM's preferences change over time with $\mathcal{S}\left(\mathcal{F}\left(\mathbf{x}\right),\mathbf{r}\right):= \underset{i=1,\dots,m}{\max}\left\{r_i\left(f_i\left(\mathbf{x}\right)-\mathbf{a}_i\right)\right\}$. Moreover, we also consider $Q= Q_1 \times Q_2 \times\dots\times Q_m$ where $Q_i$ is a box constraint such that $f_i(\mathbf{x})\in [\mathbf{a}_i, \mathbf{b}_i], \mathbf{a}_i \ge 0, i=1,\dots, m$. From the definition of the normal set, then $Q$ is a normal set. Therefore, the controllable Pareto front learning problem is modeled in the following manner by combining the properties of split feasibility constraints:
\begin{align*}\label{LP}\tag{LP}
 \phi^* &= \argmin_{\phi} \mathbb{E}_{\mathbf{r} \sim Dir(\alpha)} \text{ }\big[\underset{i=1,\dots,m}{\max}\left\{\mathbf{r}_i\left(f_i\left(h\left(\mathbf{r}, \phi\right)\right)-\mathbf{a}_i\right)\right\}\big]\\
    \text{ s.t. } & h\left(\mathbf{r}, \phi\right) \in X \\
    &  \mathcal{F}(h\left(\mathbf{r}, \phi\right)) \le \mathbf{b},
\end{align*}
where $h: \mathcal{P} \times \mathbb{R}^{n} \rightarrow \mathbb{R}^{n}$ is a hypernetwork, and $Dir(\alpha)$ is Dirichlet distribution with concentration parameter $\alpha >0$.
\begin{theorem}\label{t4.1}
    If $\mathbf{x}^*$ is an optimal solution of Problem \eqref{SMOP}, then there exists a reference vector $\mathbf{r} \left(\mathbf{r}_i>0\right)$ such that $\mathbf{x}^*$ is also an optimal solution of Problem \eqref{LP}.
\end{theorem}
The pseudocode that solves Problem \eqref{LP} is presented in Algorithm \ref{alg:hypertrans1}. In contrast to the algorithm proposed by \cite{tuan2023framework}, our approach incorporates upper bounds on the objective function during the inference phase and lower bounds during model training. The model can weed out non-dominated Pareto solutions and solutions that do not meet the split feasibility constraints by adding upper-bounds constraints during post-processing.

Moreover, we propose building a Hypernetwork based on the Transformer architecture instead of the MLP architecture used in other studies \citep{navon2020learning,hoang2022improving,tuan2023framework}. Take advantage of the universal approximation theory of sequence-to-sequence function and the advantages of Transformer's Attention Block over traditional CNN or MLP models \citep{cordonnier2019relationship, li2021can}.

\subsection{Hypernetwork-Based Multilayer Perceptron}
We define a Hypernetwork-Based Multilayer Perceptron (Hyper-MLP) $h$ is a function of the form:
\begin{align*}\tag{Hyper-MLP}\label{Hyper-MLP}
    \mathbf{x}_{\mathbf{r}} &= h_{\text{mlp}}(\mathbf{r} ;[\boldsymbol{W}, \boldsymbol{b}])\\
    &=W^k \cdot \sigma\left(W^{k-1} \ldots \sigma\left(W^1\mathbf{a}+b^1\right)+b^{k-1}\right)\nonumber,
\end{align*}
with weights $W^i \in \mathbb{R}^{k_{i+1} \times k_i}$ and biases $b^i \in \mathbb{R}^{k_{i+1}}$, for some $k_i \in \mathbb{N}$. In addition, $\phi=[\boldsymbol{W}, \boldsymbol{b}]$ accumulates the parameters of the hypernetwork. The function $\sigma$ is a non-linear activation function, typically ReLU, logistic function, or hyperbolic tangent. An illustration is shown in Figure \ref{fig:wo_join_input}a.
\begin{theorem}\citep{cybenko1989approximation}
Let $\sigma$ be any continuous sigmoidal function. Then finite sums of the form
\begin{align*}
    g(x)=\sum_{j=1}^N \alpha_j \sigma\left(y_j^{\mathrm{T}} x+\theta_j\right),
\end{align*}
are dense in $C\left(I_n\right)$. In other words, given any $f \in C\left(I_n\right)$ and $\varepsilon>0$, there is a sum, $g(x)$, of the above form, for which
\begin{align*}
    |g(x)-f(x)|<\varepsilon \quad \text { for all } \quad x \in I_n .
\end{align*}
\end{theorem}
It has been known since the 1980s \citep{cybenko1989approximation, hornik1989multilayer} that feed-forward neural nets with a single hidden layer can approximate essentially any function if the hidden layer is allowed to be arbitrarily wide. Such results hold for a wide variety of activations, including ReLU. However, part of the recent renaissance in neural nets is the empirical observation that deep neural nets tend to achieve greater expressivity per parameter than their shallow cousins.
\begin{theorem}\citep{hanin2017approximating}\label{theorem_mlp}
    For every continuous function $f:[0,1]^{d_{\text {in }}} \rightarrow \mathbb{R}^{d_{\text {out }}}$ and every $\varepsilon>0$ there is a Hyper-MLP $h$ with ReLU activations, input dimension $d_{\text {in }}$, output dimension $d_{\text {out }}$,  hidden layer widths $d_{\text {in }} = d_1, d_2,\dots,d_k,d_{k+1}=d_{\text {out}}$ that $\varepsilon$-approximates $f$ :
    \begin{align*}
        \sup _{\mathbf{x} \in[0,1]^{d_{\text {in }}}}\left\|f(\mathbf{x})-h(\mathbf{r})\right\| \leq \varepsilon.
    \end{align*}
\end{theorem}

\subsection{ Hypernetwork-Based Transformer block}
The paper \cite{yun2019transformers} gave a clear mathematical explanation of contextual mappings and showed that multi-head self-attention layers can correctly calculate contextual mappings for input sequences. They show that the capacity to calculate contextual mappings and the value mapping capability of the feed-forward layers allows transformers to serve as universal approximators for any permutation equivariant sequence-to-sequence function. There are well-known results for approximation, like how flexible Transformer networks are at it \citep{yun2019transformers}. Its sparse variants can also universally approximate any sequence-to-sequence function \citep{yun2020n}.

A transformer block is a sequence-to-sequence function mapping $\mathbb{R}^{d \times n}$ to $\mathbb{R}^{d \times n}$. It consists of two layers: a multi-head self-attention layer and a token-wise feed-forward layer, with both layers having a skip connection. More concretely, for an input $\mathbf{r} \in \mathbb{R}^{d \times m}$ consisting of $d$-dimensional embeddings of $m$ tasks, a Transformer block with multiplicative or dot-product attention \citep{luong2015effective} consists of the following two layers. We propose a hypernetwork-based transformer block (Hyper-Trans) as follows:
\begin{align*}\tag{Hyper-Trans}\label{Hyper-Trans}
    \mathbf{x}_{\mathbf{r}} &= h_{\text{trans}}(\mathbf{r};\phi)=\operatorname{MultiHeadAttn}(\mathbf{r})+\operatorname{MLP}(\mathbf{r}),
\end{align*}
with:
\begin{align*}
    \operatorname{MultiHeadAttn}(\mathbf{r}) & =\mathbf{r}+\sum_{i=1}^h \boldsymbol{W}_O^i \boldsymbol{W}_V^i \mathbf{r} \cdot \sigma\left[\left(\boldsymbol{W}_K^i \mathbf{r}\right)^T \boldsymbol{W}_Q^i \mathbf{r}\right], \\
    \operatorname{MLP}(\mathbf{r}) & =\boldsymbol{W}_2 \cdot \operatorname{ReLU}\left(\boldsymbol{W}_1 \cdot \operatorname{MultiHeadAttn}(\mathbf{r}) +\boldsymbol{b}_1 \mathbf{1}_n^T\right)+\boldsymbol{b}_2 \mathbf{1}_n^T,\\
\end{align*}
where $\boldsymbol{W}_O^i \in \mathbb{R}^{d \times k}, \boldsymbol{W}_V^i, \boldsymbol{W}_K^i, \boldsymbol{W}_Q^i \in \mathbb{R}^{k \times d}, \boldsymbol{W}_2 \in \mathbb{R}^{d \times r}, \boldsymbol{W}_1 \in \mathbb{R}^{r \times d}, \boldsymbol{b}_2 \in \mathbb{R}^d, \boldsymbol{b}_1 \in \mathbb{R}^r$, and $\operatorname{MLP}(\mathbf{r})$ is multilayer perceptron block with a ReLU activation function. Additionally, we can also replace the ReLU function with the GeLU function. The number of heads $e$ and the head size $k$ are two main parameters of the attention layer, and $l$ denotes the hidden layer size of the feed-forward layer.

We would like to point out that our definition of the Multi-Head Attention layer is the same as \cite{vaswani2017attention}, in which they combine attention heads and multiply them by a matrix $\boldsymbol{W}_O \in \mathbb{R}^{d \times k e}$. One difference in our setup is the absence of layer normalization, which simplifies our analysis while preserving the basic architecture of the transformer.

We define transformer networks as the composition of Transformer blocks. The family of the sequence-to-sequence functions corresponding to the Transformers can be defined as:
\begin{align*}
    \mathcal{T}^{e, k, l}:=\left\{h\right\},
\end{align*}
where $h: \mathbb{R}^{d \times m} \rightarrow \mathbb{R}^{d \times m}$ is a composition of Transformer blocks $t^{e, k, l}: \mathbb{R}^{d \times m} \rightarrow \mathbb{R}^{d \times m}$ denotes a Transformer block defined by an attention layer with $e$ heads of size $k$ each, and a feed-forward layer with $l$ hidden nodes. An illustration is shown in Figure \ref{fig:wo_join_input}b.
\begin{theorem}\citep{yun2019transformers}\label{theorem_trans}
    Let $\mathbb{F}$ be the sequence-to-sequence function class, which consists of all continuous permutation equivariant functions with compact support that map $\mathbb{R}^{d \times m} \rightarrow \mathbb{R}^{d \times m}$. For $1 \leq p<\infty$ and $\epsilon>0$, then for any given $f \in \mathbb{F}$, there exists a Transformer network $h \in \mathcal{T}^{2,1,4}$, such that:
    \begin{align*}
        \mathrm{d}_p\left(h, f\right):=\left(\int\left\|h(\mathbf{r})-f\right\|_p^p d \mathbf{r}\right)^{1 / p} < \epsilon.
    \end{align*}
\end{theorem}
\subsection{Solution Constraint layer}
In many real-world applications, there could be constraints on the solution structure $\mathbf{x}$ across all preferences. The hypernetwork model can properly handle these constraints for all solutions via constraint layers.

We first begin with the most common constraint: that the decision variables are explicitly bounded. In this case, we can simply add a transformation operator to the output of the hyper-network:
\begin{align*}
    \mathbf{x}_{\mathbf{r}}=c\left(h\left(\mathbf{r}, \phi\right)\right),
\end{align*}
where $h\left(\mathbf{r}, \phi\right)$ is Hypernetwork and $c: \mathbb{R}^n \rightarrow \mathbb{R}^n$ is an activation function that maps arbitrary model output $h\left(\mathbf{r}; \cdot\right) \in(-\infty, \infty)^n$ into the desired bounded range. The activation function should be differentiable, and hence, we can directly learn the bounded hypernetwork by the gradient-based method proposed in the main paper. We introduce three typical bounded constraints and the corresponding activation functions in the following:

\textbf{Non-Negative Decision Variables.} We can set $c(\cdot)$ as the rectified linear function (ReLU):
\begin{align*}
    \mathbf{x}_{\mathbf{r}} = c(\mathbf{x}_{\mathbf{r}})=\max \{0, \mathbf{x}_{\mathbf{r}}\}.
\end{align*}
Which will keep the values for all non-negative inputs and set the rest to 0. In other words, all the output of hypernetwork will now be in the range $[0, \infty)$.

\textbf{Box-bounded Decision Variables.} We can set $c(\cdot)$ as the sigmoid function:
\begin{align*}
   \mathbf{x}_{\mathbf{r}} = c(\mathbf{x}_{\mathbf{r}})=\frac{1}{1+e^{-\mathbf{x}_{\mathbf{r}}}}.
\end{align*}

Now, all the decision variables will range from 0 to 1. It is also straightforward to other bounded regions with arbitrary upper and lower bounds for each decision variable.

\textbf{Simplex Constraints.} In some applications, a fixed amount of resources must be arranged for different agents or parts of a system. We can use the Softmax function $c(\cdot)$ where
\begin{align*}
   \mathbf{x}_{\mathbf{r}} = c(\mathbf{x}_{\mathbf{r}})=\frac{e^{\mathbf{x}^{i}_{\mathbf{r}}}}{\sum_{j=1}^n e^{\mathbf{x}^{i}_{\mathbf{r}}}}, \forall i \in\{1,2, \cdots, n\},
\end{align*}
such that all the generated solutions are on the simplex $\left\{\mathbf{x} \in \mathbb{R}^n \mid \sum_{i=1}^n \mathbf{x}_i=1\right.$ and $\mathbf{x}_i>0$ for $i=$ $1, \ldots, n\}$.

These bounded constraints are for each individual solution. With the specific activation functions, all (infinite) generated solutions will always satisfy the structure constraints, even for those with unseen contexts and preferences. This is also an anytime feasibility guarantee during the whole optimization process.
\begin{algorithm}[H]
\hsize=\textwidth 
\KwIn{Init $\phi_0,t=0$, $\mathbf{a}, \alpha$.}
\KwOut{$\phi^*$.}

\While{not converged}{

    $\mathbf{r}_t = Dir(\alpha)$

    \eIf{MLP}{
    
    $\mathbf{x}_{\mathbf{r}_t} =  h_{\text{mlp}}\left(\mathbf{r}_t, \phi\right)$}
    {

    $\mathbf{x}_{\mathbf{r}_t} =  h_{\text{trans}}\left(\mathbf{r}_t, \phi\right)$
    }
    $\phi_{t+1} = \phi_t - \xi\nabla_{\phi}\mathcal{S}\left(\mathcal{F}\left(\mathbf{x}_{\mathbf{r}_t}\right),\mathbf{r}_t,\mathbf{a}\right)$
    
    $t=t+1$

}
$\phi^* = \phi_t$
\caption{: Hypernetwork training for Connected Pareto Front.}
\label{alg:hypertrans1}
\end{algorithm}
\begin{figure}[b]
    \centering
    \includegraphics[scale=0.5]{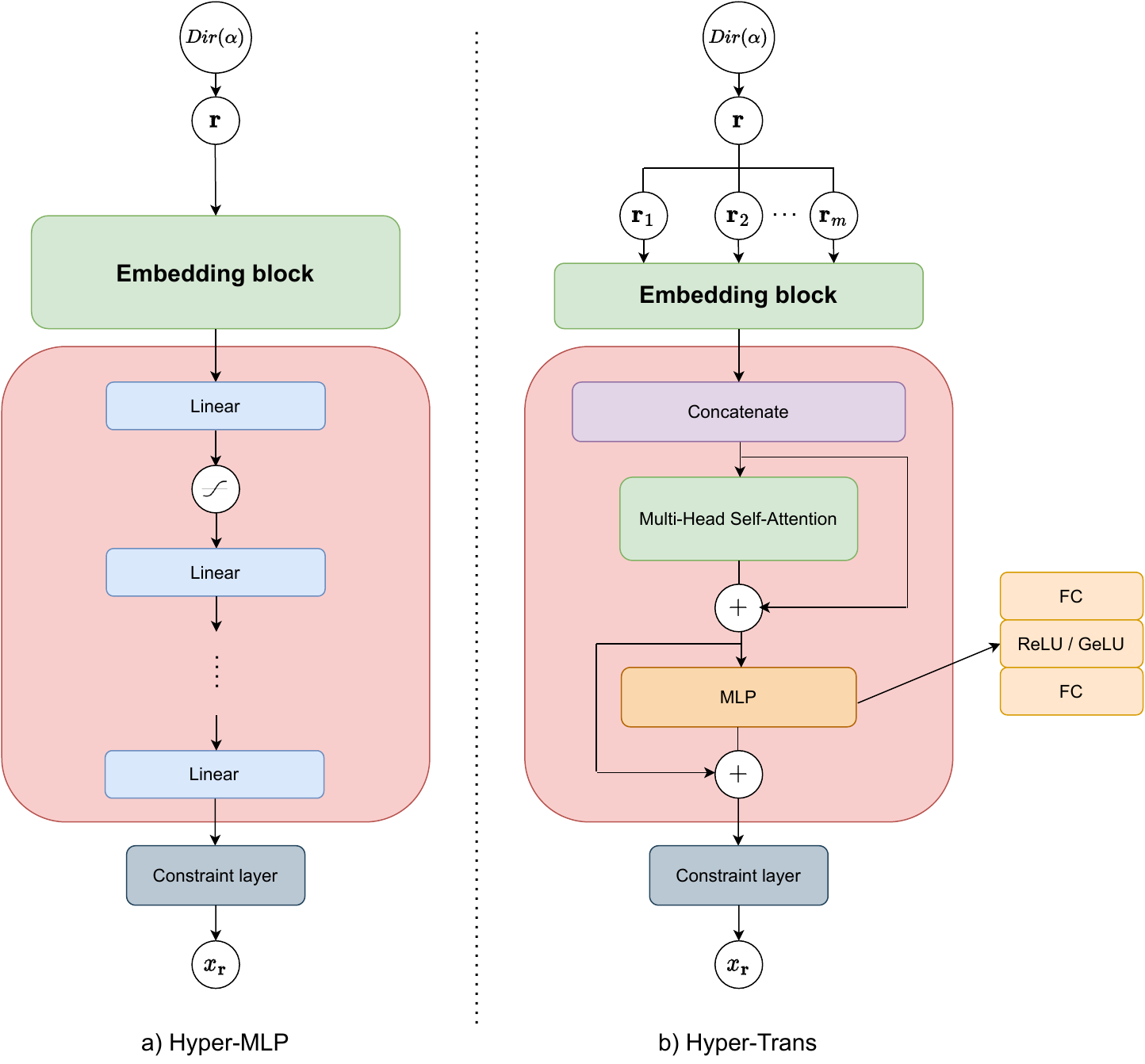}
    \caption{Hyper-MLP (\textit{left}) receives an input reference vector, Hyper-Trans (\textit{right}) receives each coordinate of the input reference vector and outputs the corresponding Pareto optimal solution.}
    \label{fig:wo_join_input}
\end{figure}
\begin{theorem}\label{t4.4}
    Let neural network $h$ be a set of multilayer perceptron or transformer blocks with $\sigma$ activation. Assume that $\phi^*$ is stationary point of Algorithm \ref{alg:hypertrans1} and $\nabla_{\phi}\mathbf{x}(\hat{\mathbf{r}};\phi^*)\neq 0$. Then $\mathbf{x}\left(\hat{\mathbf{r}}\right) = h\left(\hat{\mathbf{r}},\phi^*\right)$ is a global optimal solution to Problem \eqref{LP}, and there exists a neighborhood $U$ of $\hat{\mathbf{r}}$ and a smooth mapping $\mathbf{x}(\mathbf{r})$ such that $\mathbf{x}\left(\mathbf{r}^*\right)_{\mathbf{r}^*\in U}$ is also a global optimal solution to Problem \eqref{LP}. 
\end{theorem}
\begin{proof}
    Assume that $\mathbf{x}\left(\hat{\mathbf{r}}\right)$ is not a local optimal solution to Problem \eqref{LP}. Indeed, by using universal approximation Theorem \ref{theorem_mlp} and \ref{theorem_trans}, we can approximate smooth function $\mathbf{x}\left(\hat{\mathbf{r}}\right)$ by a network $h\left(\hat{\mathbf{r}},\phi^*\right)$. Since $\mathcal{S}\left(\mathcal{F}\left(\mathbf{x}\right),\mathbf{r}\right) := \underset{i=1,\dots,m}{\max}\left\{r_i\left(f_i\left(\mathbf{x}\right)-\mathbf{a}_i\right)\right\}$ is pseudoconvex on $X$ and $\nabla_{\phi}\mathbf{x}(\hat{\mathbf{r}};\phi^*)\neq 0$, we imply: 
    \begin{align}\label{eq1}
        \exists \mathbf{x}^{'}\in X,\mathbf{x}^{'}\neq \mathbf{x}: S(\mathbf{x}^{'})<S(\mathbf{x})\Rightarrow \left[
        \begin{matrix}
        \nabla S(\mathbf{x})(\mathbf{x}^{'}-\mathbf{x})\nabla_{\phi}\mathbf{x}(\hat{\mathbf{r}};\phi^*)<0 & \\
        \nabla S(\mathbf{x})(\mathbf{x}^{'}-\mathbf{x})\nabla_{\phi}\mathbf{x}(\hat{\mathbf{r}};\phi^*)>0 & \\
        \end{matrix}
        \right..
    \end{align}
 Besides $\phi^*$ is stationary point of Algorithm \ref{alg:hypertrans1}, hence:
    \begin{align*}
        \nabla S(\mathbf{x})\nabla_{\phi}\mathbf{x}(\hat{\mathbf{r}};\phi^*) = 0,
    \end{align*}
    then:
    \begin{align}\label{eq2}
        \nabla_{\mathbf{x}} S(\mathbf{x}) = 0.
    \end{align}
 Combined with $\mathbf{x}^{'}\neq \mathbf{x}$, we have:
    \begin{align}\label{eq3}
        \nabla S(\mathbf{x})\nabla_{\phi}\mathbf{x}(\hat{\mathbf{r}};\phi^*)(\mathbf{x}^{'}-\mathbf{x}) = 0.
    \end{align}
From \eqref{eq1}, \eqref{eq2}, and \eqref{eq3}, we have $\mathbf{x}\left(\hat{\mathbf{r}}\right) = h\left(\hat{\mathbf{r}},\phi^*\right)$ is a stationary point or a local optimal solution to Problem \eqref{LP}. With $S$ is pseudoconvex on $X$, then $\mathbf{x}\left(\hat{\mathbf{r}}\right)$ is a global optimal solution to Problem \eqref{LP} \citep{mangasarian1994nonlinear}. We choose any $\mathbf{r}^{*}\in U$ that is neighborhood of $\hat{\mathbf{r}}$, i.e. $\mathbf{r}^*\in\mathcal{P}$. Reiterate the procedure of optimizing Algorithm \ref{alg:hypertrans1} we have $\mathbf{x}\left(\mathbf{r}^*\right)$ is a global optimal solution to Problem \eqref{LP}. 

\end{proof}
\begin{remark}
    Via Theorem \ref{t4.4}, we can see that the optimal solution set of Problem \eqref{SMOP} can be approximated by Algorithm \ref{alg:hypertrans1}. From Theorem \ref{t4.1}, it guarantees that any reference vector $\mathbf{r} \left(\mathbf{r}_i>0\right)$ of Dirichlet distribution $Dir(\alpha)$ always generates an optimal solution of Problem \eqref{SMOP} such that split feasibility constraints. Then the Pareto front is also approximated accordingly by mapping $\mathcal{F}(\cdot)$ respectively.
\end{remark}

\section{Learning Disconnected Pareto Front with Hyper-Transformer network}
The PF of some MOPs may be discontinuous in real-world applications due to constraints, discontinuous search space, or complicated shapes. Existing methods are mostly built upon a revolutionary searching algorithm, which requires massive computation to give acceptable solutions. In this work, we introduce two transformer-based methods to effectively learn the irregular Pareto Front, which we shall call Hyper-Transformer with Joint Input and Hyper-Transformer with Mixture of Experts.
\subsection{Hyper-Transformer with Joint Input}
However, to guarantee real-time and flexibility in the system, we re-design adaptive model joint input for split feasibility constraints as follows:
\begin{align*}\label{Joint-PHN-VOP}\tag{Joint-Hyper-Trans}
 \phi^* &= \argmin_{\phi} \underset{\substack{\mathbf{r} \sim Dir(\alpha) \\ \mathbf{a}\sim U(0,1)}}{\mathbb{E}} \text{ }\big[\mathcal{S}\left(\mathcal{F}\left(h_{\text{trans-joint}}\left(\mathbf{r},\mathbf{a}, \phi\right)\right),\mathbf{r},\mathbf{a}\right)\big]\\
    \text{ s.t. } & h_{\text{trans-joint}}\left(\mathbf{r},\mathbf{a}, \phi\right)\in X\\
    & \mathcal{F}\left(h_{\text{trans-joint}}\left(\mathbf{r},\mathbf{a}, \phi\right)\right) \le \mathbf{b},
\end{align*}
where $U(0,1)$ is uniform distribution. 
\subsection{Hyper-Transformer with Mixture of Experts}
Despite achieving notable results in the continuous Pareto front, the joint input approach fails to achieve the desired MED in the discontinuous scenario. We, therefore, integrate the idea from the mixture of experts \citep{Shazeer2017OutrageouslyLN} into the transformer-based model and assume that each Pareto front component will be learned by one expert. 

In its simplest form, the MoE consists of a set of $k$ experts (neural networks) $e_i: \mathcal{X} \to \mathbb{R}^u, i \in \{1,2,\dots,k\}$, and a gate $g: \mathcal{X} \to \mathbb{R}^n$ that assigns weights to the experts. The gate's output is assumed to be a probability vector, i.e., $g(x) \geq 0$ and $\sum_i^{k} g(x)_i=1$, for any $x \in \mathcal{X}$. Given an example $x \in \mathcal{X}$, the corresponding output of the MoE is a weighted combination of the experts:
\begin{align*}
    \sum_i^n e_i(x)g(x)_i.
\end{align*}
In most settings, the experts are usually MLP modules and the gate $g$ is chosen to be a \textit{Softmax} gate, and then the top-$k$ expert with the highest values will be chosen to process the inputs associated with the corresponding value. As shown in Figure \ref{fig:w_join_input}b, our model takes $r_i, i=1,2,\dots, m$ as input, the corresponding reference vector for $i$th constraints. We follow the same architecture design for the expert networks but omit the gating mechanism by fixing the routing of $r_i$ to the $i$th expert. This allows each expert to specialize in a certain region of the image space in which may lie a Pareto front. By this setting, our model resembles a multi-model approach but has much fewer parameters and is simpler. 

We adapt this approach for Hyper-Transformer as follows:
\begin{align*}\label{Expert-PHN-VOP}\tag{Expert-Hyper-Trans}
 \phi^* &= \argmin_{\phi} \underset{\mathbf{r} \sim Dir(\alpha)}{\mathbb{E}} \text{ }\big[\mathcal{S}\left(\mathcal{F}\left(_{\text{trans-expert}}\left(\mathbf{r},ID, \phi\right)\right),\mathbf{r},\mathbf{a}[ID]\right)\big]\\
    \text{ s.t. } &h_{\text{trans-expert}}\left(\mathbf{r},ID, \phi\right)\in X\\
    & \mathcal{F}\left(_{\text{trans-expert}}\left(\mathbf{r},ID, \phi\right)\right) \le \mathbf{b},
\end{align*}
where $h_{\text{trans-expert}}\left(\mathbf{r},ID, \phi\right) = \bigg[\displaystyle\sum_{i=1}^{k} \text{MLP}_i\left(h_{\text{trans}}(\mathbf{r};\phi)\right)\bigg]g(ID)$ with $g(ID) = (0_1,\dots,1_{ID},\dots, 0_k)$.

Hypernetwork $h$ with architecture corresponding to Joint Input and Mixture of Experts was illustrated in Figures \ref{fig:w_join_input}a and \ref{fig:w_join_input}b.
\begin{algorithm}[H]
\KwIn{Init $\phi_0,t=0$, $\mathbf{idxs} = [0,\dots,k], \mathbf{a},\alpha$.}
\KwOut{$\phi^*$.}
\For{$ID$ in $\mathbf{idxs}$} {

\While{not converged}{
    
    $\mathbf{a}_t = \mathbf{a}[ID]$

    $\mathbf{r}_t = Dir(\alpha)$

    \eIf{Joint input}{
        $\mathbf{x}_{\mathbf{r}_t} =  h_{\text{trans-joint}}\left(\mathbf{r}_t,\mathbf{a}_t, \phi_t\right)$}
      {
      
      $\mathbf{x}_{\mathbf{r}_t} =  h_{\text{trans-expert}}\left(\mathbf{r}_t,ID, \phi_t\right)$
      
      }
    
    $\phi_{t+1} = \phi_t - \xi\nabla_{\phi}\mathcal{S}\left(\mathcal{F}\left(\mathbf{x}_{\mathbf{r}_t}\right),\mathbf{r}_t,\mathbf{a}_t\right)$
    
    $t=t+1$

}
}
$\phi^* = \phi_t$
\caption{: Hyper-Transformer training for Disconnected Pareto Front.}
\label{alg:hypertrans2}
\end{algorithm}
\begin{figure}[h]
    \centering
    \includegraphics[scale=0.5]{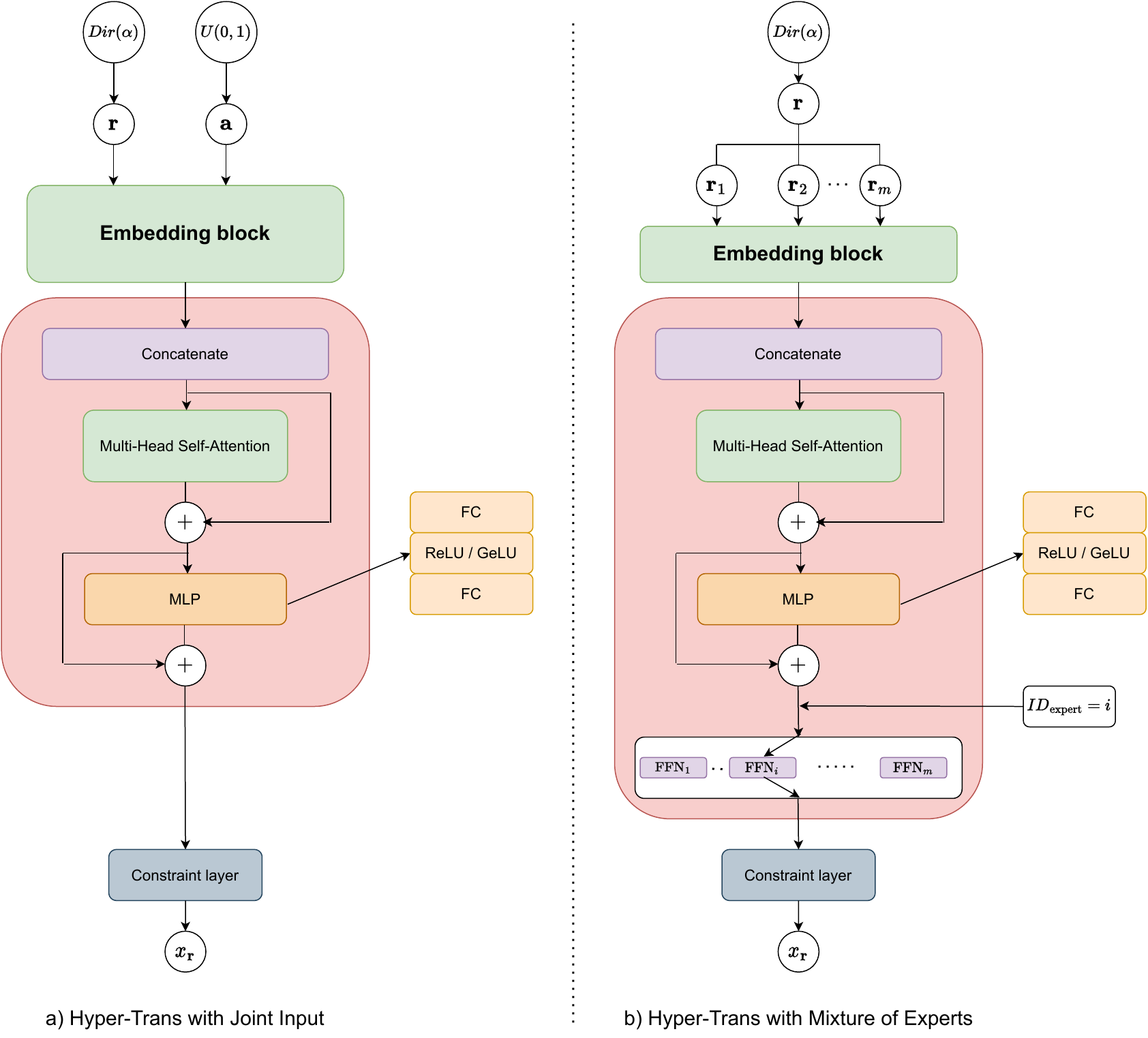}
    \caption{Proposed Transformer-based Hypernetwork. \textit{Left}: The Joint Input model takes reference vectors and objective function's lower bounds corresponding to each Pareto front component. \textit{Right}: Mixture of Experts integrated model which inputs reference vectors.}
    \label{fig:w_join_input}
\end{figure}

\section{Computational experiments}
The code is implemented in Python language programming and the Pytorch framework \citep{paszke2019pytorch}. We compare the performance of our method with the baseline method \citep{tuan2023framework} and provide the setting details and additional experiments in Appendices \ref{appendixA} and \ref{appendixB}.
\subsection{Evaluation metrics}

\textbf{Mean Euclid Distance (MED).} How well the model maps preferences to the corresponding Pareto optimal solutions on the Pareto front serves as a measure of its quality. To do this, we use the Mean Euclidean Distance (MED) (\citep{tuan2023framework}) between the truth corresponding Pareto optimal solutions $\mathcal{F}^* = \{f (\mathbf{x}^*_\mathbf{r})¿$ and the learned solutions $\hat{\mathcal{F}} = \{f(h(\mathbf{r};\phi)\}$.
\begin{align*}
    MED(\mathcal{F}^*,\hat{\mathcal{F}}) = \dfrac{1}{|\mathcal{F}^*|}\left(\sum_{i=1}^{|\mathcal{F}^*|}\left\Vert \mathcal{F}^*_i - \hat{\mathcal{F}}_i\right\Vert_2\right).
\end{align*}

\textbf{Hypervolume (HV).} Hypervolume \citep{zitzler1999multiobjective} is the area dominated by the Pareto front. Therefore, the quality of a Pareto front is proportional to its hypervolume. Given a set of $k$ points $\mathcal{M} = \{m^j | m^j \in \mathbb{R}^m; j=1,\dots, k\}$ and a reference point $\rho\in\mathbb{R}^m_{+}$. The Hypervolume of $\mathcal{S}$ is measured by the region of non-dominated points bounded above by $m \in \mathcal{M}$, and then the hypervolume metric is defined as follows:
\begin{align*}
    HV(S) = VOL\left(\underset{m \in \mathcal{M}, m \prec \rho}{\bigcup}\displaystyle{\Pi_{i=1}^m}\left[m_i,\rho_i\right]\right).
\end{align*}

\textbf{Hypervolume Difference(HVD).} The area dominated by the Pareto front is known as Hypervolume. The higher the Hypervolume, the better the Pareto front quality. For evaluating the quality of the learned Pareto front, we employ Hypervolume Difference (HVD) between the Hypervolumes computed by the truth Pareto front $\mathcal{F}$ and the learned Pareto front $\hat{\mathcal{F}}$ as follows:
\begin{align*}
HVD(\mathcal{F}^*,\hat{\mathcal{F}}) = HV(\mathcal{F}^*) - HV(\hat{\mathcal{F}}).
\end{align*}





\subsection{Synthesis experiments}
We utilized a widely used synthesis multi-objective optimization benchmark problem in the following to evaluate our proposed method with connected and disconnected Pareto front. For ease of test problems, we normalize the PF to $[0, 1]^m.$ Our source code is available at \url{https://github.com/tuantran23012000/CPFL_Hyper_Transformer.git}.
\subsubsection{Problems with Connected Pareto Front}

\noindent \textbf{CVX1} \citep{tuan2023framework}:
\begin{align*}\label{CVX1}\tag{CVX1} 
\min & \left\{x,(x-1)^2\right\}\\
\text{s.t. } & 0\le x\le 1.
\end{align*}

\noindent \textbf{CVX2} \citep{binh1997mobes}:
\begin{align*}\label{CVX2}\tag{CVX2}
    \min & \left\{f_1,f_2\right\}\\
\text{s.t. } & x_i\in[0,5], i=1,2
\end{align*}
where \begin{align*}
    f_1 = \dfrac{x_1^2 + x_2^2}{50}, \textbf{ } f_2 = \dfrac{(x_1-5)^2 + (x_2-5)^2}{50}.
\end{align*}

\noindent \textbf{CVX3} \citep{thang2020monotonic}:
\begin{align*}\label{CVX3}\tag{CVX3}
    \min & \left\{f_1,f_2,f_3\right\}\\
\text{s.t. } & x_{1}^2+x_{2}^2+x_3^2 = 1\\
& x_i\in[0,1], i=1,2,3
\end{align*}
where \begin{align*}
    & f_1 = \dfrac{x_1^2 + x_2^2 + x_3^2 +x_2 - 12x_3 + 12}{14},\\
    & f_2 = \dfrac{x_1^2 + x_2^2 + x_3^2 + 8x_1 - 44.8x_2 + 8x_3 + 44}{57},\\
    & f_3 = \dfrac{x_1^2 + x_2^2 + x_3^2 - 44.8x_1 + 8x_2 + 8x_3 + 43.7}{56}.
\end{align*}
Moreover, we experiment with the additional Non-Convex MOO problems, including ZDT1-2 \citep{zitzler2000comparison}, and DTLZ2 \citep{deb2002scalable}.

\noindent \textbf{ZDT1} \citep{zitzler2000comparison}: It is a classical multi-objective optimization benchmark problem with the form:
\begin{align*}\label{ZDT1}\tag{ZDT1}
    & f_1(\mathbf{x})=x_1 \\
& f_2(\mathbf{x})=g(\mathbf{x})\left[1-\sqrt{f_1(\mathbf{x}) / g(\mathbf{x})}\right],
\end{align*}
where $g(\mathbf{x})=1+\frac{9}{n-1} \sum_{i=1}^{n-1}x_{i+1} \text { and } 0 \leq x_i \leq 1 \text { for } i=1, \ldots, n$.

\noindent \textbf{ZDT2} \citep{zitzler2000comparison}: It is a classical multi-objective optimization benchmark problem with the form:
\begin{align*}\label{ZDT2}\tag{ZDT2}
    & f_1(\mathbf{x})=x_1 \\
    & f_2(\mathbf{x})=g(\mathbf{x})\left(1-\left(f_1(\mathbf{x}) / g(\mathbf{x})\right)^2\right),
\end{align*}
where $g(\mathbf{x})=1+\frac{9}{n-1} \sum_{i=1}^{n-1}x_{i+1} \text { and } 0 \leq x_i \leq 1 \text { for } i=1, \ldots, n$.

\noindent \textbf{DTLZ2} \citep{deb2002scalable}: It is a classical multi-objective optimization benchmark problem in the form:
\begin{align*}\label{DTLZ2}\tag{DTLZ2}
    & f_1(\mathbf{x})=(1+g(\mathbf{x})) \cos \frac{\pi x_1}{2} \cos \frac{\pi x_2}{2} \\
& f_2(\mathbf{x})=(1+g(\mathbf{x})) \cos \frac{\pi x_1}{2} \sin \frac{\pi x_2}{2} \\
& f_3(\mathbf{x})=(1+g(\mathbf{x})) \sin \frac{\pi x_1}{2} 
\end{align*}
where $g(\mathbf{x})=\sum_{i=1}^{n-2}\left(x_{i+2}\right)^2 \text { and } 0 \leq x_i \leq 1 \text { for } i=1, \ldots, n$. 

The statistical comparison results of the connected Pareto front problems of the MED scores between our proposed Hyper-Trans and the Hyper-MLP \citep{tuan2023framework} are given in Table \ref{wo_join_input}. These results show that the MED scores of Hyper-Trans are statistically significantly better than those of Hyper-MLP in all comparisons. The state trajectories of $\mathcal{F}(x)$ are shown in Figure \ref{fig:connect_trajectory}. These were calculated using Hyper-Transformer and Hyper-MLP for 2D problems where $\mathbf{x}$ is generated at $\mathbf{r} = [0.5,0.5]$ and for 3D problems where $\mathbf{x}$ is generated at $\mathbf{r} = [0.4,0.3,0.3]$. The number of iterations needed to train the model goes up, and the fluctuation amplitude of the objective functions produced by the hyper-transformer goes down compared to the best solution.

\subsubsection{Problems with Disconnected Pareto Front}
\noindent \textbf{ZDT3} \citep{zitzler2000comparison}: It is a classical multi-objective optimization benchmark problem with the form:
\begin{align*}\label{ZDT3}\tag{ZDT3}
    & f_1(\mathbf{x})=x_1 \\
    & f_2(\mathbf{x})=g(\mathbf{x})\left(1-\sqrt{\left(f_1(\mathbf{x}) / g(\mathbf{x})\right)} - \left(f_1(\mathbf{x}) / g(\mathbf{x})\right)\sin{10\pi f_1}\right),
\end{align*}
where $g(\mathbf{x})=1+\frac{9}{n-1} \sum_{i=1}^{n-1}x_{i+1} \text { and } 0 \leq x_i \leq 1 \text { for } i=1, \ldots, n$.

\noindent $\textbf{ZDT3}^{*}$ \citep{chen2023data}: It is a classical multi-objective optimization benchmark problem in the form:
\begin{align*}\label{ZDT3_variant}\tag{$\text{ZDT3}^{*}$ }
    & f_1(\mathbf{x})=x_1 \\
    & f_2(\mathbf{x})=g(\mathbf{x})\left(1-\sqrt{\left(f_1(\mathbf{x}) / g(\mathbf{x})\right)} - \left(f_1(\mathbf{x})^{\gamma} / g(\mathbf{x})\right)\sin{A\pi f_1^{\beta}}\right),
\end{align*}
where $g(\mathbf{x})=1+\frac{9}{n-1} \sum_{i=1}^{n-1}x_{i+1} \text { and } 0 \leq x_i \leq 1 \text { for } i=1, \ldots, n$. The $A$ determines the number of disconnected regions of the PF. $\gamma$ controls the overall shape of the PF where $\gamma>1, \gamma<1$, and $\gamma=1$ lead to a concave, a convex, and a linear PF, respectively. $\beta$ influences the location of the disconnected regions.

\noindent \textbf{DTLZ7} \citep{deb2002scalable}: It is a classical multi-objective optimization benchmark problem with the form:
\begin{align*}\label{DTLZ7}\tag{DTLZ7}
    & f_1(\mathbf{x}_1)=x_1 \\
    & f_2(\mathbf{x}_2) = x_2 \\
    & \vdots \\
    & f_{m-1}(\mathbf{x}_{m-1}) = x_{m-1} \\
    & f_m(\mathbf{x}_m)=\frac{(1+g(\mathbf{x}_m))h\left(f_1,f_2,\dots,f_{m-1},g\right)}{6},
\end{align*}
where $g(\mathbf{x}_m)=1 + \frac{9}{|\mathbf{x}_m|}\sum_{x_i\in\mathbf{x}_m}x_i, h\left(f_1,f_2,\dots,f_{m-1},g\right) = m - \sum_{i=1}^{m-1}\big[\frac{f_i}{1+g}\left(1+\sin{3\pi f_i}\right)\big] \text { and } 0 \leq x_i \leq 1 \text { for } i=1, \ldots, n$.  The functional $g$ requires $k=|\mathbf{x}_m|=n-m+1$ decision variables.

The disparity between the Hypervolume calculated utilizing the actual Pareto front $\mathcal{F}$ and the learned Pareto front $\hat{\mathcal{F}}$ of the Joint Input model is illustrated in Table \ref{w_join_input} and Figure \ref{fig:appro_disconnect}. The outcomes of the Joint Input model surpass those of the Mixture of Experts structure. However, this distinction is not statistically significant. In addition, the hyper-transformer model with MoE still gets a much lower MED score than the joint input when comparing disconnected Pareto front tests. Using complex MoE designs for the Hyper-Transformer model shows that Controllable Disconnected Pareto Front Learning could have good future results.
 \begin{figure*}[b]
     \centering
        \begin{subfigure}[b]{0.15\textwidth}
         \centering
        \includegraphics[width=\textwidth]{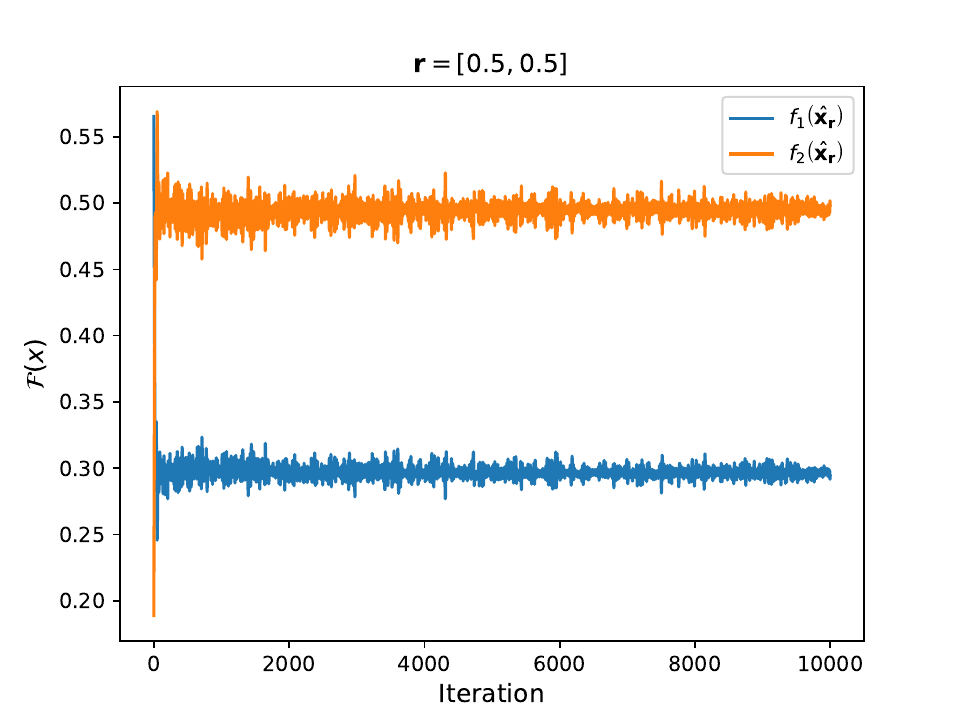}
     \end{subfigure}
     \hfill
     \begin{subfigure}[b]{0.15\textwidth}
         \centering
        \includegraphics[width=\textwidth]{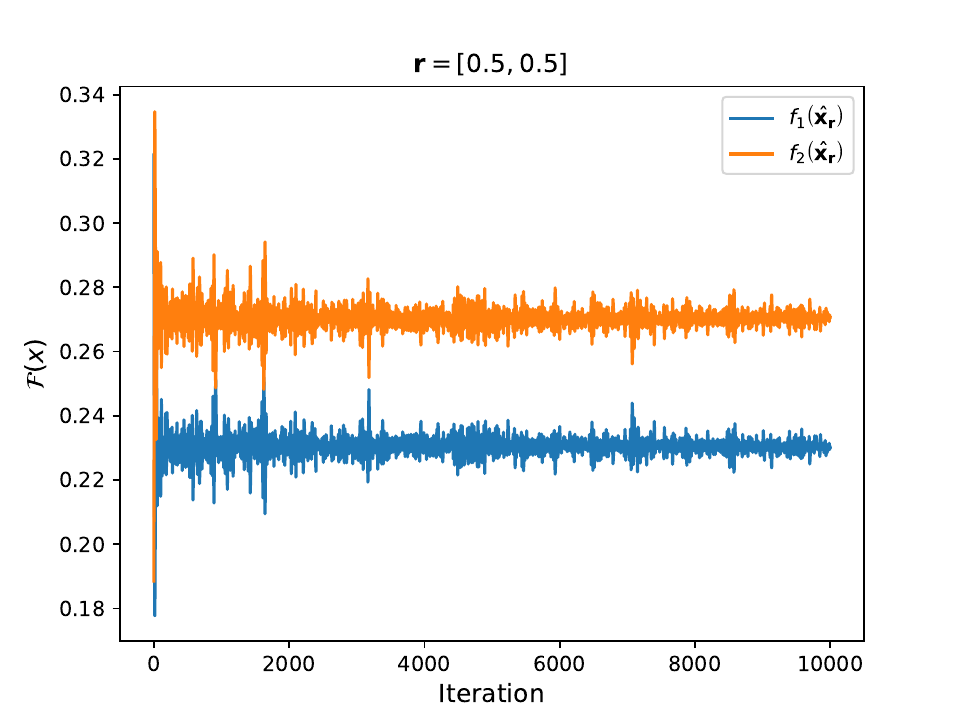}
     \end{subfigure}
     \hfill
     \begin{subfigure}[b]{0.15\textwidth}
         \centering
        \includegraphics[width=\textwidth]{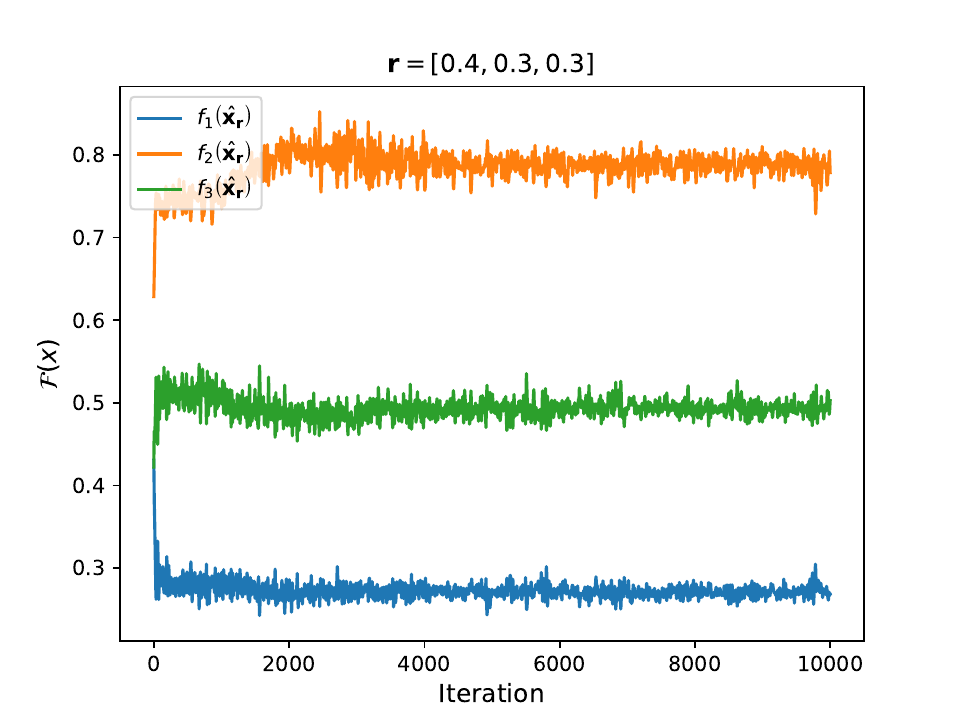}
     \end{subfigure}
     \hfill
     \begin{subfigure}[b]{0.15\textwidth}
         \centering
        \includegraphics[width=\textwidth]{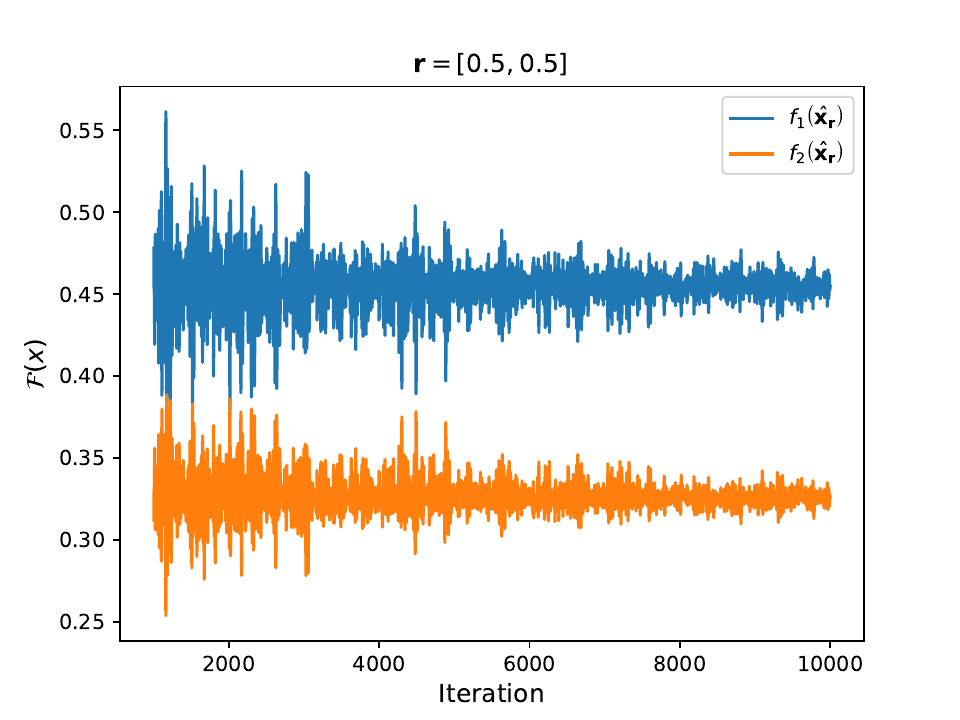}
     \end{subfigure}
     \hfill
     \begin{subfigure}[b]{0.15\textwidth}
         \centering
        \includegraphics[width=\textwidth]{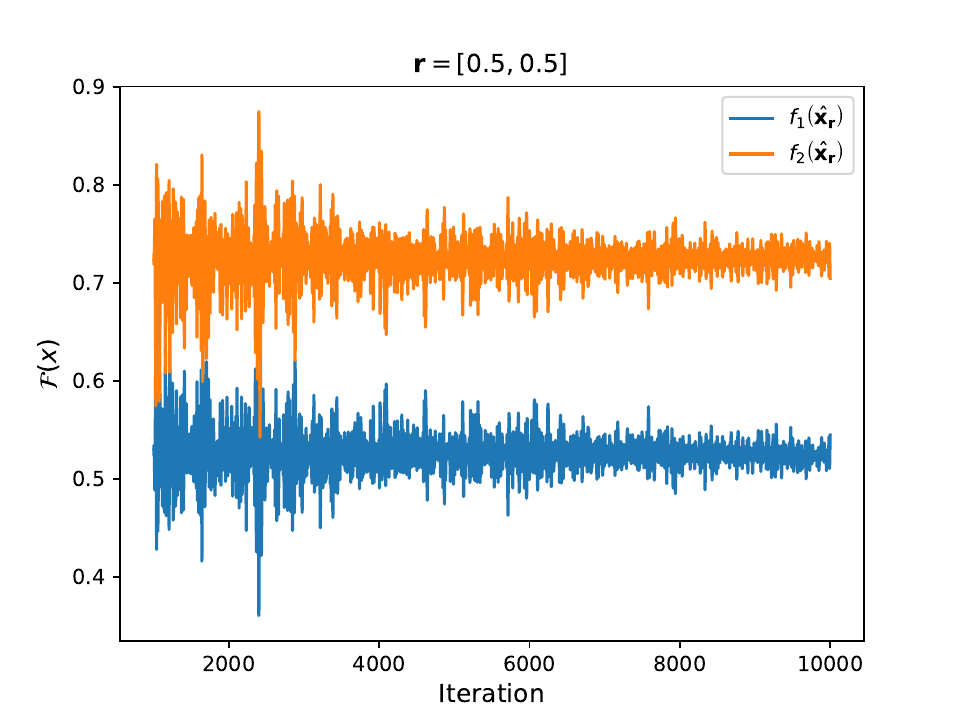}
     \end{subfigure}
     \hfill
     \begin{subfigure}[b]{0.15\textwidth}
         \centering
        \includegraphics[width=\textwidth]{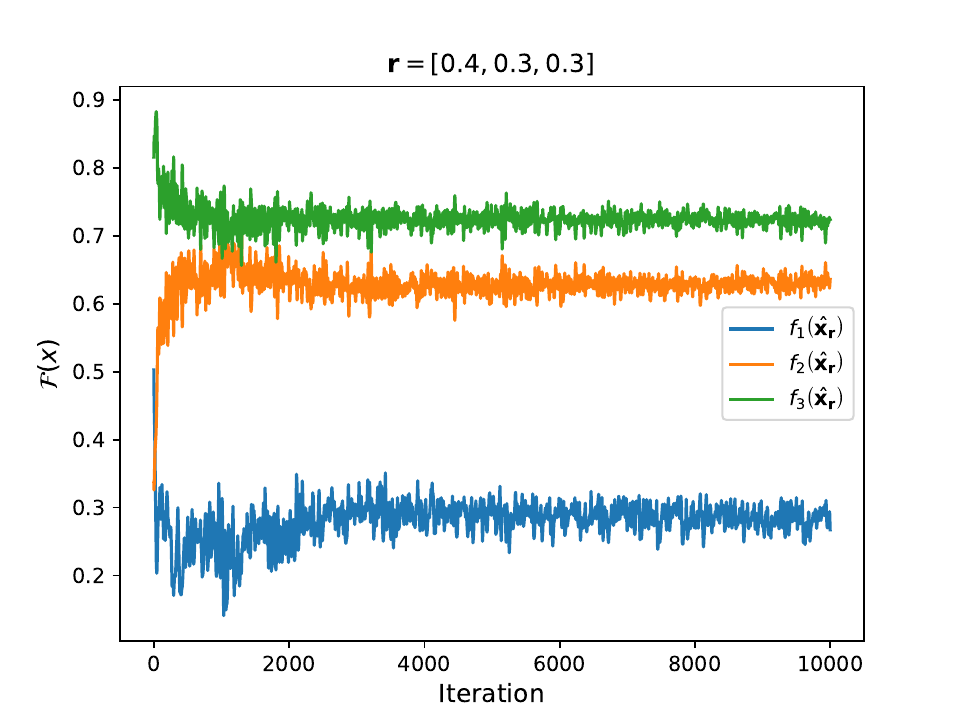}
     \end{subfigure}

    \begin{subfigure}[b]{0.15\textwidth}
         \centering
        \includegraphics[width=\textwidth]{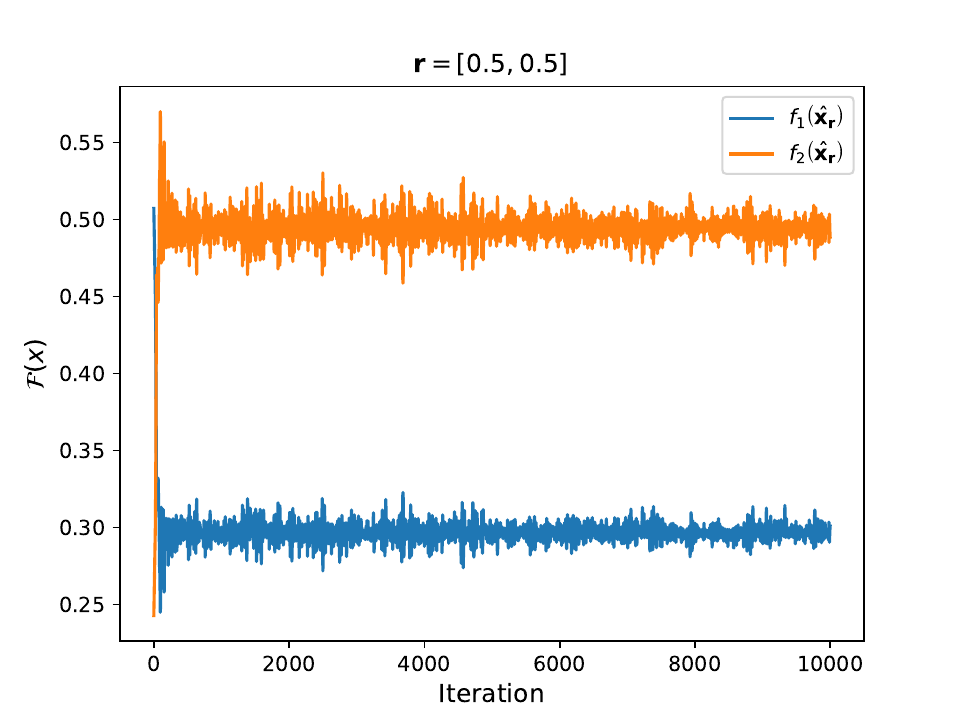}
     \end{subfigure}
     \hfill
     \begin{subfigure}[b]{0.15\textwidth}
         \centering
        \includegraphics[width=\textwidth]{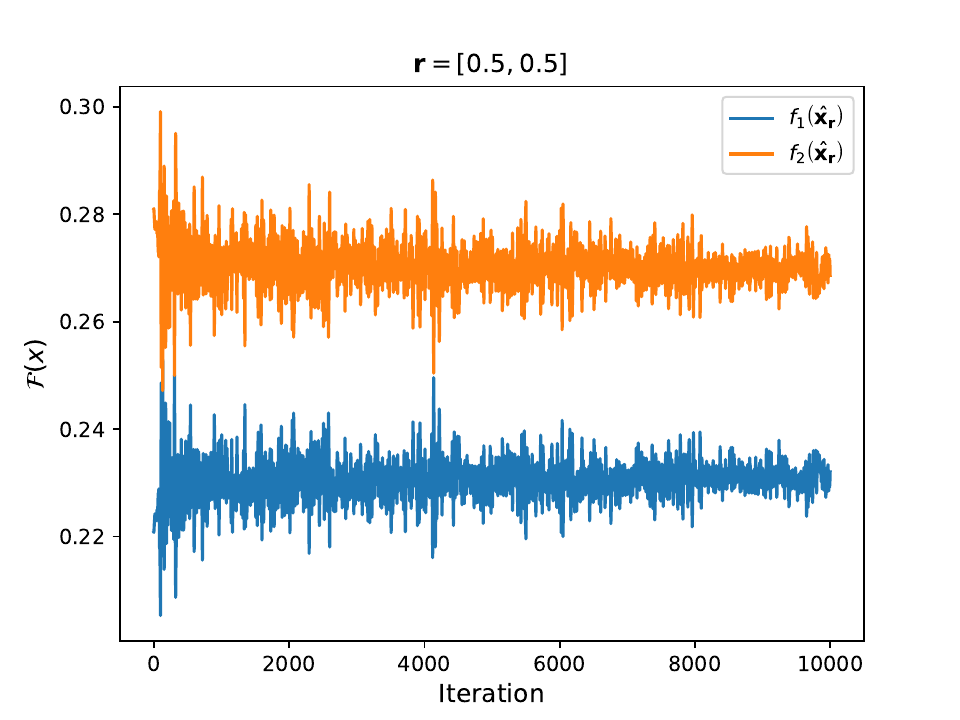}
     \end{subfigure}
     \hfill
     \begin{subfigure}[b]{0.15\textwidth}
         \centering
        \includegraphics[width=\textwidth]{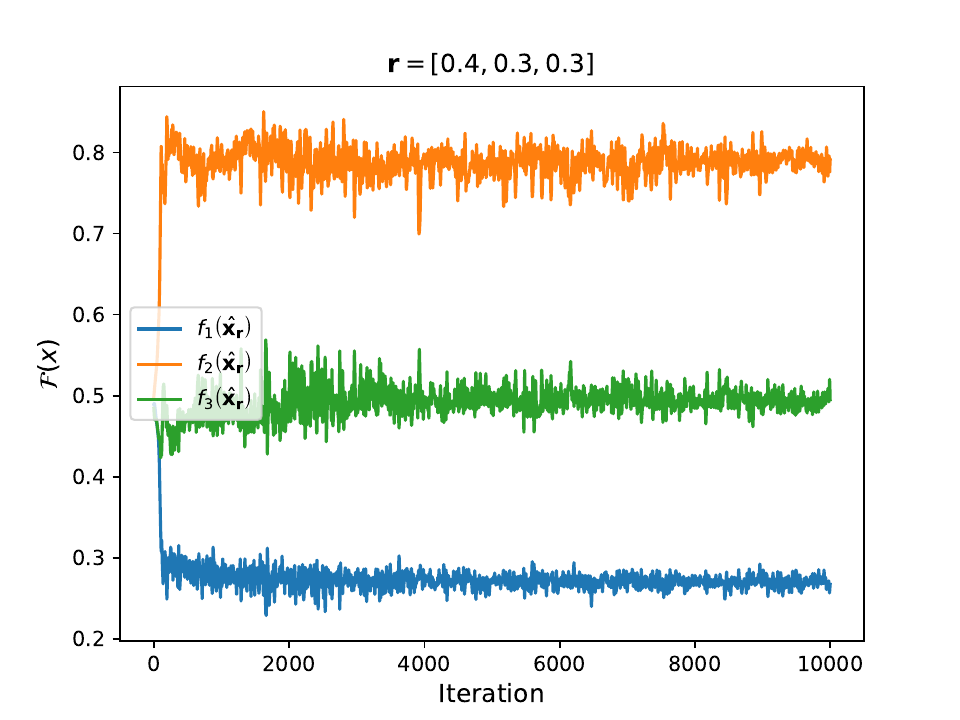}
     \end{subfigure}
     \hfill
     \begin{subfigure}[b]{0.15\textwidth}
         \centering
        \includegraphics[width=\textwidth]{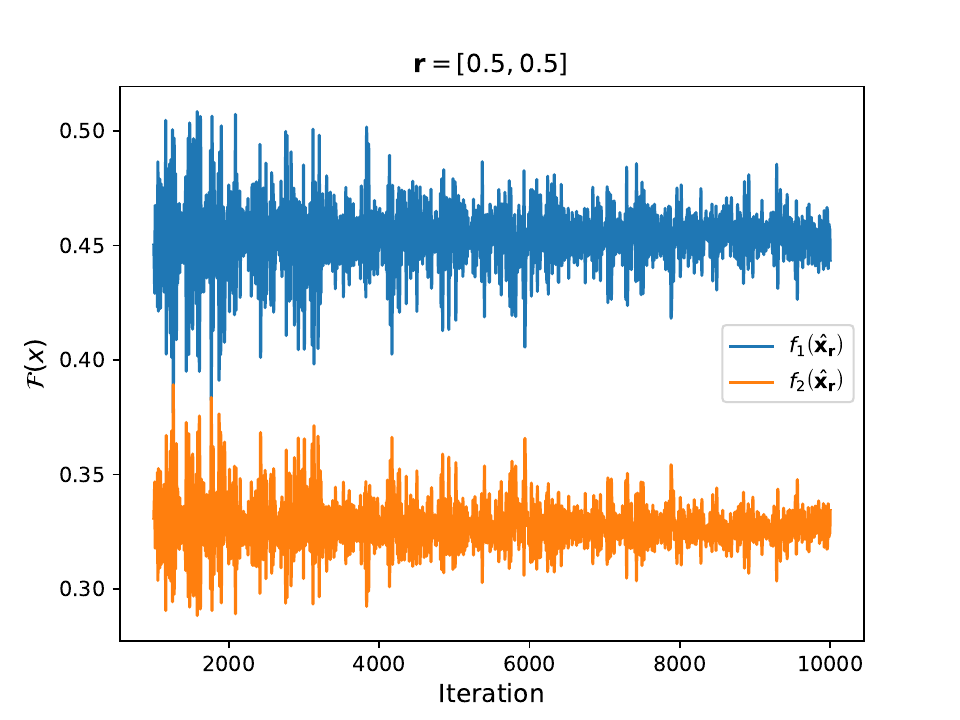}
     \end{subfigure}
     \hfill
     \begin{subfigure}[b]{0.15\textwidth}
         \centering
        \includegraphics[width=\textwidth]{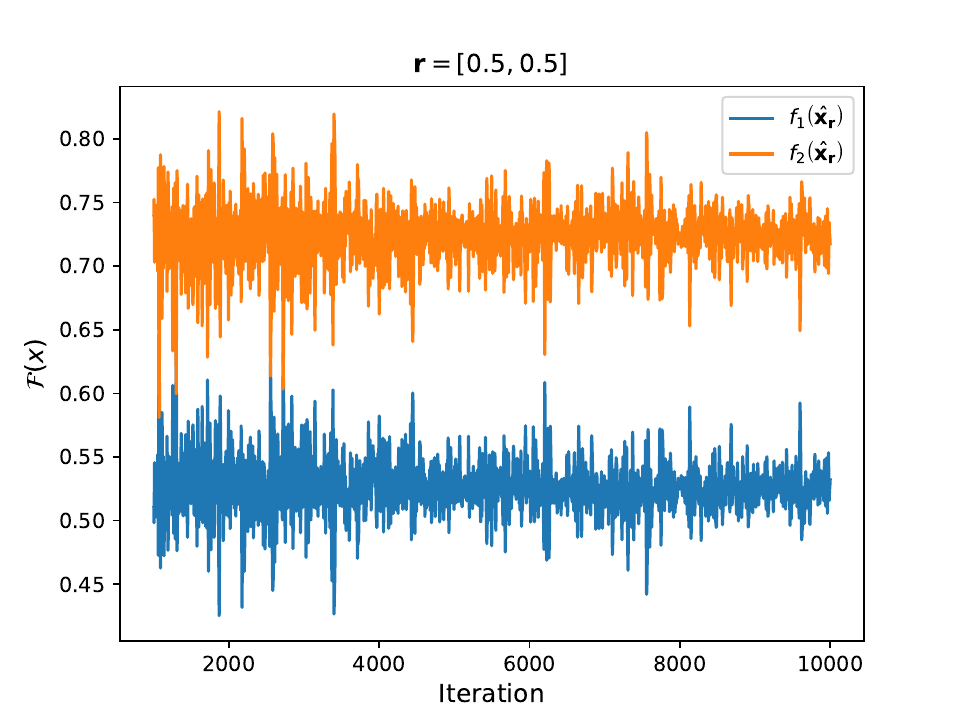}
     \end{subfigure}
     \hfill
     \begin{subfigure}[b]{0.15\textwidth}
         \centering
        \includegraphics[width=\textwidth]{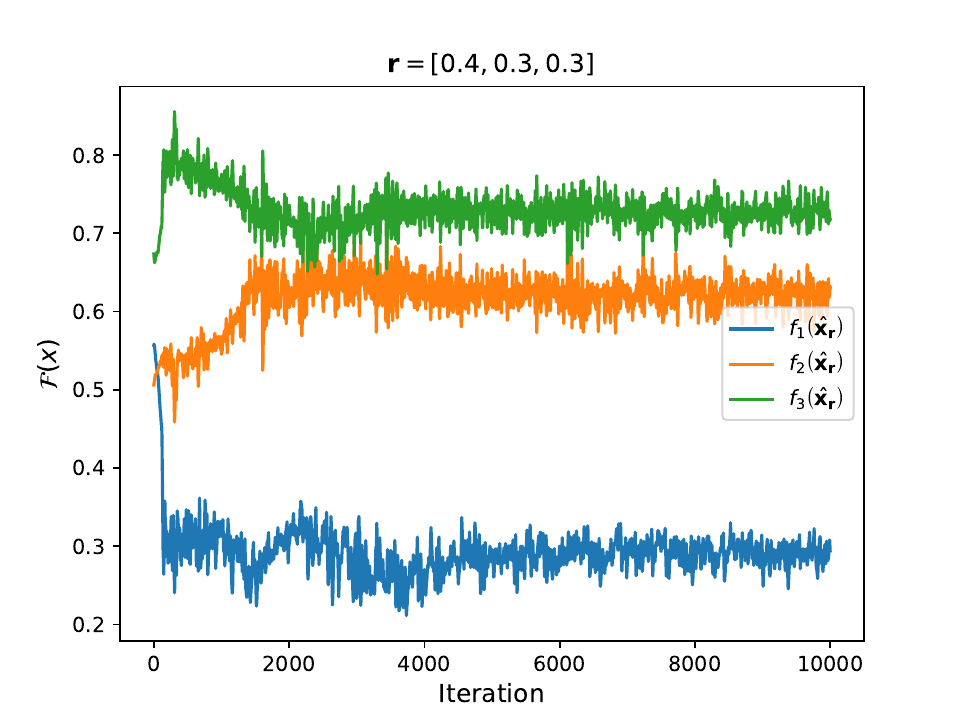}
     \end{subfigure}
     \caption{Comparison of multi-objective trajectories between Hyper-Trans and Hyper-MLP. The top panel shows the evolution of $\mathcal{F}(\mathbf{x})$ obtained by Hyper-Trans with $\mathbf{r}= [0.5,0.5]$ in 2 objectives and $\mathbf{r}= [0.4,0.3,0.3]$ in 3 objectives at \ref{CVX1}, \ref{CVX2}, \ref{CVX3}, \ref{ZDT1}, \ref{ZDT2}, \ref{DTLZ2} problems (from left to right). The bottom panel shows the evolution of $\mathcal{F}(\mathbf{x})$ obtained by Hyper-MLP.}
     \label{fig:connect_trajectory}
\end{figure*}
 \begin{table}[!b]
\caption{We evaluate $30$ random seed with lower bounds in Table \ref{attributes}.}
\label{wo_join_input}
\centering
\resizebox{1.0\textwidth}{!}{\begin{tabular}{|c|c|c|c|c|c|}
\toprule
\bf Example & \bf Constraint layer &\bf Hyper-MLP\citep{tuan2023framework} & \bf Hyper-Trans (\textbf{ours}) & \bf Params & \bf MED$\Downarrow$\\ 
 \midrule
\multirow{2}{*}{\ref{CVX1}}& \multirow{2}{*}{sigmoid} &\checkmark & & $5\times 5701$ & $0.00229 \pm  0.00119$\\
& & & \checkmark & $5\times 5731$ &   $\bf 0.00161\pm \bf 0.00129$\\
 \midrule
\multirow{2}{*}{\ref{CVX2}}& \multirow{2}{*}{sigmoid} & \checkmark & & $5\times 5732$ & $0.00353\pm  0.00144$\\
& & & \checkmark & $5 \times 5762$ &   $\bf 0.00258\pm \bf 0.00127$\\
 \midrule
\multirow{2}{*}{\ref{CVX3}}& \multirow{2}{*}{softmax + sqrt} & \checkmark & & $5\times 5793$ &  $0.01886 \pm 0.00784$\\
& & & \checkmark  & $5\times 5853$ &$\bf 0.00827 \pm 0.00187$\\
 \midrule
\multirow{2}{*}{\ref{ZDT1}}& \multirow{2}{*}{sigmoid} & \checkmark & & $5\times 6600$ & $0.00682\pm 0.00385$\\
& & & \checkmark & $5 \times 6630$ &  $\bf 0.00219\pm \bf 0.00049$\\
 \midrule
\multirow{2}{*}{\ref{ZDT2}}& \multirow{2}{*}{sigmoid} & \checkmark & & $5 \times 6600$ & $0.00859\pm 0.00476$\\
& & & \checkmark & $5 \times 6630$ &  $\bf 0.00692\pm \bf 0.00304$\\
 \midrule
\multirow{2}{*}{\ref{ZDT3}}& \multirow{2}{*}{sigmoid} & \checkmark & & $5 \times 3210$ & $0.18741\pm 0.00653$\\
& & & \checkmark & $5 \times 3230$ &  $\bf 0.00767\pm \bf 0.00414$\\
 \midrule
\multirow{2}{*}{\ref{ZDT3_variant}}& \multirow{2}{*}{sigmoid} & \checkmark & & $2\times6600$ & $0.00641\pm 0.00594$\\
& & & \checkmark & $2\times6630$ &  $\bf 0.00391\pm \bf 0.00404$\\
 \midrule
\multirow{2}{*}{\ref{DTLZ2}}& \multirow{2}{*}{sigmoid} & \checkmark & & $5 \times 2810$ &$0.06217\pm  0.01528$\\
& & & \checkmark & $5 \times 2850$ &  $\bf 0.01083\pm \bf 0.00142$\\
\midrule
\multirow{2}{*}{\ref{DTLZ7}}& \multirow{2}{*}{sigmoid} & \checkmark & & $4\times 6010$ & $0.03439\pm  0.02409$\\
& & & \checkmark & $4\times 6070$ &  $\bf 0.01116 \pm \bf  0.00217$\\
\bottomrule
\end{tabular}}
\end{table}
 \begin{table*}[ht]
\caption{We evaluate $30$ random seed with lower bounds in Table \ref{attributes}.}
\label{w_join_input}
\centering
\resizebox{1.0\textwidth}{!}{\begin{tabular}{|c|c|c|c|c|c|c|c|}
\toprule
\bf Example & \bf Constraint layer &\bf Model &\bf Joint Input& \bf  Mixture of Experts&\bf Params& \bf HVD$\Downarrow$& \bf MED$\Downarrow$\\ 
 \midrule
 \multirow{2}{*}{\ref{ZDT3}} & \multirow{2}{*}{sigmoid}& \multirow{2}{*}{Hyper-Trans} & \checkmark &  & 22230 &$0.04088$ & $0.52587 \pm 0.37795$\\
& & & & \checkmark & 20250 &$\bf0.00091$&$\bf 0.24787 \pm \bf 0.22053$\\
 \midrule
\multirow{2}{*}{\ref{ZDT3_variant}} & \multirow{2}{*}{sigmoid}& \multirow{2}{*}{Hyper-Trans} & \checkmark &  & 4110 &$\bf-0.00472$&$0.35923 \pm 0.35548$\\
& & & & \checkmark & 3900 &  $-0.00466$& $\bf 0.12789 \pm \bf 0.09911$\\
 \midrule
\multirow{2}{*}{\ref{DTLZ7}} & \multirow{2}{*}{sigmoid}& \multirow{2}{*}{Hyper-Trans} & \checkmark &  & 7990 & $\bf 0.00265$& $ 0.52847 \pm 0.31352$\\
& & & & \checkmark & 7880 & $ 0.00302$ & $\bf 0.12338 \pm 0.07585$\\
\bottomrule 
\end{tabular}}

\end{table*}
 \begin{figure*}[b]
     \centering
        \begin{subfigure}[b]{0.45\textwidth}
         \centering
         \includegraphics[width=\textwidth]{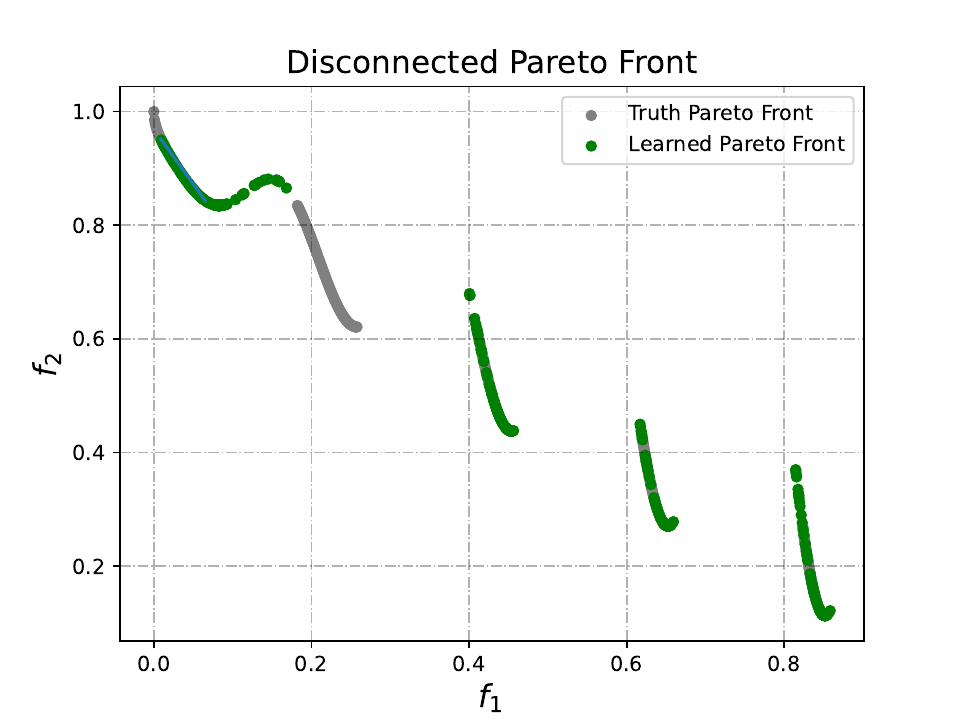}
     \end{subfigure}
     \hfill
     \begin{subfigure}[b]{0.45\textwidth}
         \centering
         \includegraphics[width=\textwidth]{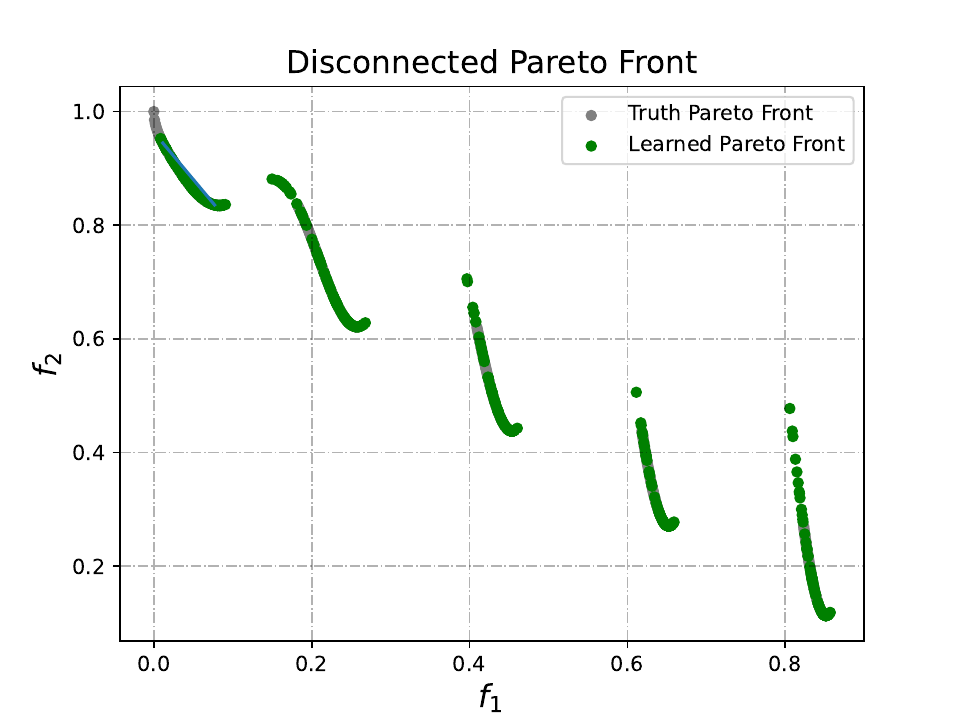}
     \end{subfigure}

     \centering
     \begin{subfigure}[b]{0.45\textwidth}
         \centering
         \includegraphics[width=\textwidth]{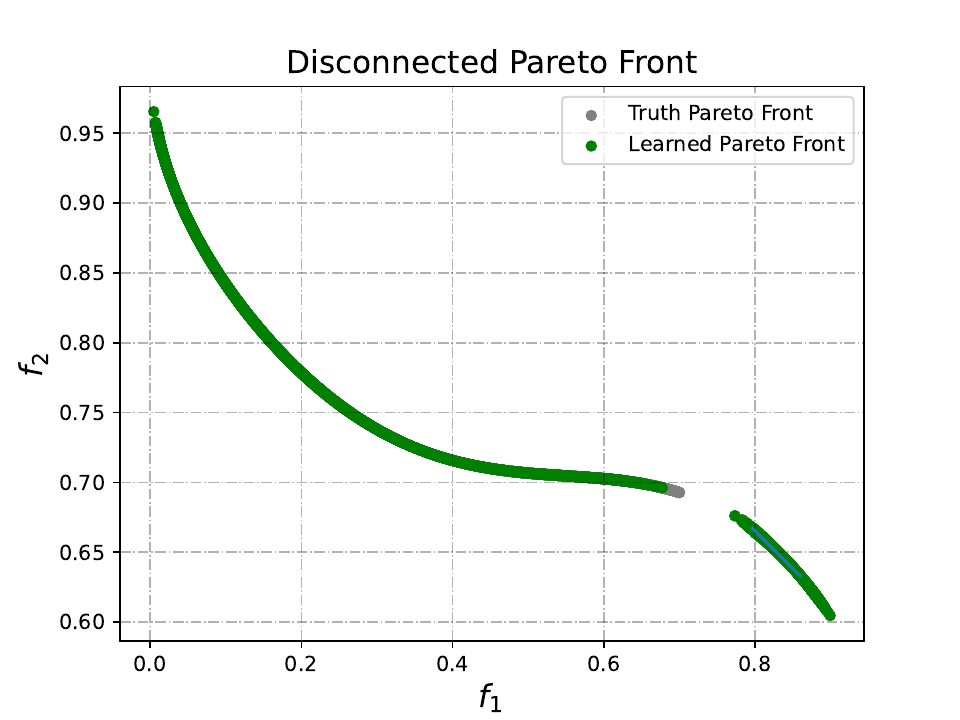}
     \end{subfigure}
     \hfill
     \begin{subfigure}[b]{0.45\textwidth}
         \centering
         \includegraphics[width=\textwidth]{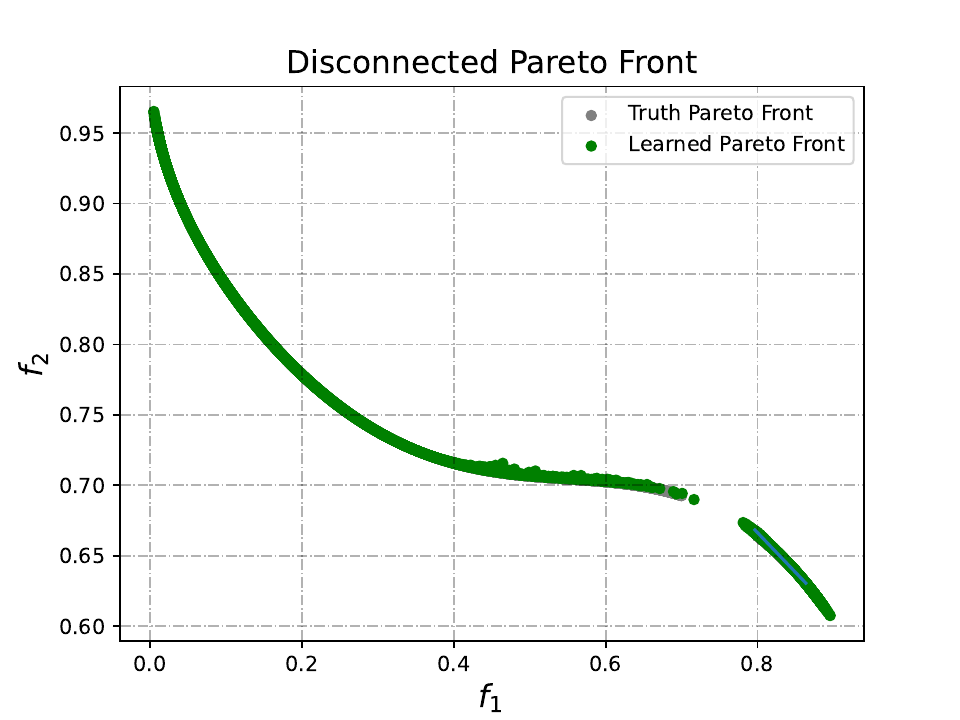}
     \end{subfigure}

    \begin{subfigure}[b]{0.45\textwidth}
         \centering
         \includegraphics[width=\textwidth]{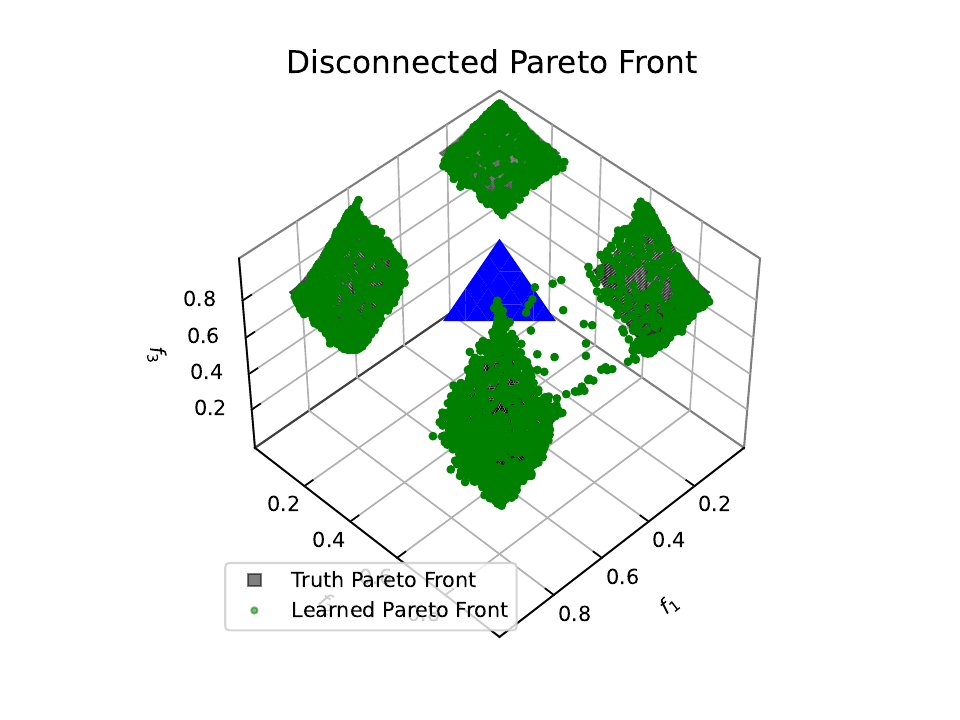}
         \caption{Hyper-Transformer with joint input}
     \end{subfigure}
     \hfill
     \begin{subfigure}[b]{0.45\textwidth}
         \centering
         \includegraphics[width=\textwidth]{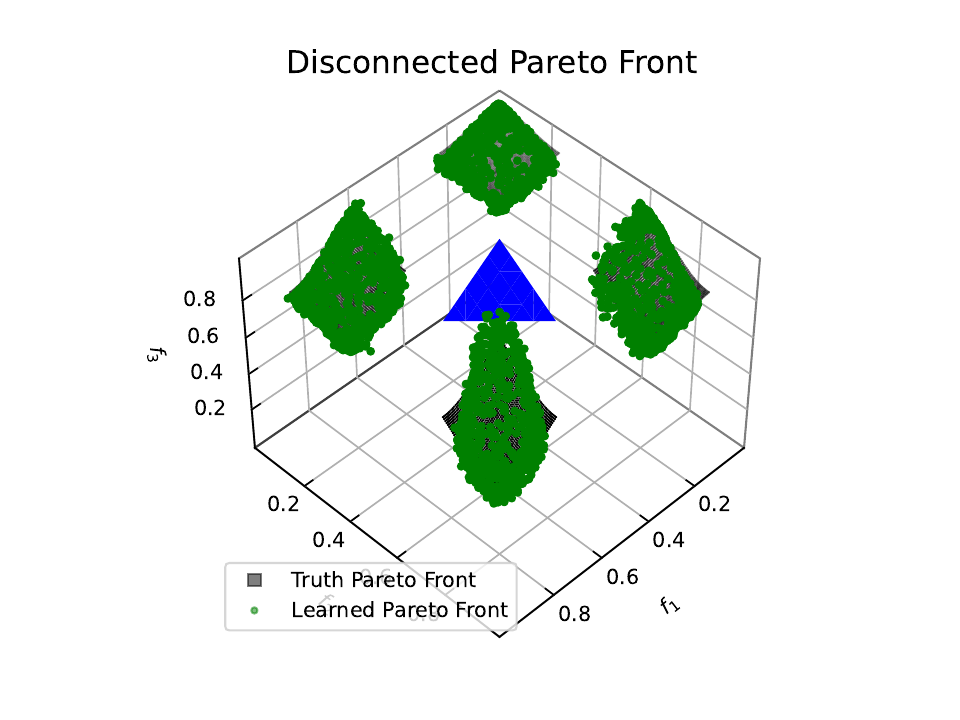}
         \caption{Hyper-Transformer with Mixture of Experts}
     \end{subfigure}
      \caption{\textit{Left}: Pareto Front is approximated by the Joint Input model. \textit{Right}: Pareto Front is approximated by the Mixture of Experts model in example \ref{ZDT3} (top), example \ref{ZDT3_variant} (middle), and example \ref{DTLZ7} (bottom).}
      \label{fig:appro_disconnect}
\end{figure*}

\section{Conclusion and Future Work}

This paper presents a novel approach to tackle controllable Pareto front learning with split feasibility constraints. Additionally, we provide mathematical explanations for accurately mapping a priority vector to the corresponding Pareto optimal solution by hyper-transformers based on a universal approximation theory of the sequence-to-sequence function. Furthermore, this study represents the inaugural implementation of Controllable Disconnected Pareto Front Learning. Besides, we provide experimental computations of controllable Pareto front learning with a MED score to substantiate our theoretical reasoning. The outcomes demonstrate that the hypernetwork model, based on a transformer architecture, exhibits superior performance in the connected Pareto front and disconnected Pareto front problems compared to the multi-layer Perceptron model.

Although the early results are promising, several obstacles must be addressed. Multi-task learning studies show promise for real-world multi-objective systems that need real-time control and involve difficulties with split feasibility constraints. Nevertheless, more enhancements are required for our suggested approach to addressing disconnected Pareto Front issues. This is due to the need for the model to possess prior knowledge of the partition feasibility limitations, which restricts the model's capacity to anticipate non-dominated solutions. Future research might involve the development of a resilient MoE hyper-transformer that can effectively adjust to various split feasibility limitations and prevent the occurrence of non-dominated solutions and weakly Pareto-optimal solutions.
\section*{Declaration of competing interest}
\section*{Data availability}
Data will be made available on request.
\section*{Acknowledgments}
\appendix
\section{Experiment Details}\label{appendixA}
\subsection{Computational Analysis}
Hyper-Transformer consists of two blocks: the Self-Attention mechanism and the Multilayer Perceptron. With Hypernetwork w/o join input architect, we assume dimension of three matrics $\boldsymbol{W}_Q, \boldsymbol{W}_K, \boldsymbol{W}_V$ is $d$, number of heads is $2$. Besides, we also assume the input and output of the MLP block with $d$ dimension. Hence, the total parameters of the Hyper-Transformer is $4d^2 + 4d$. The Hyper-MLP architect uses six hidden linear layers with $d$ dimension input and output. Therefore, the total parameters of Hyper-MLP is $6d^2 + 6d$. 

Although the total parameters of Hyper-MLP are larger than Hyper-Transformer, the number of parameters that need to be learned for the MOP examples is the opposite when incorporating the Embedding block. In the two architectures described in Figures\ref{fig:wo_join_input}a and \ref{fig:wo_join_input}b, the parameters to be learned of Hyper-Trans are:
\begin{align*}
    6d^2 + 6d + 2md + (d+1)n,
\end{align*}
and with Hyper-MLP are:
\begin{align*}
    6d^2 + 6d + (m+1)d + (d+1)n.
\end{align*}
From there, we see that the difference in the total parameters to be learned $(m-1)d$ between the Hyper-Trans and Hyper-MLP models is insignificant. It only depends on the width of the hidden layers $d$ and the number of objective functions $m\ge 2$.
\subsection{Training setup}
The experiments MOO were implemented on a computer with CPU - Intel(R) Core(TM) i7-10700, 64-bit CPU $@ 2.90$GHz, and 16 cores. Information on MOO test problems is illustrated in Table \ref{attributes}. 
\begin{table}[b]
\caption{\revise{Information of MOO problems.}}
\label{attributes}
\centering
\resizebox{0.6\textwidth}{!}{\begin{tabular}{c|c|c|c|c|c}
\toprule
\bf Problem & \bf n & \bf m & \textbf{Objective function} & \textbf{Pareto-optimal} & \textbf{Pareto front}\\
 \midrule
\ref{CVX1} & 1 & 2 & convex & convex & connected \\
\ref{CVX2} & 2 & 2 & convex & convex & connected\\
\ref{CVX3} & 3 & 3 & convex & convex & connected\\
\ref{ZDT1} & 30 & 2 & non-convex & convex & connected\\
\ref{ZDT2} & 30 & 2 & non-convex & non-convex & connected\\
\ref{ZDT3} & 30 & 2 & non-convex & non-convex & disconnected\\
\ref{ZDT3_variant} & 30 & 2 & non-convex & non-convex& disconnected\\
\ref{DTLZ2} & 10 & 3 & non-convex & non-convex& connected\\
\ref{DTLZ7} & 10 & 3 & non-convex & non-convex& disconnected\\
\bottomrule
\end{tabular}}
\end{table}

\begin{table}[b]
\caption{\revise{Hyperparameters for training MOO problems.}}
\label{hyperparameters}
\centering
\resizebox{1.\textwidth}{!}{\begin{tabular}{c|c}
\toprule
\bf Problem & \textbf{Hyperparameters}\\
 \midrule
\ref{CVX1} & Adam optimizer, $
\alpha = 0.6, d = 20, iter = 20000, lr = 0.001, a = [[0,0.8],[0.1,0.6],[0.2,0.4],[0.35,0.22],[0.6,0.1]]$  \\
 \midrule
\ref{CVX2} & Adam optimizer, $
\alpha = 0.6, d = 20, iter = 20000, lr = 0.001, a = [[0,0.6],[0.02,0.4],[0.16,0.2],[0.2,0.15],[0.4,0.02]]$ \\
\ref{CVX3} & Adam optimizer, $
\alpha = 0.6, d = 20, iter = 20000, lr = 0.001, a = [[0.15, 0.2, 0.7], [0.2, 0.5, 0.6], [0.2, 0.7, 0.4], [0.35, 0.6, 0.22], [0.6, 0.1, 0.46]]$ \\
\ref{ZDT1} & Adam optimizer, $
\alpha = 0.6, d = 20, iter = 20000, lr = 0.001, a = [[0, 0.8], [0.1, 0.6], [0.2, 0.4], [0.35, 0.22], [0.6, 0.1]]$ \\
\ref{ZDT2} & Adam optimizer, $
\alpha = 0.6, d = 20, iter = 20000, lr = 0.001, a = [[0.1, 0.9], [0.1, 0.6], [0.2, 0.4], [0.35, 0.22], [0.6, 0.1]]$ \\
\ref{ZDT3} & Adam optimizer, $
\alpha = 0.6, d = 30, iter = 20000, lr = 0.001, a = [[0.01, 0.81], [0.16, 0.61], [0.4, 0.41], [0.62, 0.23], [0.81, 0.1]]$ \\
\ref{ZDT3_variant} & Adam optimizer, $
\alpha = 0.6, d = 10, iter = 20000, lr = 0.001, a = [[0.8, 0.62], [0.01, 0.7]], A=2, \gamma = 3,\beta = \dfrac{1}{3}$ \\
\ref{DTLZ2} & Adam optimizer, $
\alpha = 0.6, d = 20, iter = 20000, lr = 0.001, a = [[0.15, 0.2, 0.7], [0.2, 0.5, 0.6], [0.2, 0.7, 0.4], [0.35, 0.6, 0.22], [0.6, 0.1, 0.46]]$ \\
\ref{DTLZ7} & Adam optimizer, $
\alpha = 0.6, d = 20, iter = 20000, lr = 0.001, a = [[0.62, 0.62, 0.4], [0.01, 0.62, 0.5], [0.01, 0.01, 0.82], [0.62, 0.01, 0.6]]$ \\
\bottomrule
\end{tabular}}
\end{table}
We use Hypernetwork based on multi-layer perceptron (MLP), which has the following structure:
\begin{align*}
    h_{\text{mlp}}(\mathbf{r},\phi): \textbf{Input}(\mathbf{r}) &\to \textbf{Linear}(m, d) \to\textbf{ReLU}\to \textbf{Linear}(d,d) \to \textbf{ReLU}\\
    &\to \textbf{Linear}(d,d)
    \to \textbf{ReLU} \to \textbf{Linear}(d,d)
    \to \textbf{ReLU} \\
    &\to \textbf{Linear}(d,d)
    \to \textbf{ReLU} \to \textbf{Linear}(d,d)
    \to \textbf{ReLU}\\
    &\to \textbf{Linear}(d,d)
    \to \textbf{ReLU} \to \textbf{Linear}(d, n)\\
    & \to \textbf{Constraint layer}\to \textbf{Output}(\mathbf{x}_{\mathbf{r}}).
\end{align*}
Toward Hypernetwork based on the Transformer model, we use the structure as follows:
\begin{align*}
    h_{\text{trans}}(\mathbf{r},\phi): \textbf{Input}(\mathbf{r}) &\to \left[
\begin{matrix}
\textbf{Linear}(1, d)  & \\
\dots \\
\textbf{Linear}(1, d) & \\
\end{matrix}
\right.\to \textbf{Concatenate}\to\left[
\begin{matrix}
&\textbf{Multi-Head Self-Attention}   \\
&\textbf{Identity layer}  \\
\end{matrix}
\right.\\
& \to \textbf{Sum} \to \left[
\begin{matrix}
&\textbf{MLP}   \\
&\textbf{Identity layer}  \\
\end{matrix}
\right. \to \textbf{Sum} \to \textbf{Linear}(d, n)\\
& \to \textbf{Constraint layer}\to \textbf{Output}(\mathbf{x}_{\mathbf{r}}).
\end{align*}
\section{ADDITIONAL EXPERIMENTS}\label{appendixB}
\subsection{Application of Controllable Pareto Front Learning in Multi-task Learning}
\subsubsection{Multi-task Learning as Multi-objectives optimization.}
Denotes a supervised dataset $\left(\mathbf{x},\mathbf{y}\right)=\left\{\left(x_j,y_j\right)\right\}_{j=1}^N$  where $N$ is the number of data points. They specified the MOO formulation of Multi-task learning from the empirical loss $\mathcal{L}^i(\mathbf{y},g(\mathbf{x},\boldsymbol{\theta}))$ using a vector-valued loss $\mathcal{L}$:
\begin{align*}
    \boldsymbol{\theta} &= \argmin_{\boldsymbol{\theta}} \mathcal{L}\left(\mathbf{y},g\left(\mathbf{x},\boldsymbol{\theta}\right)\right), \\
    \mathcal{L}\left(\mathbf{y}, g\left(\mathbf{x}, \boldsymbol{\theta}\right)\right) &= \left(\mathcal{L}_1\left(\mathbf{y},g\left(\mathbf{x},\boldsymbol{\theta}\right)\right),\dots,\mathcal{L}_m\left(\mathbf{y},g\left(\mathbf{x},\boldsymbol{\theta}\right)\right)\right)^T
\end{align*}
where $g\left(\mathbf{x}; \boldsymbol{\theta} \right): \mathcal{X}\times\Theta \rightarrow \mathcal{Y}$ represents to a Target network with parameters $\boldsymbol{\theta}$.
\subsubsection{Controllable Pareto Front Learning in Multi-task Learning.}
Controllable Pareto Front Learning in Multi-task Learning by solving the following:
\begin{align*}
    & \phi^* = \argmin_{\phi} \underset{\substack{\mathbf{r} \sim Dir(\alpha)\\(\mathbf{x}, \mathbf{y}) \sim p_D}}{\mathbb{E}} \mathcal{S}(\mathcal{L}(\mathbf{y},g(\mathbf{x}, \boldsymbol{\theta}_{\mathbf{r}}), \mathbf{r},\mathbf{a}) \\
    & \text{ s.t. } \boldsymbol{\theta}_{\mathbf{r}} = h(\mathbf{r}, \phi^*)\\
    & \mathcal{L}(\mathbf{y},g(\mathbf{x}, \boldsymbol{\theta}_{\mathbf{r}}) \le \mathbf{b},
\end{align*}
where $h: \mathcal{P} \times \Phi \rightarrow \Theta$ represents to a Hypernetwork, $\mathbf{a} = (\mathbf{a}_1,\dots,\mathbf{a}_m), \mathbf{a}_i \ge 0$ is the lower-bound vector for the loss vector $\mathcal{L}(\mathbf{y},g(\mathbf{x}, \boldsymbol{\theta}_{\mathbf{r}})$, and the upper-bound vector denoted as $\mathbf{b} = (\mathbf{b}_1,\dots,\mathbf{b}_m), \mathbf{b}_i \ge 0$ is the desired loss value. The random variable $\mathbf{r}$ is a preference vector, forming the trade-off between loss functions.

\subsubsection{Multi-task Learning experiments}
The dataset is split into two subsets in MTL experiments: training and testing. Then, we split the training set into ten folds and randomly picked one fold to validate the model. The model with the highest HV in the validation fold will be evaluated. All methods are evaluated with the same well-spread preference vectors based on \citep{das}. The experiments MTL were implemented on a computer with CPU - Intel(R) Xeon(R) Gold 5120 CPU @ 2.20GHz, 32 cores, and GPU - VGA NVIDIA Tesla V100-PCIE with VRAM 32 GB.

\textbf{Image Classification.} Our experiment involved the application of three benchmark datasets from Multi-task Learning for the image classification task: Multi-MNIST \citep{sabour2017}, Multi-Fashion \citep{xiao2017fashion}, and Multi-Fashion+MNIST \citep{lin2019pareto}. We compare our proposed Hyper-Trans model with the Hyper-MLP model based on Multi-LeNet architecture \citep{tuan2023framework}, and we report results in Table \ref{mnist}.

\begin{table*}[b]
\caption{Testing hypervolume on Multi-MNIST, Multi-Fashion, and Multi-Fashion+MNIST datasets with 10 folds split.}
\label{mnist}
\centering
\resizebox{0.8\textwidth}{!}{\begin{tabular}{| c | >{\centering\arraybackslash}m{2.cm} | >{\centering\arraybackslash}m{2.cm} | >{\centering\arraybackslash}m{2.cm} |>{\centering\arraybackslash}m{1.cm} | }
\toprule
                & {\bf Multi-MNIST} & {\bf Multi-Fashion} & {\bf Fashion-MNIST} &   \\ \midrule  

{\bf Method} & \bf HV $\Uparrow$ & \bf HV$\Uparrow$ & \bf HV$\Uparrow$ &  {\bf Params} \\
\midrule
Hyper-MLP \citep{tuan2023framework} &$2.860 \pm 0.027$  & $2.164 \pm 0.045$  & $2.781 \pm 0.039$ &  8.66M \\ 
\midrule
Hyper-Trans + ReLU (\textbf{ours}) & $\bf 2.883 \pm \bf 0.029$   & $2.166 \pm 0.059$  & $\bf 2.806 \pm \bf 0.041$  & 8.66M  \\ 
Hyper-Trans + GeLU (\textbf{ours})& $2.879 \pm 0.017$   & $\bf 2.196 \pm \bf 0.046$  & $2.802 \pm 0.049$  & 8.66M  \\ 
\bottomrule
\end{tabular}}
\end{table*}

\textbf{Scene Understanding.} The NYUv2 dataset \citep{silberman2012indoor} serves as the basis experiment for our method. This dataset is a collection of 1449 RGBD images of an indoor scene that have been densely labeled at the per-pixel level using 13 classifications. We use this dataset as a 2-task MTL benchmark for depth estimation and semantic segmentation. The results are presented in Table \ref{nuyv2} with (3, 3) as hypervolume's reference point. Our method, which includes ReLU and GeLU activations, achieves the best HV on the NYUv2 dataset with the same parameters as Hyper-MLP.
\begin{table*}[ht]
\caption{Testing hypervolume on NYUv2 dataset.}
\label{nuyv2}
\centering
\resizebox{0.5\textwidth}{!}{\begin{tabular}{| c | >{\centering\arraybackslash}m{1.cm} |>{\centering\arraybackslash}m{1.cm} | }
\toprule
\multicolumn{3}{|c|}{NYUv2} \\ 
\midrule
{\bf Method} & \bf HV $\Uparrow$& {\bf Params} \\
\midrule
Hyper-MLP \citep{tuan2023framework} & $4.058$ & 31.09M \\ 
\midrule
Hyper-Trans + ReLU (\textbf{ours}) &  $ 4.093$ & 31.09M  \\ 
Hyper-Trans + GeLU (\textbf{ours}) &   $\bf 4.135$ & 31.09M  \\ 
\bottomrule
\end{tabular}}
\end{table*}

\textbf{Multi-Output Regression.} We conduct experiments using the SARCOS dataset \citep{Sethu2000}  to illustrate the feasibility of our methods in high-dimensional space. The objective is to predict seven relevant joint torques from a 21-dimensional input space (7 tasks) (7 joint locations, seven joint velocities, and seven joint accelerations). In Table \ref{sarcos}, our proposed model shows superiority over Hyper-MLP in terms of hypervolume value.

\begin{table*}[ht]
\caption{Testing hypervolume on SARCOS dataset with ten folds split. }
\label{sarcos}
\centering
\resizebox{0.5\textwidth}{!}{\begin{tabular}{| c | >{\centering\arraybackslash}m{2.cm} |>{\centering\arraybackslash}m{1.cm} | }
\toprule
\multicolumn{3}{|c|}{SARCOS} \\ 
\midrule
{\bf Method} & \bf HV $\Uparrow$& {\bf Params} \\
\midrule
Hyper-MLP \citep{tuan2023framework} & $0.6811 \pm 0.227$ & 7.1M \\ 
\midrule
Hyper-Trans + ReLU (\textbf{ours}) & $ 0.7107 \pm 0.0236$ & 7.1M  \\ 
Hyper-Trans + GeLU (\textbf{ours}) & $\bf 0.7123 \pm \bf 0.0134$ & 7.1M  \\ 
\bottomrule
\end{tabular}}
\end{table*}

\textbf{Multi-Label Classification.} Continually investigate our proposed architecture in MTL problem, we solve the problem of recognizing 40 facial attributes (40 tasks) in 200K face images on CelebA dataset \citep{liu2015deep} using a big Target network: Resnet18 (11M parameters) of \citep{he2016deep}. Due to the very high dimensional scale (40 dimensions), we only test hypervolume value on ten hard-tasks CelebA datasets, including 'Arched Eyebrows,' 'Attractive,' 'Bags Under Eyes,' 'Big Lips,' 'Big Nose,' 'Brown Hair,' 'Oval Face,' 'Pointy Nose,' 'Straight Hair,' 'Wavy Hair.' Table \ref{celeba} shows that the Hyper-Trans model combined with the GeLU activation function gives the highest HV value with the reference point (1,\dots,1).

\begin{table*}[ht]
\caption{Testing hypervolume on ten hard-tasks CelebA dataset.}
\label{celeba}
\centering
\resizebox{0.5\textwidth}{!}{\begin{tabular}{| c | >{\centering\arraybackslash}m{2.cm} |>{\centering\arraybackslash}m{1.cm} | }
\toprule
\multicolumn{3}{|c|}{CelebA} \\ 
\midrule
{\bf Method} & \bf HV $\Uparrow$& {\bf Params} \\
\midrule
Hyper-MLP \citep{tuan2023framework}& $ 0.003995$ & 11.09M \\ 
\midrule
Hyper-Trans + ReLU (\textbf{ours}) & $0.003106$ & 11.09M  \\ 
Hyper-Trans + GeLU (\textbf{ours}) & $\bf 0.004719$ & 11.09M  \\ 
\bottomrule
\end{tabular}}
\end{table*}

\subsection{Number of Heads and Hidden dim}
To understand the impact of the number of heads and the dimension of hidden layers, we analyzed the MED error based on different numbers of heads and hidden dims in Figure \ref{fig:head_dim}.

We compare the Hyper-Transformer and Hyper-MLP models based on the MED score, where the dimension of the hidden layers $d = [16,32,64,128]$, and the number of heads $e = [1,2,4,8,16]$.
\begin{figure*}[ht]
     \centering
        \begin{subfigure}[b]{0.3\textwidth}
         \centering
         \includegraphics[width=\textwidth]{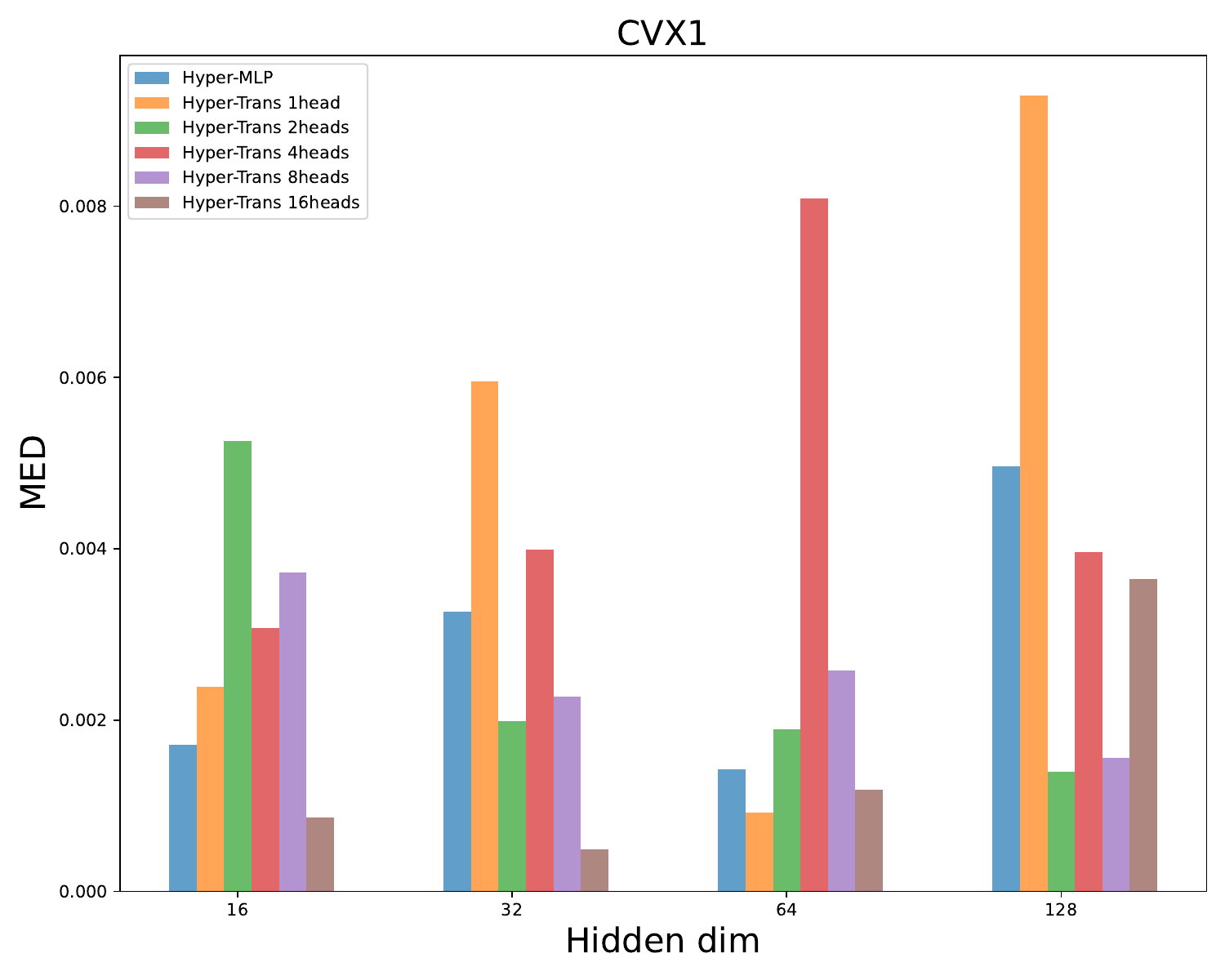}
     \end{subfigure}
     \hfill
     \begin{subfigure}[b]{0.3\textwidth}
         \centering
         \includegraphics[width=\textwidth]{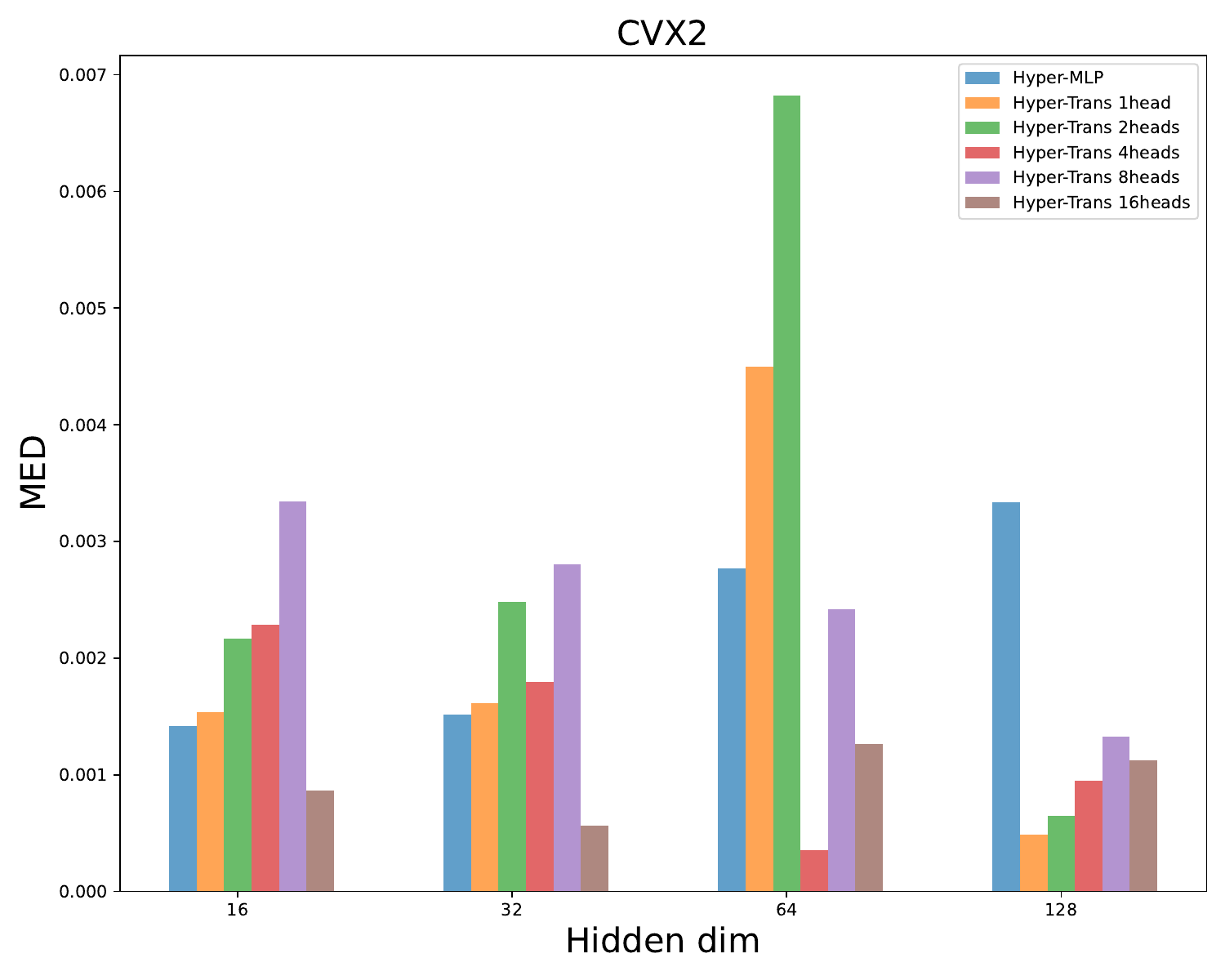}
     \end{subfigure}
     \hfill
     \begin{subfigure}[b]{0.3\textwidth}
         \centering
         \includegraphics[width=\textwidth]{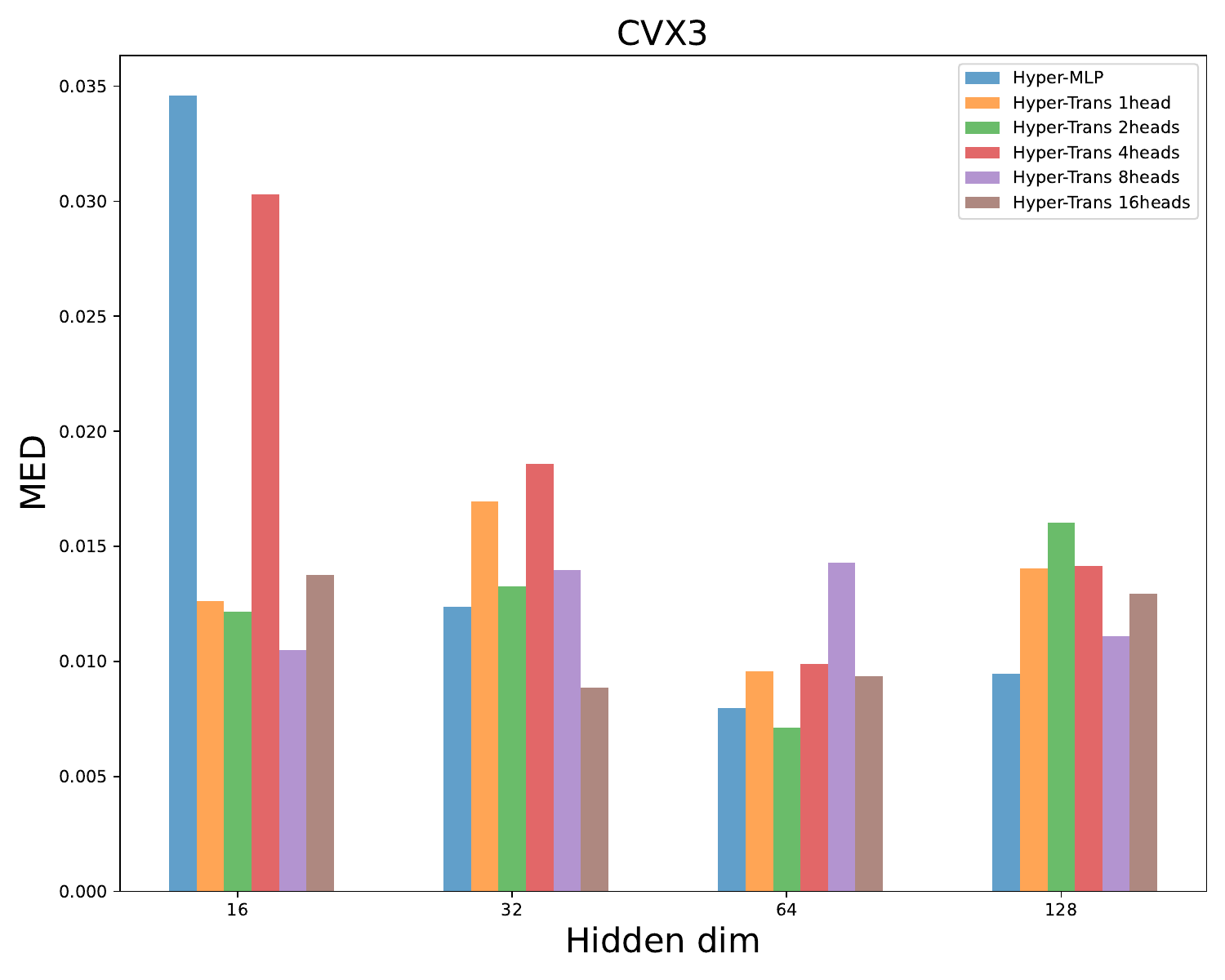}
     \end{subfigure}
     
     \begin{subfigure}[b]{0.3\textwidth}
         \centering
         \includegraphics[width=\textwidth]{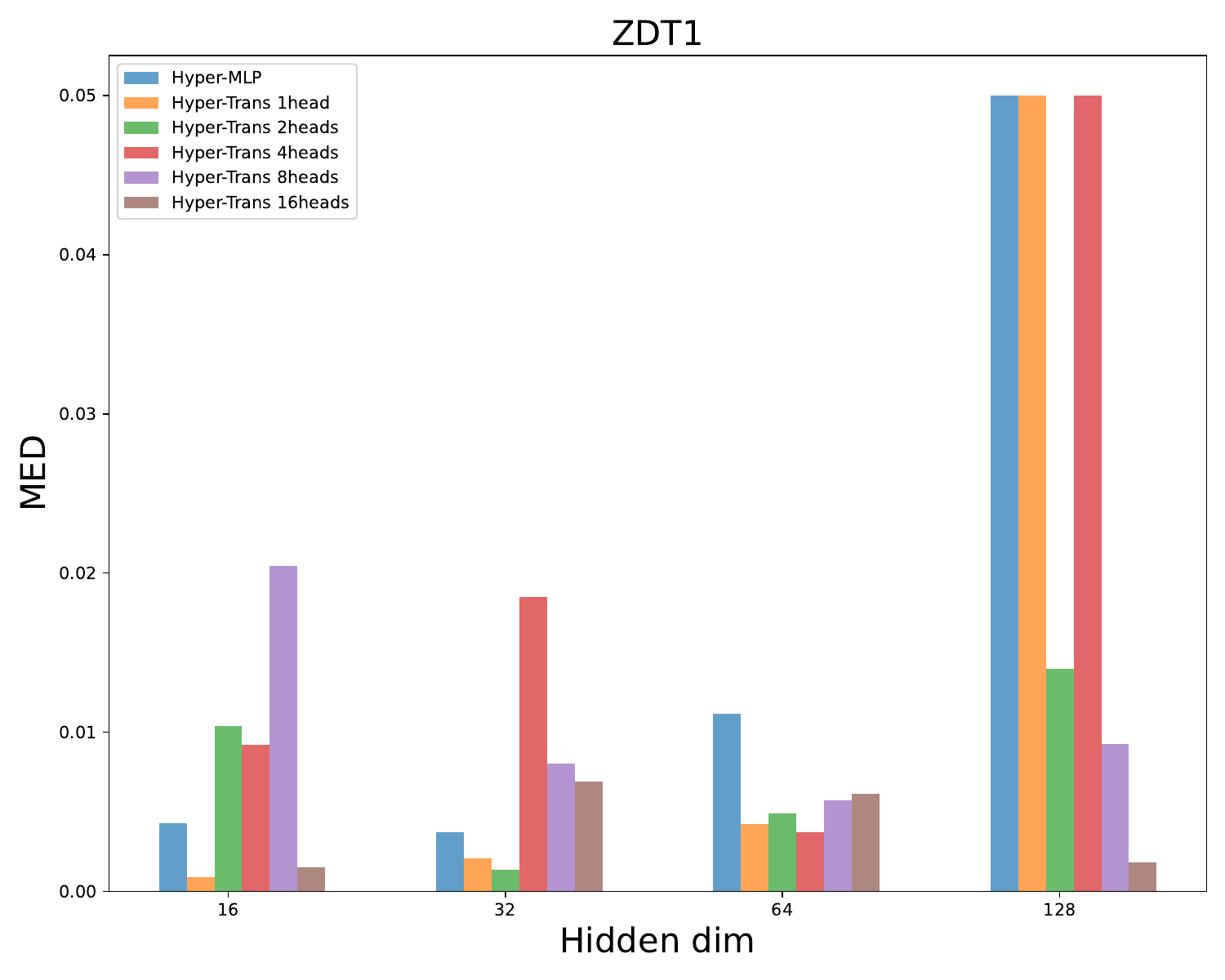}
     \end{subfigure}
     \hfill
     \begin{subfigure}[b]{0.3\textwidth}
         \centering
         \includegraphics[width=\textwidth]{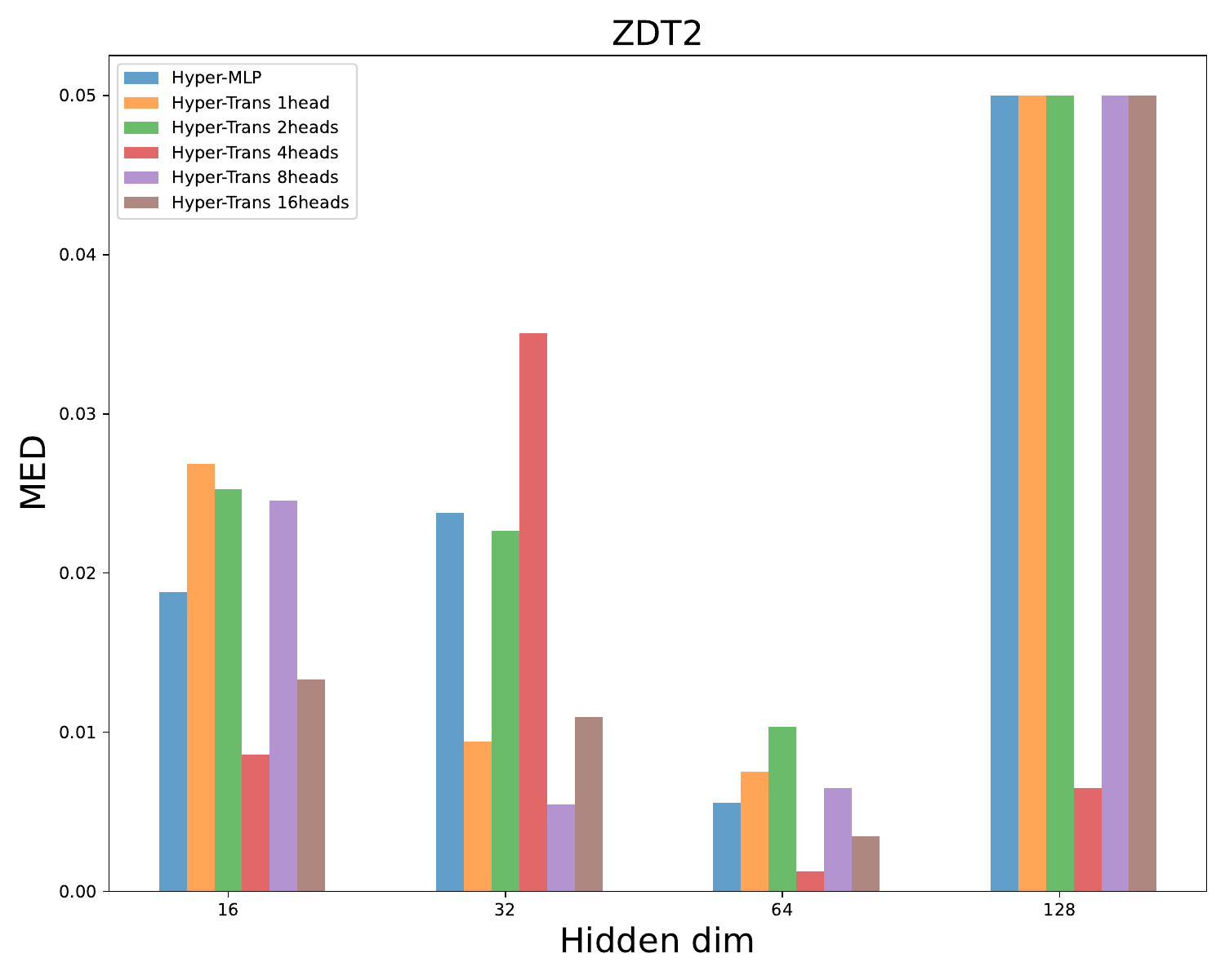}
     \end{subfigure}
     \hfill
     \begin{subfigure}[b]{0.3\textwidth}
         \centering
         \includegraphics[width=\textwidth]{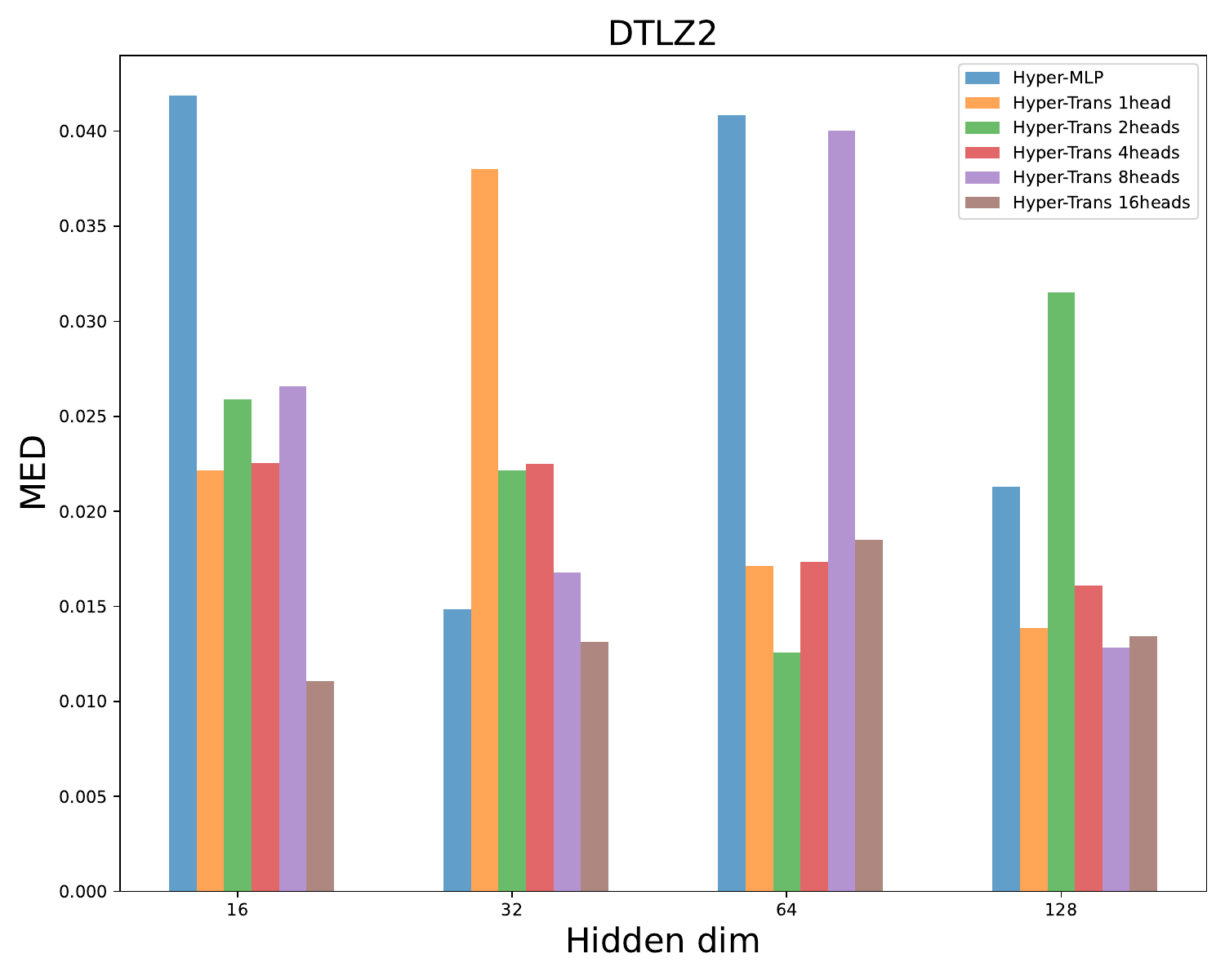}
     \end{subfigure}
      \caption{MED score of Hyper-Transformer and Hyper-MLP across the number of Heads and the dimension of Hidden layers.}
      \label{fig:head_dim}
\end{figure*}
\subsection{Feature Maps Weight generated by Hypernetworks}
We compute feature maps of the first convolutional from the weights generated by Hyper-Trans in Figure \ref{fig:feature_map}. Briefly, we averaged feature maps of this convolutional layer across all its filter outputs. This accounts for the learned weights of Multi-Lenet through hypernetworks.
\begin{figure*}[ht]
     \centering
     \includegraphics[width=0.7\textwidth]{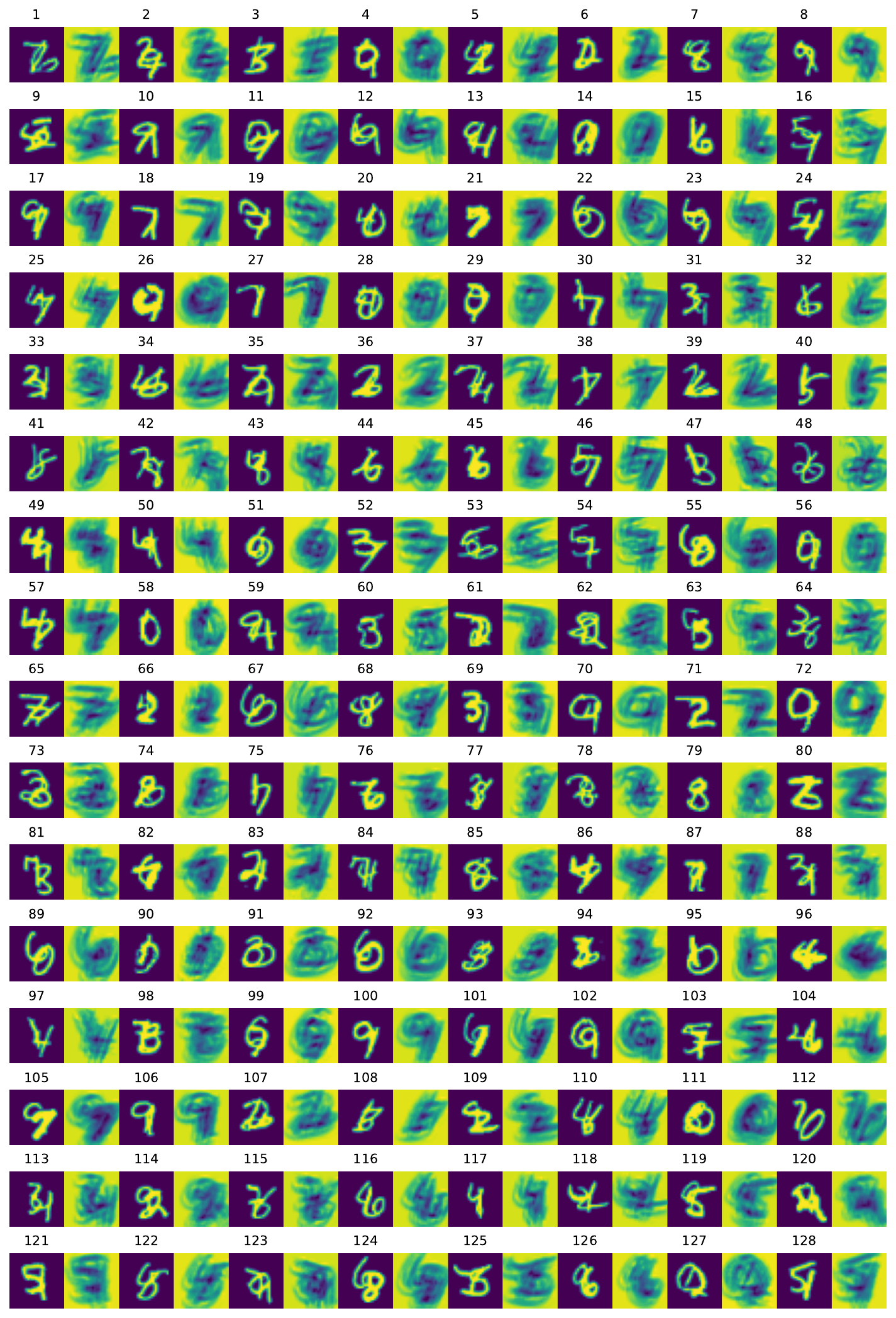}
     \caption{Feature maps of Targetnetwork with weights that Hyper-Transformer generated.}
     \label{fig:feature_map}
\end{figure*}
\subsection{Exactly Mapping of Hypernetworks}
We utilize Hypernetwork to generate an approximate efficient solution from a reference vector created by Dirichlet distribution with $\alpha = 0.6$. We trained all completion functions using an Adam optimizer \citep{kingma2014adam} with a learning rate of $1e-3$ and $20000$ iterations. In the test phase, we sampled three preference vectors based on each lower bound in Table \ref{hyperparameters}. Besides, we also illustrated target points and predicted points from the pre-trained Hypernetwork in Figure \ref{fig:connect_1}, \ref{fig:connect_2}, and \ref{fig:disconnect}.

\begin{figure*}[ht]
     \centering
        \begin{subfigure}[b]{0.45\textwidth}
         \centering
         \includegraphics[width=\textwidth]{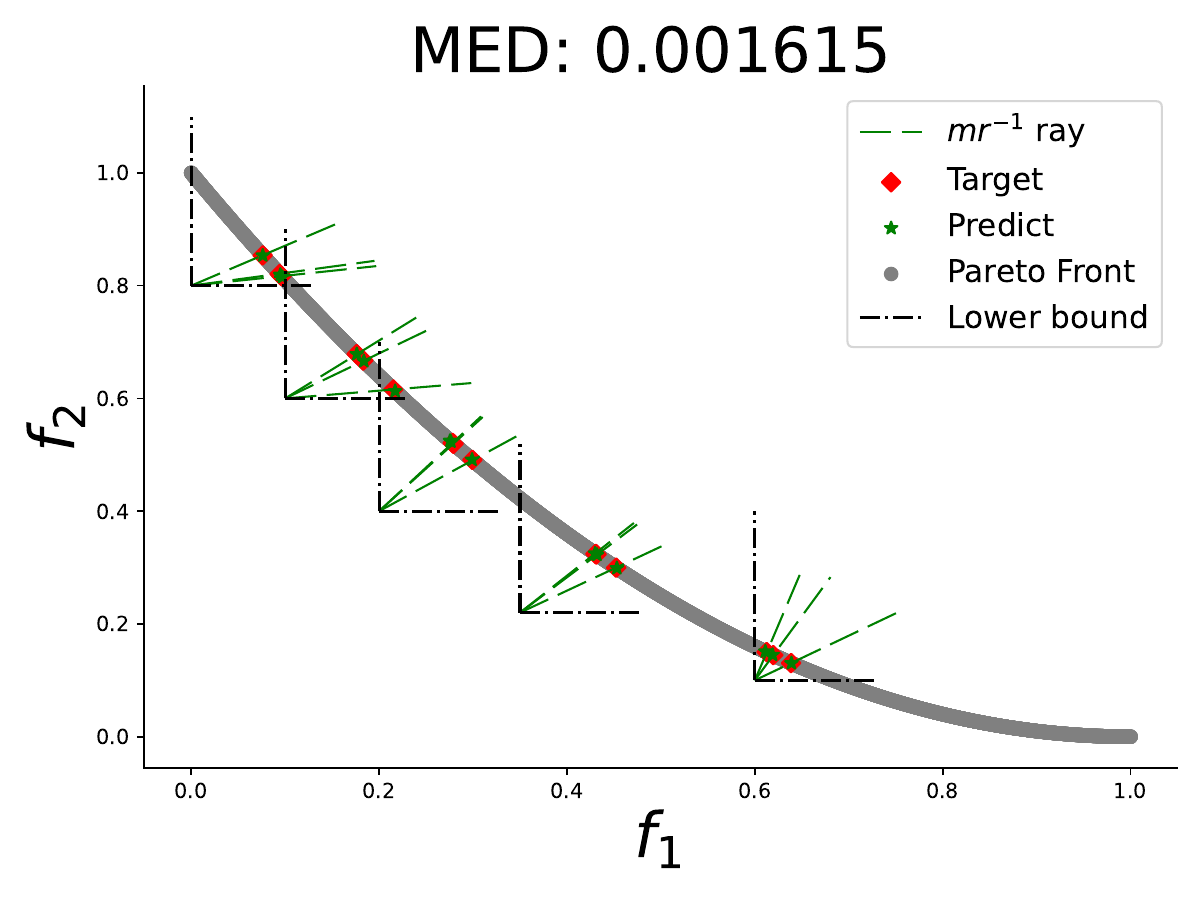}
     \end{subfigure}
     \hfill
     \begin{subfigure}[b]{0.45\textwidth}
         \centering
         \includegraphics[width=\textwidth]{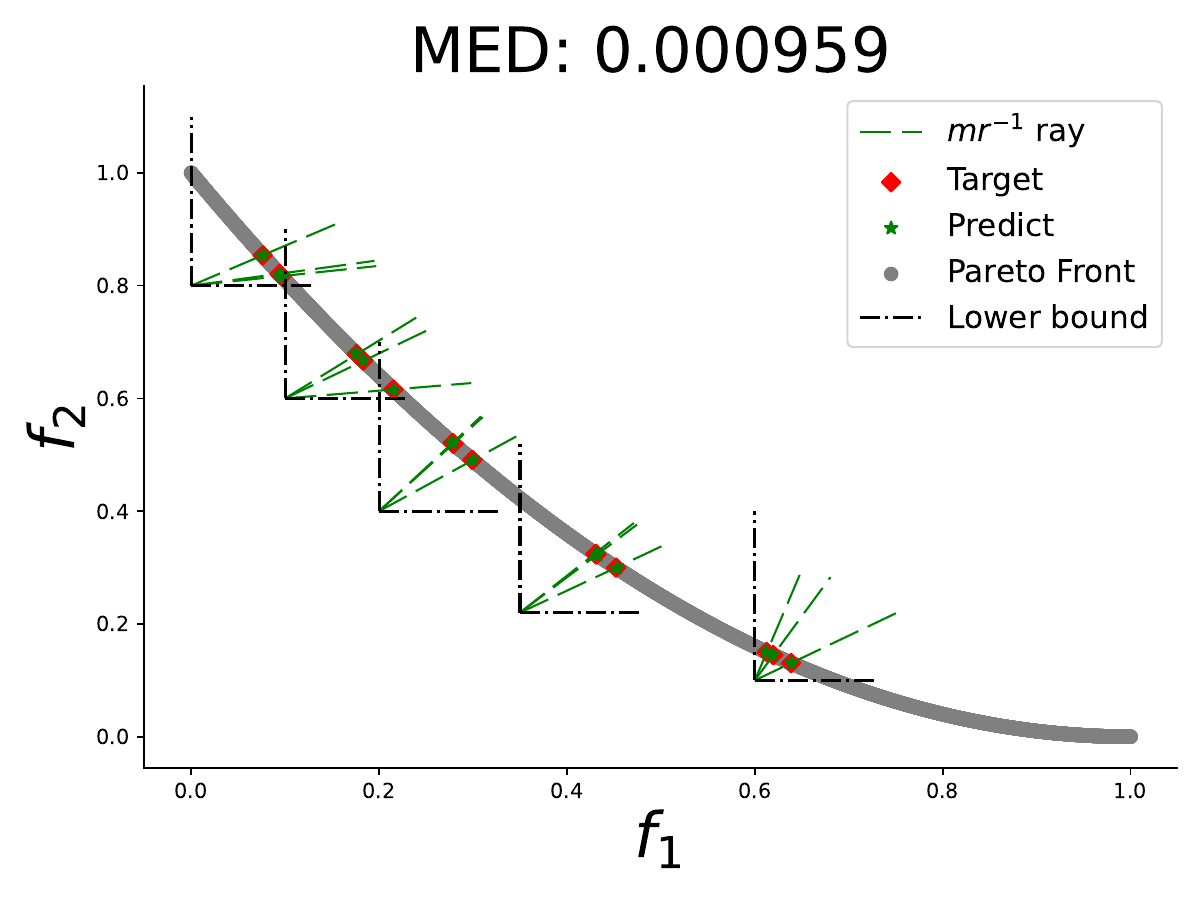}
     \end{subfigure}

     \centering
     \begin{subfigure}[b]{0.45\textwidth}
         \centering
         \includegraphics[width=\textwidth]{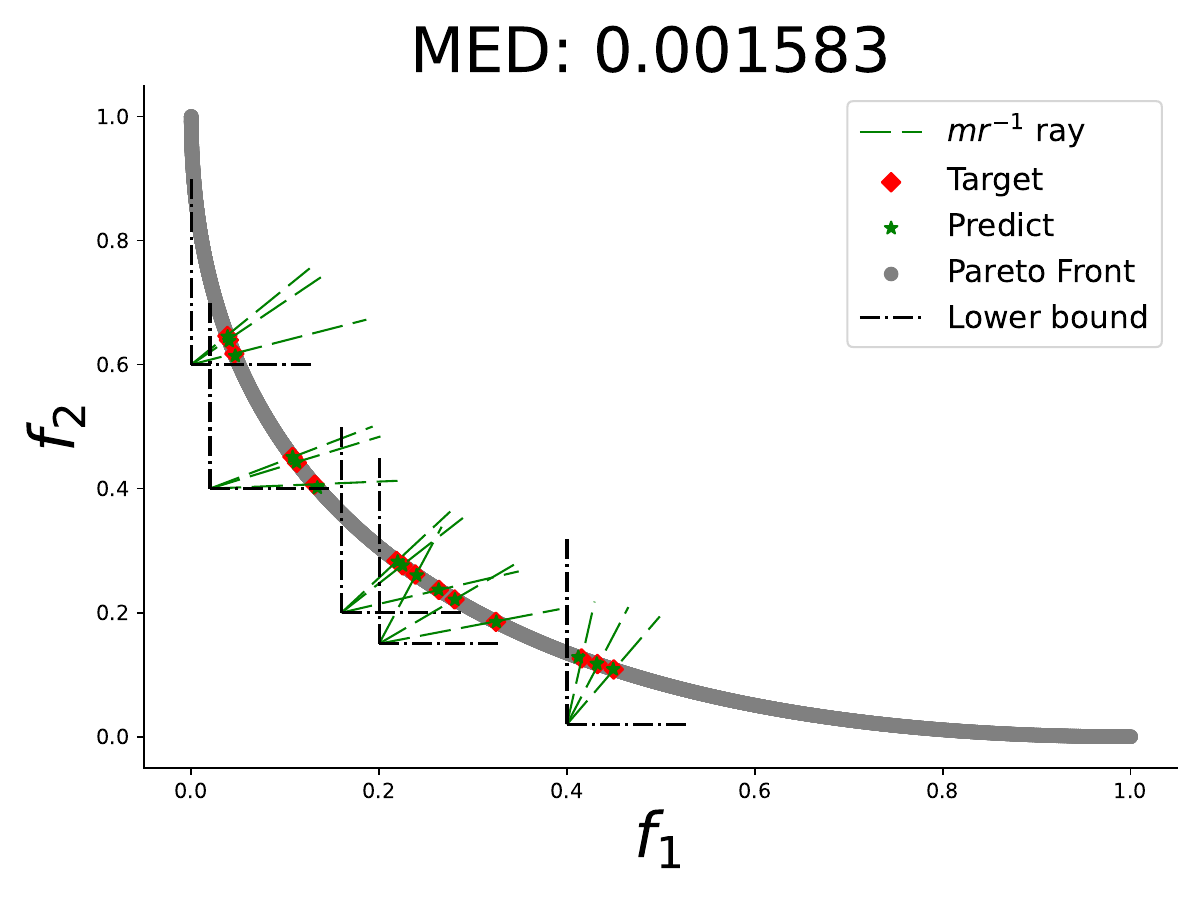}
     \end{subfigure}
     \hfill
     \begin{subfigure}[b]{0.45\textwidth}
         \centering
         \includegraphics[width=\textwidth]{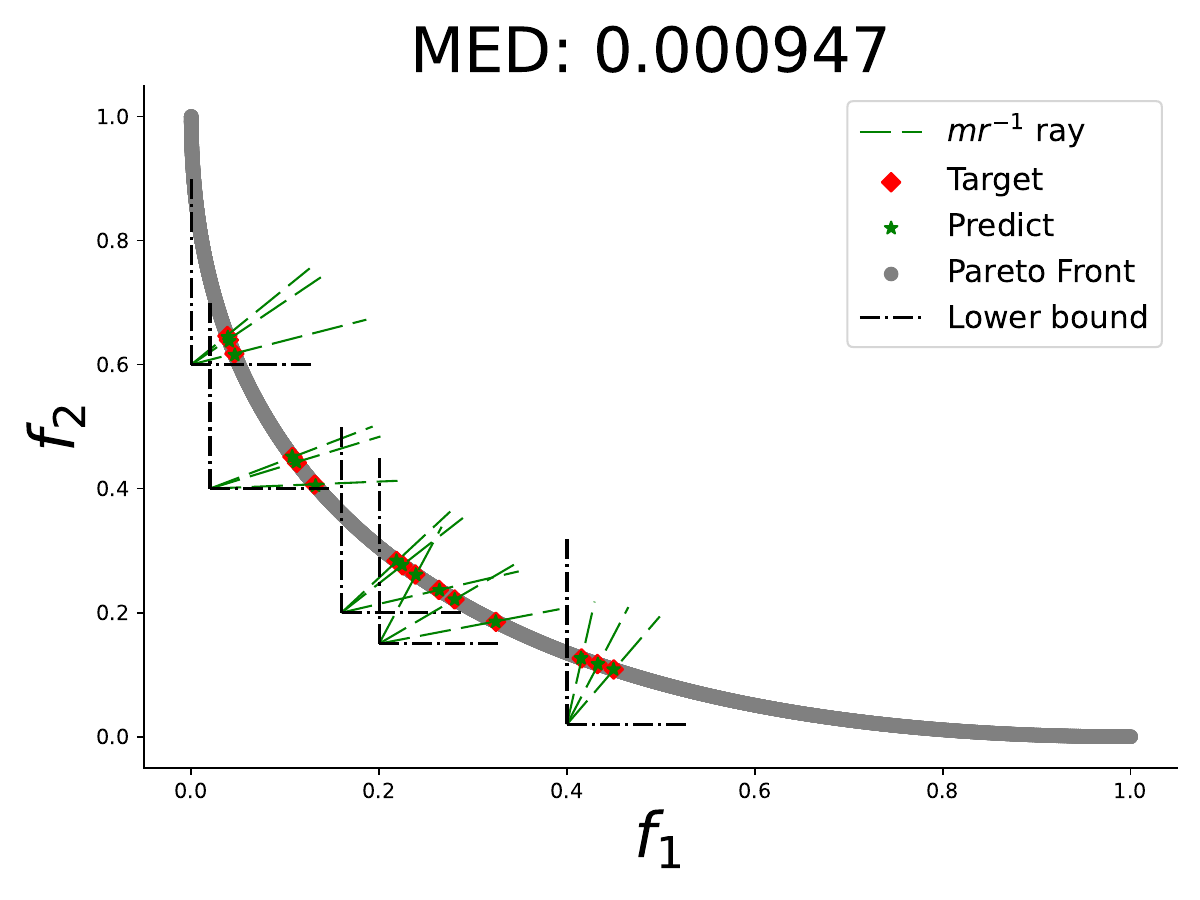}
     \end{subfigure}

    \begin{subfigure}[b]{0.35\textwidth}
         \centering
         \includegraphics[width=\textwidth]{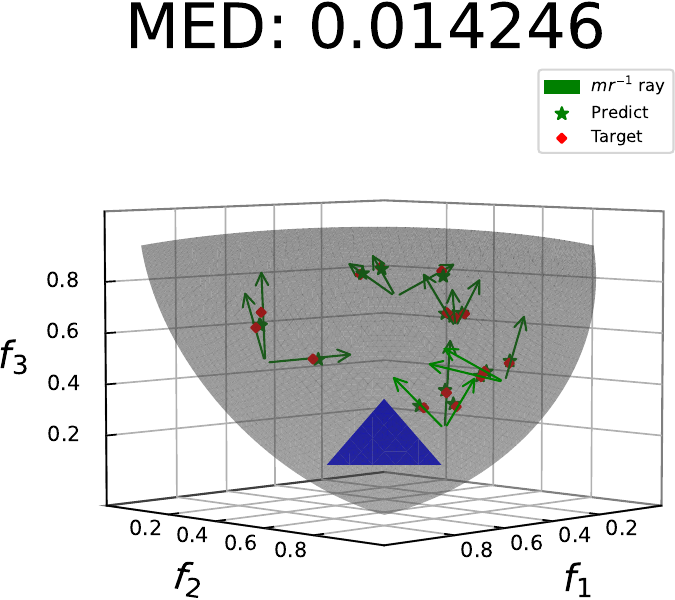}
         \caption{Hyper-MLP}
     \end{subfigure}
     \hfill
     \begin{subfigure}[b]{0.35\textwidth}
         \centering
         \includegraphics[width=\textwidth]{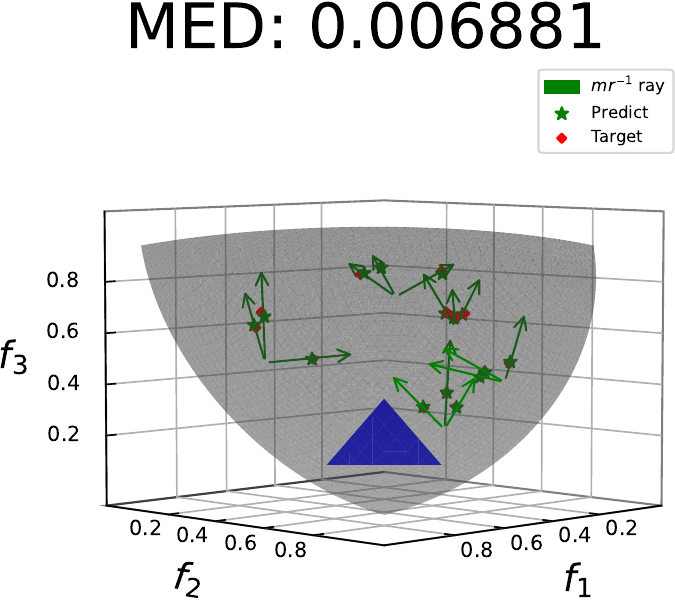}
         \caption{Hyper-Transformer}
     \end{subfigure}
      \caption{The Controllable Pareto Front Learning by Split Feasibility Constraints method achieves an exact mapping between the predicted solution of Hypernetwork and the truth solution, as illustrated in Examples \ref{CVX1} (top), \ref{CVX2} (middle), and \ref{CVX3} (bottom).}
      \label{fig:connect_1}
\end{figure*}

\begin{figure*}[ht]
    \begin{subfigure}[b]{0.45\textwidth}
         \centering
         \includegraphics[width=\textwidth]{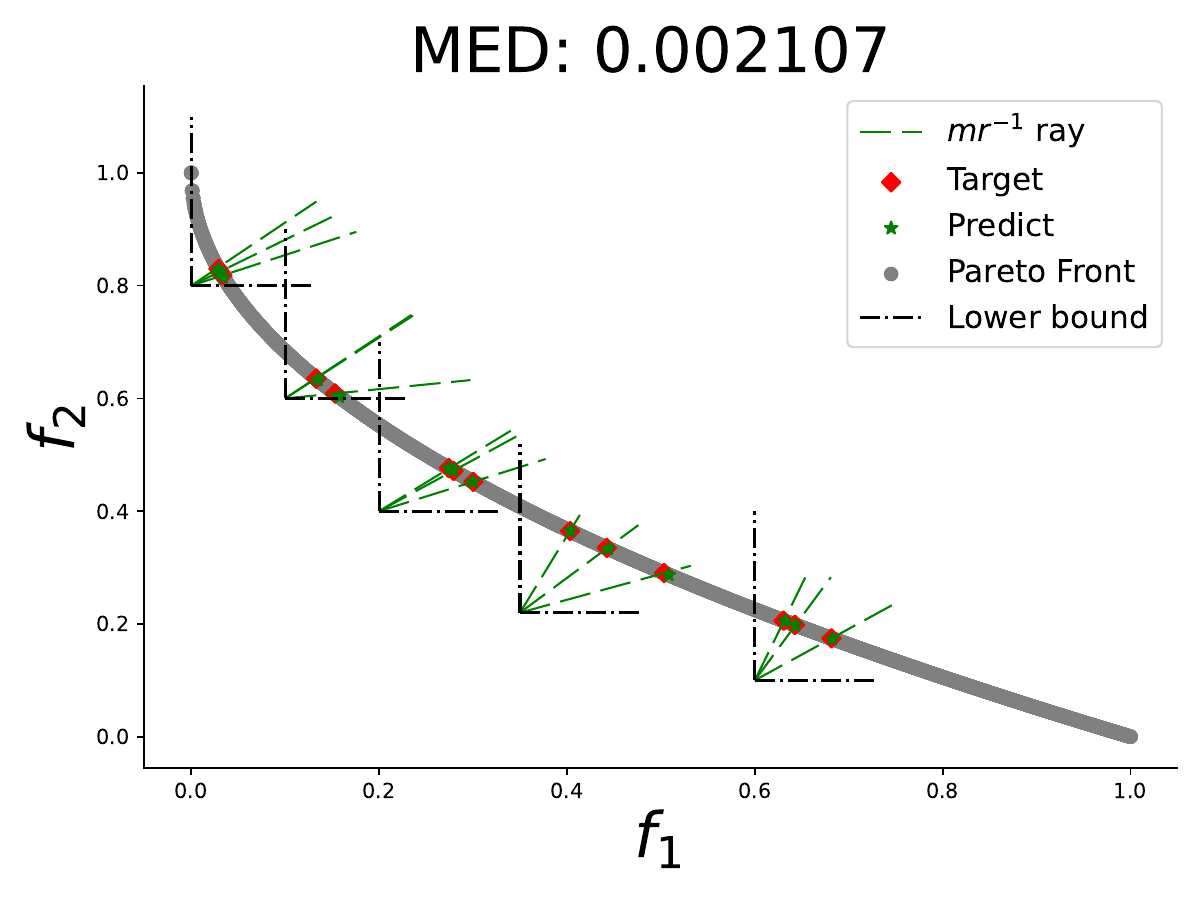}
     \end{subfigure}
     \hfill
     \begin{subfigure}[b]{0.45\textwidth}
         \centering
         \includegraphics[width=\textwidth]{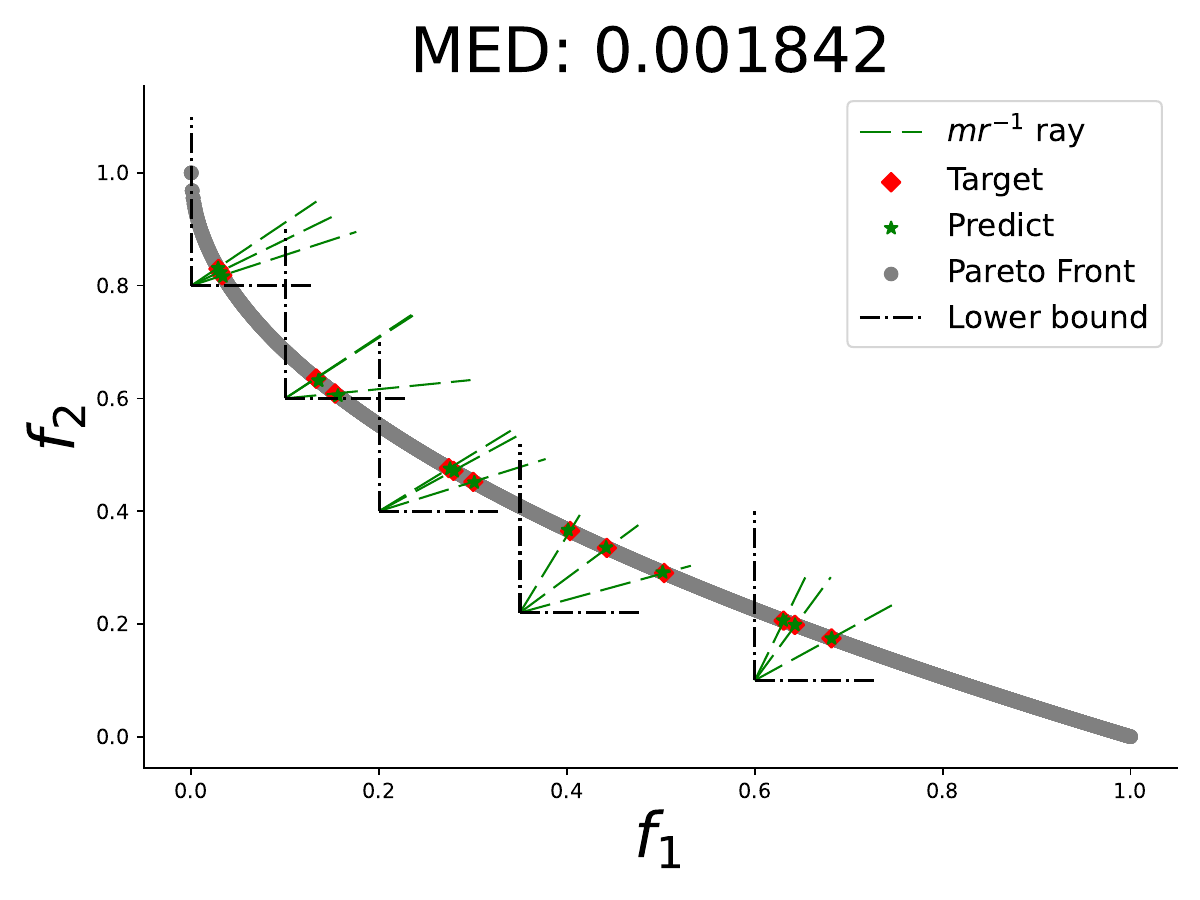}
     \end{subfigure}

    \begin{subfigure}[b]{0.45\textwidth}
         \centering
         \includegraphics[width=\textwidth]{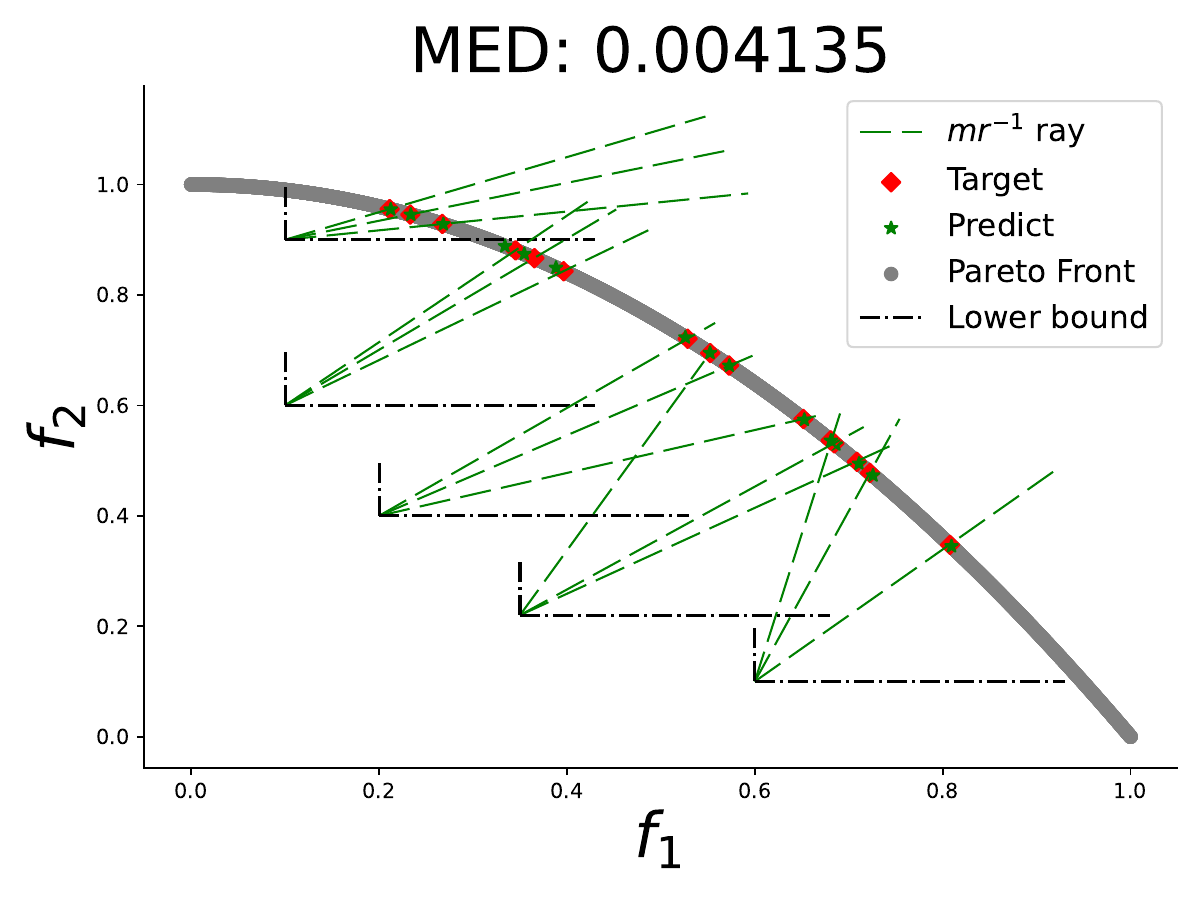}
     \end{subfigure}
     \hfill
     \begin{subfigure}[b]{0.45\textwidth}
         \centering
         \includegraphics[width=\textwidth]{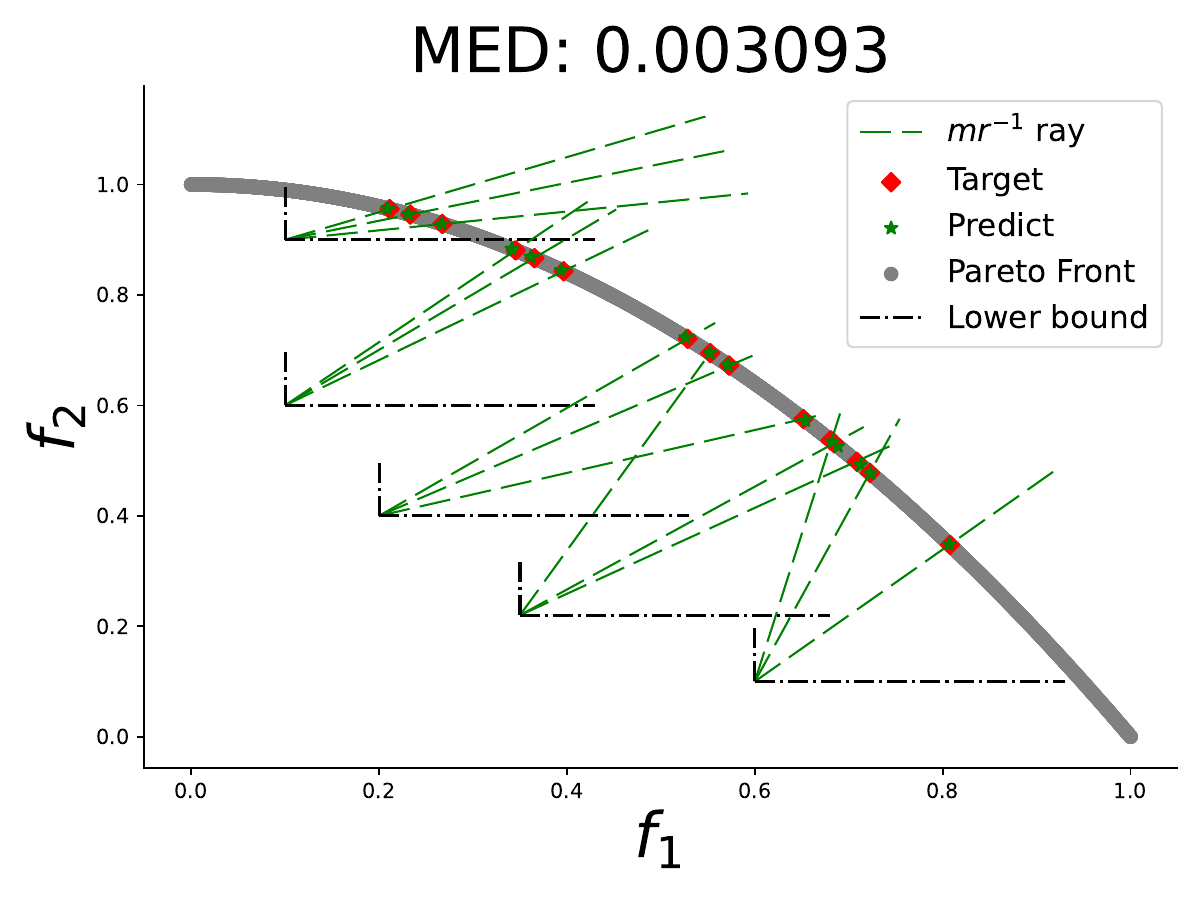}
     \end{subfigure}

    \begin{subfigure}[b]{0.35\textwidth}
         \centering
         \includegraphics[width=\textwidth]{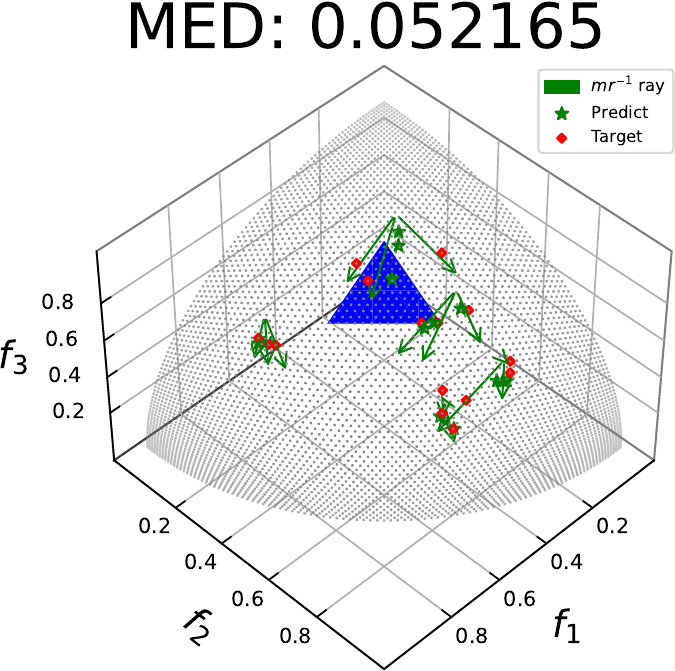}
         \caption{Hyper-MLP}
     \end{subfigure}
     \hfill
     \begin{subfigure}[b]{0.35\textwidth}
         \centering
         \includegraphics[width=\textwidth]{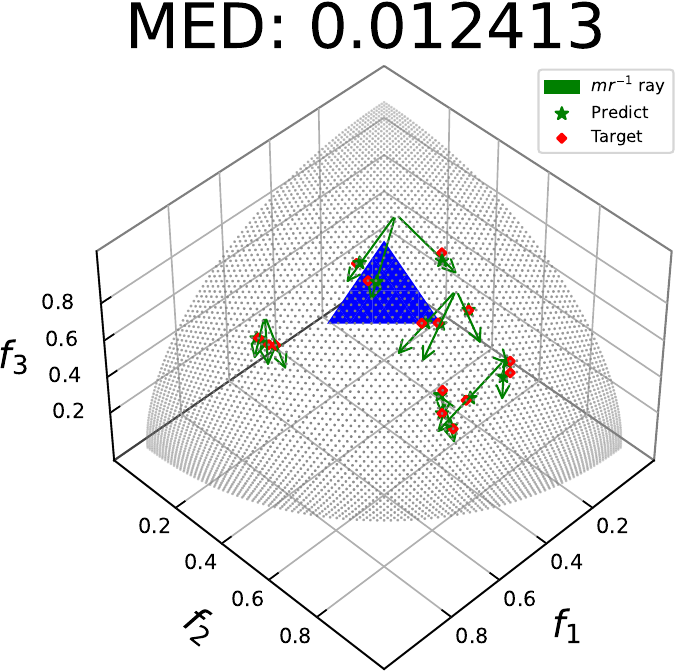}
         \caption{Hyper-Transformer}
     \end{subfigure}
      \caption{The Controllable Pareto Front Learning by Split Feasibility Constraints method achieves an exact mapping between the predicted solution of Hypernetwork and the truth solution, as illustrated in Examples \ref{ZDT} (top), \ref{ZDT2} (middle), and \ref{DTLZ2} (bottom).}
      \label{fig:connect_2}
\end{figure*}
\begin{figure*}[ht]
    \begin{subfigure}[b]{0.35\textwidth}
         \centering
         \includegraphics[width=\textwidth]{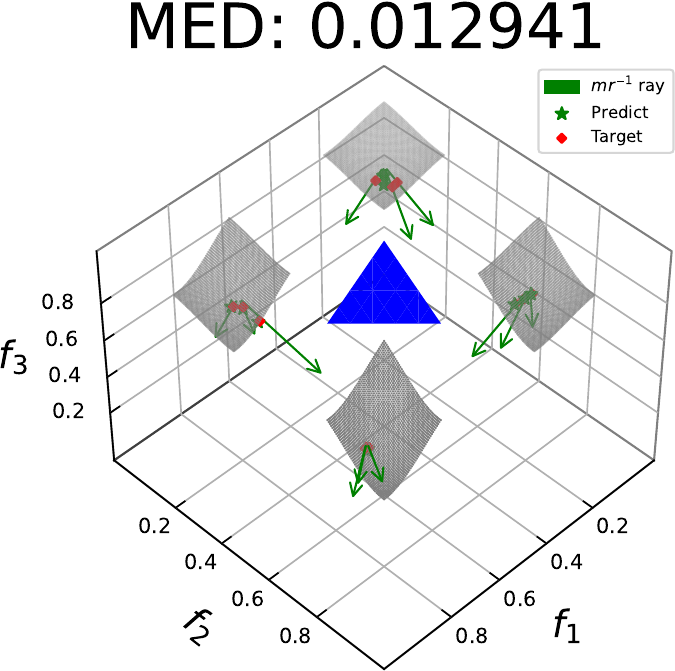}
     \end{subfigure}
     \hfill
     \begin{subfigure}[b]{0.35\textwidth}
         \centering
         \includegraphics[width=\textwidth]{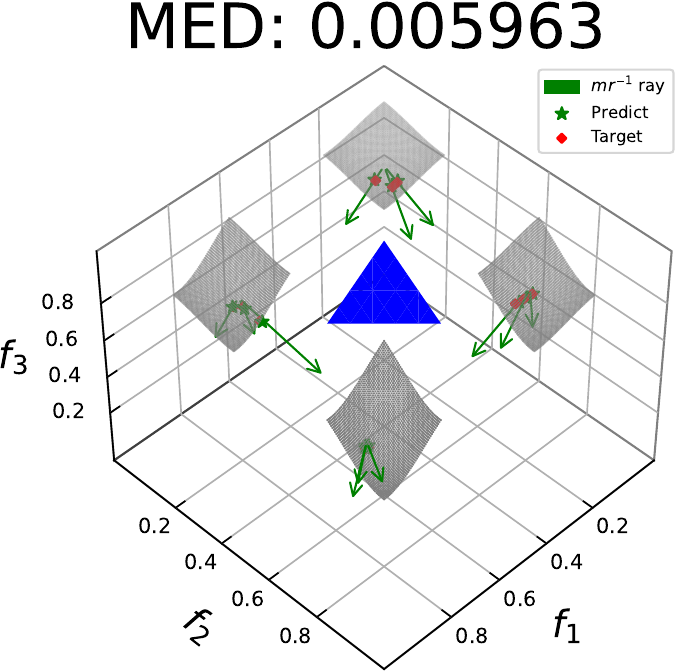}
     \end{subfigure}

    \begin{subfigure}[b]{0.45\textwidth}
         \centering
         \includegraphics[width=\textwidth]{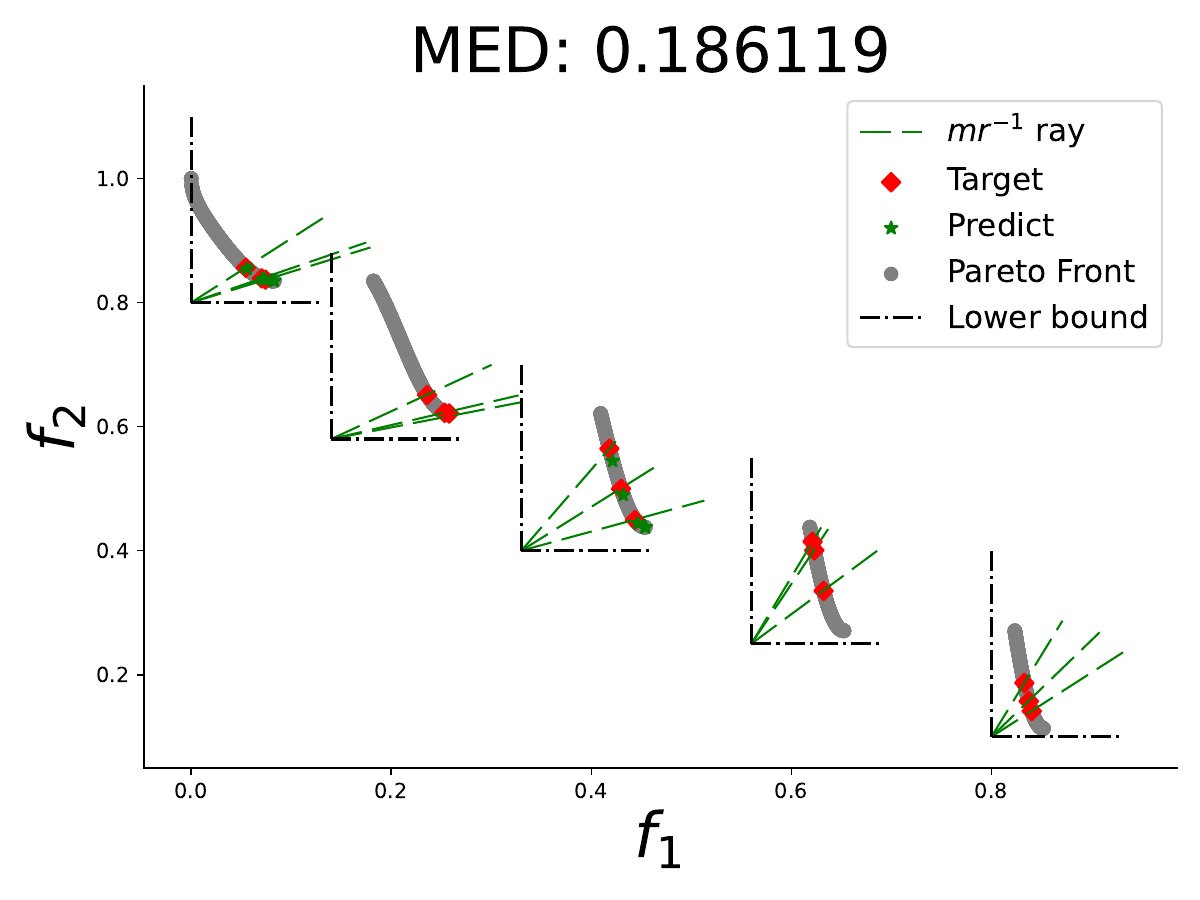}
     \end{subfigure}
     \hfill
     \begin{subfigure}[b]{0.45\textwidth}
         \centering
         \includegraphics[width=\textwidth]{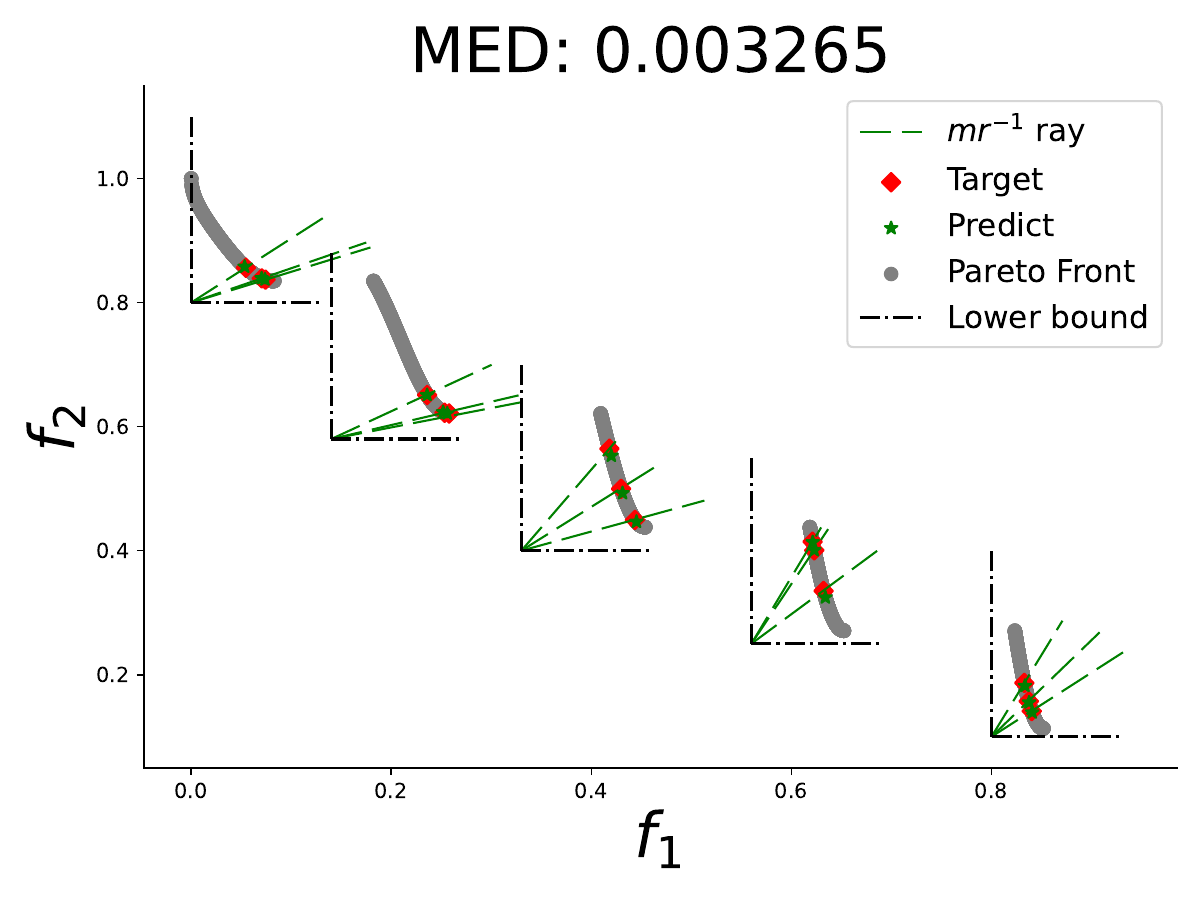}
     \end{subfigure}

    \begin{subfigure}[b]{0.45\textwidth}
         \centering
         \includegraphics[width=\textwidth]{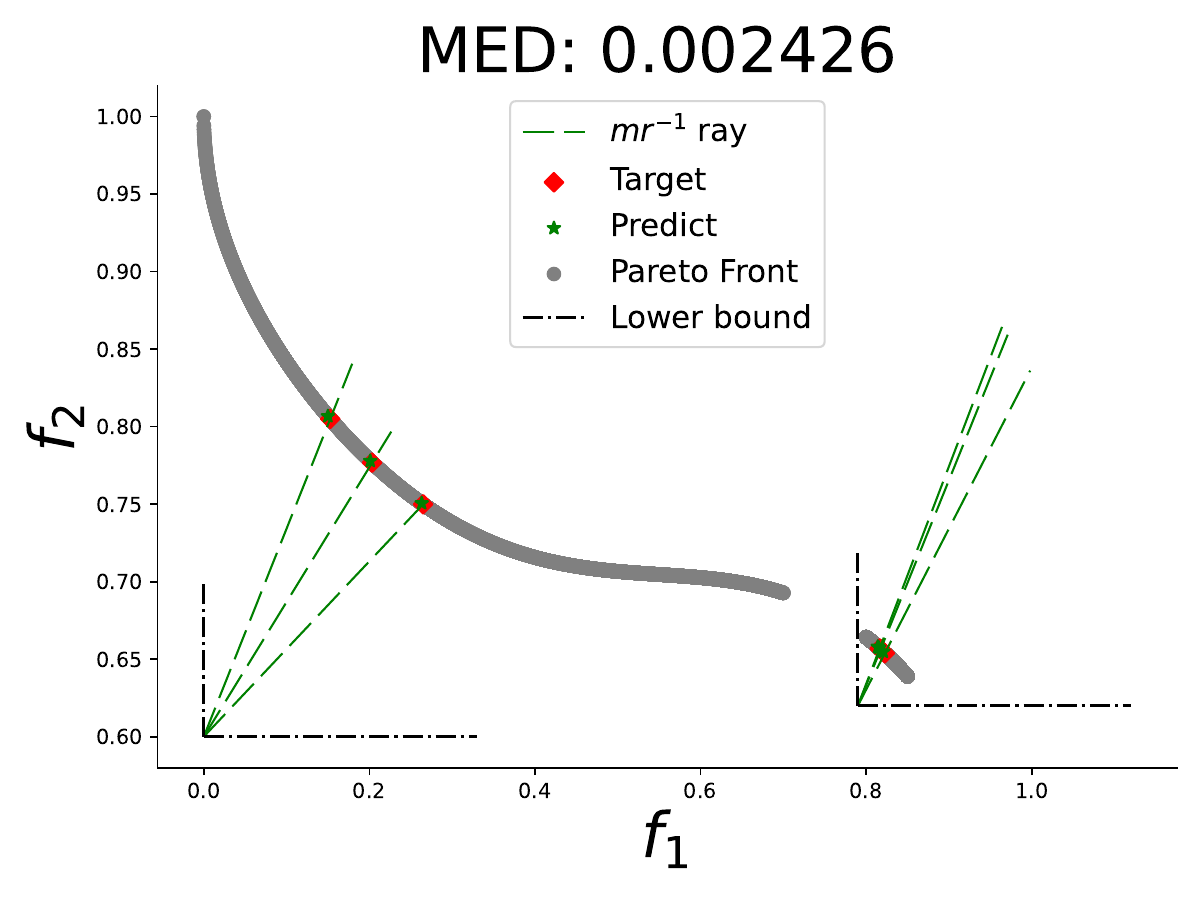}
         \caption{Join Input Hyper-Transformer}
     \end{subfigure}
     \hfill
     \begin{subfigure}[b]{0.45\textwidth}
         \centering
         \includegraphics[width=\textwidth]{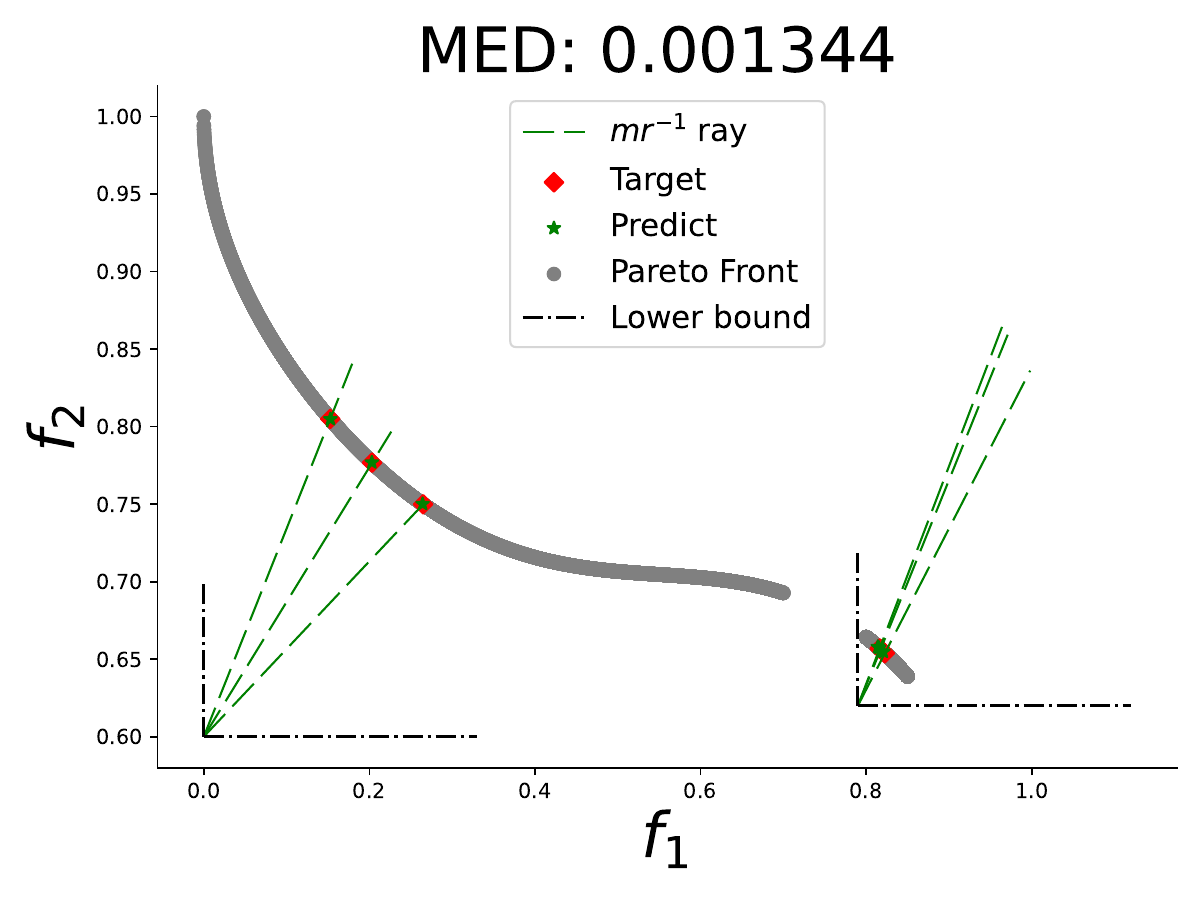}
         \caption{Mixture of Experts Hyper-Transformer}
     \end{subfigure}
      \caption{The Controllable Pareto Front Learning by Split Feasibility Constraints method achieves an exact mapping between the predicted solution of Hypernetwork and the truth solution, as illustrated in Examples \ref{DTLZ7} (top), \ref{ZDT3} (middle), and \ref{ZDT3_variant} (bottom).}
      \label{fig:disconnect}
\end{figure*}

\newpage
\bibliographystyle{cas-model2-names}

\bibliography{cas-refs}

\end{document}